\documentclass[a4paper]{article}

\usepackage[margin=1.31in]{geometry}

\usepackage{graphicx}
\usepackage[switch, modulo]{lineno}

\usepackage{times}

\usepackage{subfig}
\usepackage{amsmath}
\usepackage{amsfonts}
\usepackage{thmtools} 
\usepackage{wasysym}
\usepackage{mathtools}
\usepackage{stmaryrd}
\usepackage{tikz}
\usetikzlibrary{matrix,arrows,decorations.pathmorphing}

\usepackage{xspace}
\usepackage{colortbl}

\usepackage{enumitem} 
\setlist[itemize]{leftmargin=5mm,itemsep=0.0pt,topsep=3pt}
\setlist[description]{itemsep=1pt}

\newtheorem{myDefinition}{Definition}[section]
\numberwithin{myDefinition}{section}

\numberwithin{myTheorem}{section} 

\newtheorem{myCorollary}{Corollary}[section]
\numberwithin{myCorollary}{section} 

\newtheorem{myProposition}{Proposition}[section]
\numberwithin{myProposition}{section} 

\newtheorem{myExample}{Example}
\numberwithin{myExample}{section} 
\newtheorem{proof}{Proof}[section]
\newtheorem{myNotation}{Notation}[section]
\numberwithin{myNotation}{section}

\newcommand{\Definitionref}{\textcolor{black}{Definition}\xspace}

\newcommand{\Tableref}{\textcolor{black}{Table}\xspace}

\newcommand{\Equationref}{\textcolor{black}{Equation}\xspace}

\newcommand{\Figureref}{\textcolor{black}{Figure}\xspace}

\newcommand{\Propositionref}{\textcolor{black}{Proposition}\xspace}

\newcommand{\tuple}[1]{\left\langle #1 \right\rangle}

\newcommand{\lit}{\varphi}
\newcommand{\LIT}{\ensuremath{\Phi}\xspace}

\newcommand{\conf}{\ensuremath{conflict}\xspace}
\newcommand{\conflict}{}

\newcommand{\labelling}{\mathrm{L}}
\newcommand{\alas}{\mathbf{l}} 

\newcommand{\Trule}{Rules}
\newcommand{\Tsup}{\succ}
\newcommand{\Tconf}{\conf}

\newcommand{\argsbuilt}[1]{Args(#1)}

\newcommand{\Body}[1]{\ensuremath{Body}(#1)}
\newcommand{\Conc}{\ensuremath{conc}\xspace}
\newcommand{\Sub}{\ensuremath{Sub}\xspace}
\newcommand{\DirectSub}{\ensuremath{DirectSub}\xspace}
\newcommand{\Rules}{\ensuremath{Rules}\xspace}
\newcommand{\TopRule}{\ensuremath{TopRule}\xspace}

\newcommand{\AR}{\mathcal{A}}
\newcommand{\defeat}{\leadsto}
\newcommand{\support}{\Mapsto}

\newcommand{\gengraphsymb}[1]{G}
\newcommand{\gengraph}[1]{G_{#1}}
\newcommand{\subgraphs}[1]{\Sub(#1)}

\newcommand{\specification}{specification~}

\newcommand{\set}[1]{\{#1\}}

\newcommand{\subtheory}{U}

\newcommand{\primaryquote}[1]{`#1'}
\newcommand{\secondaryquote}[1]{``#1''}

\newcommand{\IN}{\mbox{\scalebox{0.75}{$\mathsf{IN}$}}}
\newcommand{\OUT}{\mbox{\scalebox{0.75}{$\mathsf{OUT}$}}}
\newcommand{\UND}{\mbox{\scalebox{0.75}{$\mathsf{UN}$}}}
\newcommand{\ON}{\mbox{\scalebox{0.75}{$\mathsf{ON}$}}}
\newcommand{\OFF}{\mbox{\scalebox{0.75}{$\mathsf{OFF}$}}}
\newcommand{\OFFJ}{\mbox{\scalebox{0.75}{$\mathsf{OFJ}$}}}
\newcommand{\SKJ}{\mbox{\scalebox{0.75}{$\mathsf{SKJ}$}}}
\newcommand{\CRJ}{\mbox{\scalebox{0.75}{$\mathsf{CRJ}$}}}
\newcommand{\NOJ}{\mbox{\scalebox{0.75}{$\mathsf{NOJ}$}}}

\newcommand{\inn}{\mathsf{in}}
\newcommand{\out}{\mathsf{out}}
\newcommand{\und}{\mathsf{un}}
\newcommand{\no}{\mathsf{no}}

\newcommand{\unp}{\mathsf{unp}}
\newcommand{\off}{\mathsf{off}}

\newcommand{\klabels}{\mathsf{LitLabels}}
\newcommand{\llabels}{\mathsf{ArgLab}}
\newcommand{\jarglabels}{\mathsf{JArglabels}}
\newcommand{\labarg}{\mathrm{L}}
\newcommand{\lablit}{\mathrm{K}}

\newcommand{\setofalllabels}[2]{\ensuremath{{\mathcal{L}_{#2}(#1)}}}
\newcommand{\crit}[3]{\ensuremath{\mathcal{L}^{#1}_{#3}(#2)}}  
  
\newcommand{\argset}{\mathcal{A}}
\newcommand{\genargset}{\mathcal{E}}

\newcommand{\domain}{E}
\newcommand{\domainelem}{e}

\newcommand{\rough}{worst-case sensitive}
\newcommand{\Rough}{Worst-case sensitive}

\newcommand{\gullible}{semi-skeptical\xspace}
\newcommand{\Gullible}{Semi-skeptical\xspace}

\newcommand{\CF}{\ensuremath{\mathsf{cf}}}
\newcommand{\ALL}{\ensuremath{\mathsf{all}}}

\newcommand{\rulecomplete}{rule-complete\xspace}
\newcommand{\subargumentcomplete}{subargument-complete\xspace}
\newcommand{\Rulecomplete}{Rule-complete\xspace}
\newcommand{\Subargumentcomplete}{Subargument-complete\xspace}

\newcommand{\PTF}{PTF\xspace}
\newcommand{\mPTF}{\mathrm{PTF}\xspace}
\newcommand{\PGF}{PGF\xspace}
\newcommand{\mPGF}{\mathrm{PGF}\xspace}
\newcommand{\PAF}{PLF\xspace}
\newcommand{\mPAF}{\mathrm{PLF}\xspace}
\newcommand{\PAG}{PAG\xspace}
\newcommand{\PAGs}{PAGs\xspace}
\newcommand{\mPAG}{\mathrm{PAG}\xspace}
\newcommand{\PTFs}{PTFs\xspace}
\newcommand{\PGFs}{PGFs\xspace}
\newcommand{\PAFs}{PLFs\xspace}
\newcommand{\PEF}{PEF\xspace}
\newcommand{\mPEF}{\mathrm{PEF}\xspace}
\newcommand{\PEFs}{PEFs\xspace}

\newcommand{\genPGF}[1]{PGF(#1)}
\newcommand{\genPAF}[1]{PLF(#1)}

\newcommand{\basedPGF}[1]{PGF(#1)}
\newcommand{\basedPTF}[1]{PTF(#1)}

\newcommand{\EP}{P_{\textrm{PEF}}}

\newcommand{\vBone}{0.5}
\newcommand{\vBtwo}{0.2}

\newcommand{\litbone}{\neg\mathsf{b1}}
\newcommand{\litbtwo}{\neg\mathsf{b2}}
\newcommand{\litb}{\neg\mathsf{b}}
\newcommand{\litc}{\mathsf{c}}
\newcommand{\litd}{\neg \mathsf{c}}

\newcommand{\RT}[1]{\textcolor{black}{#1}}
\newcommand{\RR}[1]{\textcolor{black}{#1}}
\newcommand{\RRR}[1]{\textcolor{black}{#1}}
\newcommand{\RRT}[1]{\textcolor{black}{#1}}
\newcommand{\PB}[1]{\textcolor{black}{#1}}
\newcommand{\PBB}[1]{\textcolor{black}{#1}}
\newcommand{\PBX}[1]{\textcolor{black}{#1}}

\newcommand{\wfag}{well-formed\xspace}
\newcommand{\Wfag}{Well-formed\xspace}

\newcommand{\Selunc}{Acceptance\xspace}
\newcommand{\selunc}{acceptance\xspace}

\newcommand{\probFrame}{\mathfrak{F}\xspace}
\begin{document}
\title{A Labelling Framework for Probabilistic Argumentation
}


\author{
\begin{tabular}{c}
  R\'{e}gis Riveret,
  \emph{CSIRO, Australia}\\[0.5ex]
  Pietro Baroni,
  \emph{University of Brescia, Italy}\\[0.5ex]
  Yang Gao,
  \emph{Chinese Academy of Sciences, China}\\[0.5ex]
  Guido Governatori,
  \emph{CSIRO, Australia}\\[0.5ex]
  Antonino Rotolo,
    \emph{University of Bologna, Italy}\\[0.5ex]
  Giovanni Sartor,
  \emph{European University Institute,  Italy}
  \end{tabular}
}

\date{}
\maketitle

\begin{abstract}
The combination of argumentation  and probability paves the way to new accounts of  qualitative and  quantitative uncertainty, thereby offering new theoretical and applicative opportunities.
Due to a variety of interests, probabilistic argumentation is approached in the literature with different frameworks, pertaining to structured  and abstract argumentation, and with respect to diverse types of  uncertainty, in particular the uncertainty on the credibility of the premises, the uncertainty about which arguments to consider, and the uncertainty on the acceptance status of arguments or statements.
\RRR{Towards a general framework for probabilistic argumentation, we investigate a labelling-oriented framework} encompassing a basic setting for rule-based argumentation and its (semi-) abstract account, along with diverse types of uncertainty. Our framework \PB{provides a systematic treatment of various kinds of uncertainty and of their relationships} and allows us to back or question assertions from the literature.
\end{abstract}

\section{Introduction}

Argumentation  aims at supporting rational persuasion and deliberation in domains where no conclusive logical proofs are available. It addresses defeasible claims raised on the basis of  partial,  uncertain and possibly conflicting pieces of information. Argumentation has been a traditional concern for philosophy and legal theory since Aristotle and Cicero, but only recently has become a focus for  research in artificial intelligence \cite{DBLP:journals/aim/AtkinsonBGHPRST17,BenchCapon2007619}. 

There exist\RRT{s} a variety of  formal models capturing different aspects of the complexity of argumentation activities.
 For instance\RRT{,} argument construction models (either rule-based or logic-based) \cite{DBLP:journals/argcom/ModgilP14,Toni2014,GarciaSimari2014} need to be integrated with formal approaches for the evaluation of the acceptance status of arguments, possibly at a more abstract level \cite{dung95acceptability,DBLP:journals/ker/BaroniCG11}. Moreover, \PBB{argument acceptance needs to be projected on the statements corresponding to their conclusions,  in order to assess their acceptance in turn.} In general, different aspects and their models need to be combined together to get a satisfactory coverage of the full picture.

While argumentation-based reasoning enhances a system's ability to \PBX{reach} defeasible determinations under conditions of uncertainty and conflict, most formal approaches to argumentation suffer a substantial limitation: they are unable to provide a comparative quantitative account of the persuasiveness of alternative conclusions. Typically,  argumentation systems either identify a single skeptical outcome, or propose a set of credulous alternatives, without specifying corresponding degrees of credibility.

It  has been convincingly argued that probability theory \RR{can} help to fill the gap between qualitative and quantitative accounts of uncertainty. For example, probabilistic methods  may propagate quantified uncertainties over (clusters of) premises into quantified uncertainties concerning the dialectical acceptability of arguments and  conclusions.
Hence, the combination of formal argumentation and probability theory 
has received increasing attention in recent years. 
 
\RRR{Different types of  uncertainty\PBX{, which can be the subject  of probabilistic evaluation,}  can be identified in argumentation. 
For instance,  there may be uncertainty about which premises to believe to construct arguments. There can also be  uncertainty about which arguments and relationships to consider when an abstract representation of arguments and their relations is adopted. Then there can be further uncertainty about the outcomes of the assessment, namely the acceptance status of arguments and statements. This list is not exhaustive and all \RRR{kinds} of uncertainty  are potentially appreciable and meaningful.}

\RRR{Since different types of probabilistic uncertainty can be discerned  and argumentation can be dealt with at different level of abstraction, \PBB{different families of approaches to probabilistic argumentation can be \RRR{found} in the literature.
First probabilistic notions have been \RRT{investigated} at different abstraction levels, ranging from structured argumentation formalisms \cite{doi:10.1080/11663081.1997.10510900,DBLP:journals/ail/RiveretRS12,DBLP:conf/at/Rienstra12,Dung:2010:TAJ} to abstract  argumentation frameworks \cite{DBLP:conf/tafa/LiON11,DBLP:conf/ecai/Thimm12,DBLP:journals/ijar/Hunter13,KRthimm:16,DBLP:conf/atal/RiveretG16,DBLP:journals/argcom/RiveretKDP15,DBLP:conf/atal/RiveretPKD15}. Further, one can distinguish approaches based on the uncertainty concerning \RRT{the arguments to  actually include in the argumentation process}, as in the so-called `constellations' approach \cite{DBLP:conf/tafa/LiON11,DBLP:journals/ijar/Hunter13}, and approaches focusing on a \PBB{probabilistic notion of the acceptance status of arguments}, as in the `epistemic' approach \cite{DBLP:conf/ecai/Thimm12,DBLP:journals/ijar/Hunter13,KRthimm:16}.} 
Existing approaches are thus heterogeneous and complementary, leading to the goal of setting up a general \RRT{and unified} framework \PBB{able} to cover multiple types of  probabilistic uncertainty in structured and abstract argumentation.}\smallskip

\textbf{Contribution.}  
\RRR{Towards a general framework for  probabilistic  argumentation,
we present a labelling-oriented framework that we call  \emph{probabilistic labellings} and  
which covers uncertainty on  inclusion of argumentative pieces as well as uncertainty regarding acceptance of arguments or statements, even \PB{in the case} all the argumentative pieces are included in the reasoning activity.}

To \RRT{establish the} framework of probabilistic labellings, we approach  probabilistic argumentation at various levels of abstraction.  \RRT{We start} from a basic rule-based argumentation setting, \RRT{then give} a semi-abstract account of it, \PBX{and} finally \RRT{devise}  probabilistic labellings of arguments, and of statements 
(an issue which has received limited attention in previous literature).
\PBB{Along this journey, the main contributions  are \RRR{as follows}:
\begin{itemize}
\item we introduce some non-traditional argument labellings, going beyond the three labels $\IN$, $\OUT$ and $\UND$ commonly used in the literature, \RRT{to capture} different kinds of uncertainty in a unified representation;
\item we extend this unified labelling-based representation to statements corresponding to argument conclusions;
\item we define a comprehensive set of probabilistic frames corresponding to different kinds of uncertainty at different abstraction level\PBX{s} and analyse the relationships betwen them;
\item we analyse the enhanced expressiveness of \RRT{probabilistic labellings}, with particular reference to  some influential probabilistic argumentation approaches in the literature.
\end{itemize}
}

We will \RRR{refrain} from discussing possible interpretations of probability values (such as classical, frequentist or Bayesian views on these matters), and  we will not investigate algorithms to compute the probability of the status of arguments and statements, but redirect the readers to some relevant \RRR{literature}. \smallskip

{ \textbf{Outline.}  The paper is organised as follows. 
In Section \ref{sec:genlabelling}, we introduce a minimalist rule-based argumentation framework and a (semi-) abstract account featuring argumentation graphs. \PBB{On this basis we then introduce our generalised labelling approach for arguments and statements}.
In Section \ref{sec:probSetting}, we present a set of probabilistic frames for argumentation, capturing uncertainty at different stages of the argumentation process and leading to the main notion of probabilistic labellings of arguments and statements.
Section \ref{secconstepa} analyses our proposal and its technical properties in relation with two \PB{influential} 
approaches to probabilistic argumentation, namely the \primaryquote{constellations} approach and the \primaryquote{epistemic} approach.
We \RRT{discuss} relevant literature in Section \ref{related}, before concluding in Section \ref{sec:conclusion}. \RRR{To help the \PBB{reader}, \RT{key notations are summarised  at the end of the paper}.}

\section{Labelling of Arguments and Statements}
\label{sec:genlabelling}

We assume a generic argumentation process which can be \RRT{summarised} as follows: arguments are built \RRT{out of} a rule base;  their relationships \RRT{induce} an argumentation graph; argument acceptance is assessed on the basis of the argumentation graph \RRT{through} labelling semantics; and statement acceptance is assessed on the basis of argument acceptance labellings.
This section introduces a \RRT{formalisation} of the various steps of the above process, which represents a necessary basis for our approach to probabilistic argumentation.
In particular,   Subsection \ref{subsection:argumentationFramework} introduces a simple rule-based formalism and the relevant graph representation\PBX{,} while subsections \ref{subsection:labellingarguments} and \ref{subsection:labellingstatements} deal with argument and statement labellings respectively.
\PBB{The overall schematisation of the argumentation process rests on the existing literature and is not novel \emph{per se}, \RRT{while labellings} to \RRT{handle} at the same time argument or statement inclusion and (if included) argument or statement acceptance, as proposed in subsections \ref{subsection:labellingarguments} and \ref{subsection:labellingstatements}, \RT{provide} an original generalisation of the traditional \RRT{usage of} labellings in the literature.}

\subsection{\PB{Argument construction and argumentation graphs}}
\label{subsection:argumentationFramework}

\sloppy We first present a  rule-based argumentation framework and its (semi-) abstract account, specifying the structures on which we will develop our proposal. 
This rule-based argumentation framework  is  \primaryquote{minimalist}, that is,  we avoid overloading it with features which are not relevant for our approach to probabilistic argumentation. 
As for reference, this rule-based setting is akin to an argumentative interpretation of Defeasible Logic  \cite{jlc:argumentation,comma16}, \RR{with some definitions inspired by the ASPIC$^+$ framework \cite{DBLP:journals/argcom/ModgilP14}}.
Eventually,  Dung's  abstract argumentation graphs \cite{dung95acceptability}  are slightly adapted to capture this rule-based setting in an abstract way and to \RR{characterise different} probabilistic argumentation frameworks \PBB{(see also \cite{LiThesis} for an approach mapping ASPIC$^+$ frameworks to a special kind of abstract frameworks, called Extended Evidential Argumentation Frameworks)}.    \smallskip

\RRT{The building blocks of the argumentation  framework are  atomic formulas (\RR{or} atoms, i.e. formulas with no propositional structures), from which literals are considered.} 

\begin{myDefinition}[Literal]\label{deflit}
A \emph{literal} is an atomic formula  or the negation of an atomic formula.
\end{myDefinition}
%
\begin{myDefinition}[Complementary literal]
Given a literal $\lit$, its  \emph{complementary literal} is a literal, denoted as $- \lit$, such that if $\lit$ is an atom $p$ then $-\lit$ is its negation $\neg p$, and if $\lit$ is $\neg q$ then $-\lit$ is $q$. 
\end{myDefinition}
We distinguish the negation $\neg$ from the negation as failure, denoted  $\sim$, according to which $\sim \lit$ indicates the failure to derive $\lit$. \smallskip

\sloppy The next construct regards rules that relate literals. 
For the sake of simplicity\PBX{,} we \RRT{have only defeasible rules, i.e. in our context, rules that can be defeated by other rules.}

\begin{myDefinition}[Defeasible rule]
\label{rule}
A \emph{defeasible rule} is a construct of the form:
\begin{center}
$r : \lit_1,\ldots, \lit_n, \sim \lit'_{1}, \ldots, \sim \lit'_{m} \Rightarrow \lit $
\end{center}
 where
\begin{itemize}
\item $r$ is the unique identifier of the rule;
\item $\lit_1,\ldots, \lit_n, \sim \lit'_{1}, \ldots, \sim \lit'_{m}$ ($0 \leq n$ and $0 \leq m$) is the body of the rule, and \RRR{all the} $\lit_i$ and $\lit'_{j}$ are literals; 
\item $\lit$ is the consequent (or head) of the rule, which is a single literal.
\end{itemize}
\end{myDefinition}
\begin{myNotation}
The body of a rule $r$ is denoted $\Body{r}$. So given a rule as in Definition \ref{rule}, we have $\Body{r} = \{ \lit_1,\ldots, \lit_n, \sim \lit'_{1}, \ldots, \sim \lit'_{m}\}$.
\end{myNotation}

Rules  may lead to conflicting literals (we ensure later that two conflicting literals cannot be both accepted). For this reason, we assume that a conflict relation is defined over the set of literals to express conflicts in addition to those corresponding to negation.

\begin{myDefinition}[Conflict relation]
\label{conflict}
Given a set of literals $\LIT$, a \emph{conflict relation} `$\conf$' is a binary relation over $\LIT$, i.e. $ \conf \subseteq \LIT\times \LIT$, such that for any $\lit_1, \lit_2 \in \LIT$, if  $\lit_1$ and  $\lit_2$ are complementary, i.e.  $\lit_1 = -\lit_2$, then  $\lit_1$ and  $\lit_2$ are in conflict, i.e. $\conflict(\lit_1, \lit_2) \in \conf$.
\end{myDefinition}
The conflict relation may be further specified. For example, the relation may be refined with asymmetric and symmetric conflicts to deal with  contrary or contradictory literals as in \cite{DBLP:journals/argcom/ModgilP14}.
However, \PB{treating in detail} such sophistications is not necessary for our purposes. \smallskip

When two rules have  heads in conflict, one rule may prevail over another one. Informally, a rule superiority relation $r_1\succ r_2$ states that the rule $r_1$ prevails over the rule $r_2$. 

\begin{myDefinition}[Superiority relation]
Let $\Trule$ be a set of rules. \PB{A \emph{superiority relation} $\succ$} is a binary relation over $\Trule$, i.e. $ \succ \subseteq \Trule\times \Trule$, \PB{with $r \succ r'$ denoting that $r$ is superior to $r'$}. 
\end{myDefinition}
The superiority relation may enjoy some particular properties. For example, it may be antireflexive and antisymmetric, so that for a rule $r$ it \PBX{does not hold} that $r \succ r$ and for two distinct rules $r$ and $r'$, we cannot have both $r \succ r'$ and $r' \succ r$. However, as for the conflict relation, \PB{treating in detail} such sophistications is not necessary for our purposes.\smallskip

From a set of rules, a conflict relation and a superiority relation, we can \RRR{define} defeasible theories, \RRR{cf. \cite{jlc:argumentation,comma16}}.

\begin{myDefinition}[Defeasible theory]
\label{def:dt}
\sloppy A \emph{defeasible theory} is a tuple $\tuple{\Trule, \conf, \Tsup}$ where $\Trule$ is a set of rules, $\Tconf$ is a conflict relation over literals, and $\Tsup$ is a superiority relation over rules.
\end{myDefinition}

By chaining rules of a defeasible theory, we can build arguments as defined below, \RRR{cf. \cite{DBLP:journals/argcom/ModgilP14}}.
\begin{myDefinition}[Argument] 
\label{def:argument}
An \emph{argument} $A$ constructed from a defeasible theory $\tuple{\Trule, \Tconf, \Tsup}$ is a finite construct of the form:  \begin{center}
$A: \, A_1, \ldots A_n, \sim \lit_{1}, \dots, \sim \lit_m \Rightarrow_r \lit$
\end{center}
where
\begin{itemize}
\item $A$ is the unique  identifier of the argument;
\item $0 \leq n$ and $0 \leq m$, and $A_1, \ldots, A_n$ are arguments constructed from the defeasible theory $\tuple{\Trule, \Tconf, \Tsup}$;
\item $\lit$ is the conclusion of the argument $A$. The conclusion of an argument $A$ is  denoted $\Conc(A)$, i.e. $\Conc(A) = \lit$; 
\item \RRR{$\exists r \in \Trule$ such that} $r: \Conc(A_1), \ldots, \Conc(A_n), \sim \lit_{1}, \dots, \sim \lit_m \Rightarrow \lit$.
\end{itemize}
\end{myDefinition}
\begin{myDefinition}[\PBB{Subarguments and rules}]
\label{definition:elements}
\PBB{Given an argument $A: \, A_1, \ldots A_n, \sim\lit_{1}, \dots,$ $\sim \lit_m \Rightarrow_r \lit$, the set of its subarguments  $\Sub(A)$, the set of its direct subarguments $\DirectSub(A)$,  the last inference rule $\TopRule(A)$\RRT{,} and the set of all the rules in the argument $\Rules(A)$ are defined as follows:} 
\begin{itemize}
\item $\Sub(A) = \Sub(A_1) \cup \ldots \cup \Sub(A_n) \cup\{A\}$;
\item $\DirectSub(A) = \{A_1, \ldots, A_n \}$;
\item $\TopRule(A) = (r: \Conc(A_1), \ldots, \Conc(A_n), \sim \lit_{1}, \dots, \sim \lit_m \Rightarrow \lit)$;
\item $\Rules(A) = \Rules(A_1) \cup \ldots \cup \Rules(A_n) \cup \{\TopRule(A)\}$.
\end{itemize}
\end{myDefinition}

\begin{myNotation}
The set of arguments constructed on the basis of a theory $T$ is denoted $\argsbuilt{T}$.
\end{myNotation}
According to Definition \ref{def:argument}, an argument without subarguments has thus the form $A:\,\, \sim \lit_{1}, \dots, \sim \lit_m \Rightarrow_r \lit$ with  ($0 \leq m$). \primaryquote{Infinite} arguments are not considered, 
 but we may have an infinite set of finite arguments constructed from a defeasible theory.

\begin{myExample}[Running program]
\label{example:basis}
Suppose a research scientist \PBX{is} managing a project critically depending on a program installed on a distant machine which is powered by a solar panel and  a battery.  This program runs \RRR{intermittently}. The scientist has no explicit information about whether the program is running, thus, a priori, its running status is undecided. However, if the scientist receives the notifications that  there is no solar power and no power from the battery, then the machine does not have enough power, and therefore the program is not working. Such setting can give rise to different scenarios, and to capture such scenarios, we \PBX{refer to} the  literals given in \Tableref \ref{literals}.

\begin{table}[ht!]
\centering
\caption{Literals and their meanings}
\label{my-label}
\begin{tabular}{cp{5cm}cl}
$\litbone$ & No solar power.  &  $\litb$    & No power.  \\
$\litbtwo$ & No battery power.   &  $\litc$    & The program is running.  
\end{tabular}
\label{literals}
\end{table}

\noindent  Based on these literals, we may \RRT{form} the  defeasible theory $\mathsf{T} =  \tuple{\mathsf{\Trule},\mathsf{\Tconf},\mathsf{\Tsup }}$  where:
\addtolength{\tabcolsep}{-5pt} 
\begin{center}
\begin{tabular}{r lllll}
 $\mathsf{\Trule}$ & =  $\{$ & $\mathsf{r}_{\litbone}:$& & $ \Rightarrow \litbone$, \\
              &         & $\mathsf{r}_{\litbtwo}:$ & &$ \Rightarrow \litbtwo$, \\
              &         & $\mathsf{r}_{\litb}:$  &  $\litbone, \litbtwo$ & $\Rightarrow  \litb$, \\
              &         & $\mathsf{r}_{\litc}:$ & $\sim  \litb$& $ \Rightarrow \litc$, \\
              &         & $\mathsf{r}_{\litd}:$ & & $ \Rightarrow \litd$ & \}, \\
              \\
\end{tabular}

\begin{tabular}{r lllll}
 $\mathsf{\Tconf}$ & =  $\{$ & $\conflict(\litc, \litd ),  \conflict(\litd, \litc )$ & $\}$, \\
\\
 $\mathsf{\Tsup}$ & =  & $\emptyset$. \\
\end{tabular}
\end{center}
\addtolength{\tabcolsep}{+5pt}

\noindent Based on the theory $\mathsf{T}$, we can build the following arguments:

\begin{center}
\begin{tabular}{lllllllll}
$\mathsf{B1}:  $& $\Rightarrow_{\mathsf{r}_{\litbone}} \litbone$ & & & & $\mathsf{C}:$ & $\sim  \litb$ & $\Rightarrow_{\mathsf{r}_{\litc}} \litc$\\
$\mathsf{B2}:  $& $\Rightarrow_{\mathsf{r}_{\litbtwo}} \litbtwo$ & & & & $\mathsf{D}:$ & &  $\Rightarrow_{\mathsf{r}_{\litd}} \litd$\\
$\mathsf{B}:   $ $\,\mathsf{B1}, \mathsf{B2} $ & $ \Rightarrow_{\mathsf{r}_{\litb}} \litb$\\

\end{tabular}
\end{center}
This example illustrates the concept of defeasible theories and the construction of arguments. \RR{We} will \RR{continue this example} later in our probabilistic setting to determine a probability degree attached to the acceptance  that the program is running.
\hfill $\square$
\end{myExample}

Arguments may conflict and thus attacks between arguments may appear.
We consider two types of attacks: rebuttals (clash of incompatible conclusions) and undercuttings\footnote{The term undercutting is overloaded in argumentation literature and is used with different meanings in different contexts, cf. \cite{DBLP:journals/argcom/ModgilP14}\RRT{.}}} (attacks on negation as failure premises). 
In regard to rebuttals, we assume that there is a preference relation over arguments determining whether two rebutting arguments mutually attack each other or only one of them (being preferred) attacks the other. The preference relation over arguments can be defined in various ways on the basis of the preference over rules. For instance one may adopt the simple last-link ordering according to which an argument $A$ is preferred over another argument $B$, denoted as $A \succ B$, if, and only if, the rule $\TopRule(A)$ is superior to the rule $\TopRule(B)$, i.e. $\TopRule(A) \succ \TopRule(B)$.
We make no assumptions over the specific argument preference relation which is adopted.  \RRR{This leads us to adopt the following definition of attack.}

\begin{myDefinition}[Attack \RRR{relation}]
\label{attack}
\RRR{An \emph{attack relation} $\defeat$ over a set of arguments $\mathcal{A}$ is a binary relation over  the set $\mathcal{A}$, i.e. $\defeat \subseteq \mathcal{A}\times \mathcal{A}$.} An argument $B$ attacks an argument $A$, i.e. $B \defeat A$, if, and only if, $B$ rebuts or undercuts $A$, where 
\begin{itemize}
\item $B$ rebuts $A$ (on $A'$) if, and only if, $\exists A' \in \Sub(A)$ such that $\Conc(B)$ and $\Conc(A')$ are in conflict, i.e. $\conf(\Conc(B),\Conc(A'))$, and  $A' \not\succ B$; 
\item $B$ undercuts $A$ (on $A'$) if, and only if, $\exists A' \in \Sub(A)$ such that $\sim\Conc(B)$ belongs to the body of $\TopRule(A')$, i.e. $(\sim\Conc(B)) \in \Body{\TopRule(A')}$.
\end{itemize}
\end{myDefinition}

On the basis of arguments (and their subarguments) and attacks between arguments, we  extend the common definition of an argumentation framework \cite{dung95acceptability}  by introducing argumentation graphs \RRT{which comprise both  attack and  subargument relations}. This is motivated by our upcoming probabilistic account \PBX{where} the subargument relation, encompassed by most structured argumentation formalisms, \PBX{plays a key role.  I}ntuitively, an argument cannot exist without its subarguments and can only be believed if also its subarguments are believed. 

\RRR{From the definition of the sets of direct subarguments of an argument (see Definition \ref{definition:elements}), we can straightforwardly define a direct subargument relation.}

\begin{myDefinition}[Direct subargument \RRR{relation}]
\label{subargument} 
\RRR{A \emph{direct subargument relation} $\support$ over a set of arguments $\mathcal{A}$ is a binary relation over the set $\mathcal{A}$, i.e. $\support \subseteq \mathcal{A}\times \mathcal{A}$.}
An argument $B$ is a direct subargument of $A$, i.e. $B \support A$, if, and only if,  $B$ belongs to the set of direct subarguments of $A$, i.e. $B \in \DirectSub(A)$.
\end{myDefinition}

 \RRR{By Definition \ref{def:argument}, an argument is not a direct subargument of itself and cannot be a subargument of its direct subarguments. Therefore, the direct subargument relation  is antireflexive  and acyclic.}\smallskip


Arguments and attack relations \RR{can be} then captured in Dung's abstract argumentation graphs, called abstract argumentation  frameworks in  \cite{dung95acceptability}.
\begin{myDefinition} [Abstract argumentation graph]
\label{def:abstractag}
\sloppy An \emph{abstract argumentation graph} is a tuple $\tuple{\AR, \defeat}$ where $\AR$ is a set of arguments, and \RRR{$\defeat$ is an \emph{attack} relation over $\AR$}.
\end{myDefinition}

\RR{It is equally possible to \PBX{refer to}} semi-abstract argumentation graphs, where both the attack and the subargument relations are encompassed, \RRR{cf.  \cite{Prakken:2014:SRA:3006652.3006776}}.


\begin{myDefinition} [Semi-abstract argumentation graph]
\label{def:ag}
\sloppy A \emph{semi-abstract argumentation graph}  is a tuple $\tuple{\AR, \defeat, \support}$ where $\AR$ is a set of  arguments, $\defeat$ is an \emph{attack} relation over $\AR$,  and $\support$ is a \emph{direct subargument}  relation over $\AR$.
\end{myDefinition}

\begin{myNotation} 
\RRR{Given a semi-abstract argumentation graph $G = \tuple{\AR,\defeat,\support}$, we  denote $\AR$ as $\mathcal{A}_{G}$, $\defeat$ as $\defeat_{G}$ and $\support$ as $\support_G$. }
\end{myNotation}

Unconstrained semi-abstract argumentation graphs can be problematic. For instance cycles of the subargument relations are not allowable. Moreover, as usual in structured argumentation formalisms and reflecting the fundamental dependence of an argument on its subarguments,  each attack against a subargument is meant to be an attack against all its superarguments.
Accordingly, \RRR{taking inspiration from \cite{Prakken:2014:SRA:3006652.3006776,Dung:2014:CCL:2655713.2655716}, we introduce the notion of \emph{\wfag} argumentation graphs.}

\begin{myDefinition}[\Wfag semi-abstract argumentation graph]
A semi-abstract argumentation graph  $\tuple{\AR, \defeat, \support}$ is \emph{\wfag} if, and only if:
\begin{itemize}
\item the relation $\support$ is acyclic and antireflexive;
\item if an argument $A$ attacks an argument $B$, and $B$ is a direct subargument of an argument $C$, then $A$ attacks $C$.
\end{itemize}
\end{myDefinition}

\noindent In a \wfag semi-abstract argumentation graph it is easy to see that, by recursion, if an argument $A$ attacks an argument $B$ then $A$ attacks also all the superarguments of $B$.

\begin{myDefinition}[Argumentation graph constructed from a theory]
An (abstract or semi-abstract) argumentation graph $G$ is constructed from a  defeasible theory $T$ if, and only if, $\AR_G$ is the set of all arguments constructed from $T$.
\label{defintion:constructed_from}
\end{myDefinition}

\begin{myNotation} 
\RRR{An argumentation graph constructed from a  defeasible theory $T$ is denoted by $\gengraph{T}$.}  
\end{myNotation}

From the facts that we exclude cyclic arguments and that the attack relation (Definition \ref{attack}) refers to subarguments, we have the following proposition.

\begin{myProposition}
A semi-abstract argumentation graph  constructed from a  defeasible theory is \wfag.
\end{myProposition}
\begin{proof}

\RR{The relation $\support$ is acyclic and antireflexive.} 
By Definition \ref{attack}, if an argument $A$ attacks an argument $B$ then $A$ attacks all the superarguments of $B$, and thus $A$ attacks every argument $C$ such that $B$ is a direct subargument of $C$. Therefore, if an argument $A$ attacks an argument $B$, and $B$ is a direct subargument of an argument $C$, then $A$ attacks $C$.
\end{proof}

In the remainder, we assume that all semi-abstract argumentation graphs are \wfag, and accordingly, the qualification \primaryquote{\wfag} is omitted.
 
\begin{myExample}[continues=example:basis]
Using the definition of attacks and direct subarguments, we can  construct the semi-abstract  argumentation graph as pictured in \Figureref \ref{fig:examplesupport}.
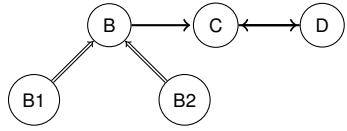
\begin{figure}[ht!]
\centering
\begin{tikzpicture}[node distance=1.4cm,main node/.style={circle,fill=white!20,draw,font=\sffamily\scriptsize}]

 \node[main node] (4) [] {B};
 \node[main node] (5) [right of=4] {C};
 \node[main node] (6) [right of=5] {D};
 \node[main node] (7) [below right of=4] {B2};
 \node[main node] (8) [below left of=4] {B1};

 \path[->,shorten >=1pt,auto, thick,every node/.style={font=\sffamily\small}]
 (4) edge node {} (5)
 (5) edge node {} (6)
 (6) edge node [right] {} (5);
 
  \path[->,>=implies]
 (8) edge [double] node {} (4)
 (7) edge [double] node {} (4);
 
\end{tikzpicture}
\caption{A semi-abstract argumentation graph. Argument $\mathsf{B}$ attacks $\mathsf{C}$, arguments $\mathsf{C}$ and $\mathsf{D}$ attack each \RT{other}. Arguments $\mathsf{B1}$ and $\mathsf{B2}$ are direct subarguments of argument $\mathsf{B}$.}
\label{fig:examplesupport}
\end{figure}

\hfill $\square$
\label{ex:argfromt}
\end{myExample}

The definition of semi-abstract argumentation graphs  allows us to give 
\PB{an} account of probabilistic argumentation which  is intermediate between Dung's abstract argumentation, encompassing attacks only, and structured argumentation formalisms\PBX{,} where both the attack and the subargument relations are defined in formalism-specific terms. 
Whilst this semi-abstract setting is built on the simple formalism based on definitions \ref{deflit}-\ref{subargument},  it can be more generally seen as an abstraction basis to encompass more sophisticated  rule-based argumentation systems (e.g. ASPIC$^+$ \cite{DBLP:journals/argcom/ModgilP14}) or related non-monotonic logic frameworks (e.g. Defeasible Logic \cite{jlc:argumentation}). 

\RRR{Semi-abstract argumentation graphs can be related to some existing  proposals on different notions of supports. 
For example, the representation is similar to the notion of \emph{inference graph} introduced by \cite{Pollock94}.
\emph{Bipolar argumentation frameworks} \cite{DBLP:journals/ijar/CayrolL13} are an example of another semi-abstract settings in the literature encompassing two relations between arguments, namely attack and support. The generic notion of support in bipolar argumentation frameworks allows for different interpretations (namely deductive support, necessary support, and evidential support) with different properties and implications on the attack relation. 
\PBB{A specific notion of \emph{evidential support} is considered in}  \emph{evidential argumentation systems }\cite{DBLP:conf/comma/OrenN08}, where a special argument, denoted as $\eta$, is used to represent evidence (namely \primaryquote{incontrovertible premises}) and evidential support, rooted in $\eta$, is `propagated' through a set-based support relation. \PBB{Relationships and possibilities of translation between evidential argumentation systems and argumentation frameworks with necessities are investigated in \cite{DBLP:conf/comma/PolbergO14}.}}
Our proposal does resort neither to specific notions\PBX{,} like incontrovertible premises\PBX{,} nor to a specific evidence-based interpretation of the notion of support. 
Our approach directly relies on the subargument relation which has a univocal interpretation in our context and is sufficient to develop a general theory of probabilistic labellings. \RR{The inclusion of} more articulated notions, like recursive attacks and supports \cite{DBLP:journals/ijar/BaroniCGG11,Cohen:2015:AAA:2849973.2850374}\PBX{,} is beyond the scope of this work and represents an interesting direction of future research. \medskip

Given an argumentation graph, some of the arguments may be omitted due to the use of some selection criterion: this leads to
 consider subgraphs (which will \PBX{be} useful to specify fundamental statuses of arguments in the next section.)    

\begin{myDefinition}[Subgraph]
\label{def:subgraphs}
Let   $G = (\AR_G, \defeat_G, \support_G )$ denote an argumentation graph. The \emph{subgraph} $H$ of $G$ induced by a set of arguments $\mathcal{A}_H \subseteq \mathcal{A}_G$ is an argumentation graph such that  $H=(\mathcal{A}_H, \defeat_G \cap (\mathcal{A}_H \times \mathcal{A}_H), \support_G  \cap (\mathcal{A}_H \times \mathcal{A}_H))$.
\end{myDefinition}

\begin{myNotation}
The set of all subgraphs of an argumentation graph $G$ is denoted $\subgraphs{G}$, i.e. $\subgraphs{G} =\set{(\mathcal{A}_H, \defeat_G \cap (\mathcal{A}_H \times \mathcal{A}_H), \support_G  \cap (\mathcal{A}_H \times \mathcal{A}_H)) \mid \mathcal{A}_H \subseteq \mathcal{A}_G}$.
\end{myNotation}

Not all subgraphs may be deemed to be \wfag or interesting, since only some of them correspond to sensible constructions from a set of rules. For this reason we introduce  \subargumentcomplete and \rulecomplete sets of arguments.

\begin{myDefinition}[\Subargumentcomplete set]
Let   $G = (\AR_G, \defeat_G, \support_G )$ denote an argumentation graph. A set of arguments $\mathcal{A}_H \subseteq \mathcal{A}_G$ is \emph{\subargumentcomplete} \textcolor{black}{if} for any argument $A$ in $\mathcal{A}_H$,  all the \textcolor{black}{direct} subarguments of $A$ are in $\mathcal{A}_H$, i.e. $\forall A \in \mathcal{A}_H$, if $((B,A) \in \support_G)$ then $B \in \mathcal{A}_H$.
\end{myDefinition}
\begin{myDefinition}[\Rulecomplete set]
Let   $G = (\AR_G, \defeat_G, \support_G )$ denote an argumentation graph. A set of arguments $\mathcal{A}_H \subseteq \mathcal{A}_G$ is \emph{\rulecomplete} \textcolor{black}{if} for all arguments $B$ in  $\mathcal{A}_G$ such that 
all the direct subarguments of $B$ are in  $\mathcal{A}_H$ and  there exists an argument $A$ in $\mathcal{A}_H,$ such that the top rule of  $B$ belongs to the rules of $A$, then $B$ belongs to  $\mathcal{A}_H$, i.e. $\forall B \in  \mathcal{A}_G, \mbox{ if } \RT{DirectSub(B)} \subseteq \mathcal{A}_H$ and $\exists A \in \mathcal{A}_H \mbox{ such that } \TopRule(B) \in \Rules(A)$, then  it holds that $B \in \mathcal{A}_H$. 
\end{myDefinition}

\noindent Notice that a set of arguments can be \subargumentcomplete without being \rulecomplete and vice versa. Therefore  these two notions are complementary.

\Subargumentcomplete or \rulecomplete sets of arguments can be straightforwardly used to qualify subgraphs of an argumentation graph.

\begin{myDefinition}[\Subargumentcomplete subgraph]
\label{definition:subargumentgraph}
\sloppy A subgraph $H = (\mathcal{A}_H, \defeat_H, \support_H)$ of an argumentation graph $G$ induced by a set of arguments $\mathcal{A}_H$ is \emph{\subargumentcomplete} if, and only if, $\mathcal{A}_H$ is \subargumentcomplete.
\end{myDefinition}

\noindent In other words, $H$ is a \subargumentcomplete subgraph of $G$ if any argument of $H$ appears with all its supporting arguments and $H$ has exactly the attacks and supports that appear in $G$ over the same set of arguments.

\begin{myExample}[continues=example:basis] 
\label{example:subgraphs}
In Figure \ref{fig:test} the graph (a) is a \subargumentcomplete subgraph of the graph in \Figureref \ref{fig:examplesupport}, while the graph (b) is not.

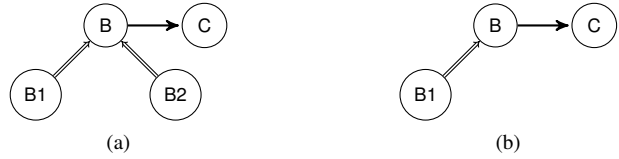
\begin{figure}[ht!]
\centering
\subfloat[]{{
\begin{tikzpicture}[node distance=1.3cm,main node/.style={circle,fill=white!20,draw,font=\sffamily\scriptsize}]

 \node[main node] (4) [] {B};
 \node[main node] (5) [right of=4] {C};
 \node[main node] (7) [below right of=4] {B2};
 \node[main node] (8) [below left of=4] {B1};

 \path[->,>=stealth',shorten >=1pt,auto, thick,every node/.style={font=\sffamily\small}]
 (4) edge node {} (5);
 
  \path[->,>=implies]
 (8) edge [double] node {} (4)
 (7) edge [double] node {} (4);
 
\end{tikzpicture}
}}
\hspace{2cm}
\subfloat[]{{
\begin{tikzpicture}[node distance=1.3cm,main node/.style={circle,fill=white!20,draw,font=\sffamily\scriptsize}]

 \node[main node] (4) [] {B};
 \node[main node] (5) [right of=4] {C};
 \node[main node] (8) [below left of=4] {B1};

 \path[->,>=stealth',shorten >=1pt,auto, thick,every node/.style={font=\sffamily\small}]

 (4) edge node {} (5);

  \path[->,>=implies]
 (8) edge [double] node {} (4);
\end{tikzpicture}
}}
\caption{Subargument-completeness.}
\label{fig:test}
\end{figure}
\vspace{-5mm}
\hfill $\square$
\end{myExample}

\begin{myDefinition}[\Rulecomplete subgraph]
\label{definition:rulecompletegraph}
A subgraph $H = \tuple{\mathcal{A}_H, \defeat_H, \support_H}$ of an argumentation graph $G_T$ built from a defeasible theory $T$, and   induced by a set of arguments $\mathcal{A}_H$, is \emph{\rulecomplete} if, and only if, $\mathcal{A}_{H}$ is \rulecomplete.  
\end{myDefinition}
\begin{myExample}[Abstract example]
Consider a defeasible theory $\mathsf{T}= \tuple{\mathsf{\Trule},\emptyset,\emptyset}$ such that:
\begin{center}
\begin{tabular}{lllllllllll}
 $\mathsf{\mathsf{\Trule}}$ & =  $\{$ & $\mathsf{r}_{1}:$ & & $ \Rightarrow \mathsf{a},$  & $\quad$ & $\mathsf{r}_{\mathsf{2}}:$ &    & $  \Rightarrow \mathsf{b}$, \\ 
             &          & $\mathsf{r}_{3}:$ & $\mathsf{a}$ &$ \Rightarrow \mathsf{b},$  & $\quad$ & $\mathsf{r}_{\mathsf{4}}:$ &  $\mathsf{b}$  & $  \Rightarrow \mathsf{c}$ & $\}.$ 
\end{tabular}
\end{center}
We can build the following arguments:
\begin{center}
\begin{tabular}{lllllllllll}
    & $\mathsf{A}:$ & $ \Rightarrow_{\mathsf{r1}} \mathsf{a},$  & $\quad$ & $\mathsf{AB}:$ &  $ \Rightarrow_{\mathsf{r1}}  \mathsf{a}  \Rightarrow_{\mathsf{r3}} \mathsf{b},$& $\quad$ & $\mathsf{ABC}:$ &  $ \Rightarrow_{\mathsf{r1}}  \mathsf{a}  \Rightarrow_{\mathsf{r3}} \mathsf{b} \Rightarrow_{\mathsf{r4}} \mathsf{c}$  \\ 
    & $\mathsf{B}:$ & $ \Rightarrow_{\mathsf{r2}} \mathsf{b},$  & $\quad$ & $\mathsf{BC}:$ &  $\Rightarrow_{\mathsf{r2}} \mathsf{b} \Rightarrow_{\mathsf{r4}} \mathsf{c}$. &  \\
\end{tabular}
\end{center}
\sloppy The direct subargument relation  $\support$ is such that:
\begin{center}
$\support = \{ (\mathsf{A}, \mathsf{AB}), (\mathsf{AB}, \mathsf{ABC}), (\mathsf{B}, \mathsf{BC})  \}.$
\end{center}
Consequently,
\begin{itemize}
\item the argumentation graph $\mathsf{G}_\mathsf{T} = \tuple{ \{\mathsf{A}, \mathsf{B}, \mathsf{AB}, \mathsf{BC}, \mathsf{ABC} \}, \emptyset, \support }$ is \rulecomplete and  \subargumentcomplete;
\item the argumentation subgraph $\mathsf{H} = \tuple{ \{\mathsf{A}, \mathsf{B}, \mathsf{AB}, \mathsf{BC} \}, \emptyset, \RRR{\{ (\mathsf{A}, \mathsf{AB}), (\mathsf{B}, \mathsf{BC})  \}} }$ is not \rulecomplete, but it is \subargumentcomplete.
\end{itemize}
This example shows that a subgraph of an argumentation graph built from a defeasible theory can be \subargumentcomplete without being \rulecomplete.
\hfill $\square$
\label{example:rulecomplete}
\end{myExample}
\begin{myProposition}
\label{proposition:basiccc}
If an argumentation graph is constructed from a  defeasible theory then it is \subargumentcomplete and \rulecomplete. 
\end{myProposition}
\begin{proof}
From \Definitionref \ref{defintion:constructed_from},  an \emph{argumentation graph} $G$ constructed from a defeasible theory $T$ is a tuple $\tuple{\AR,\defeat,\support}$ where $\AR$ is the set of \emph{all} arguments constructed from $T$. Thus, $\AR$ is \subargumentcomplete and \rulecomplete. Therefore, $G$ is  \subargumentcomplete and \rulecomplete.
\end{proof}
\noindent As we will see,  one may be interested in building a set of  \subargumentcomplete and \rulecomplete subgraphs  by considering \primaryquote{subtheories} of a defeasible theory. 
\begin{myDefinition}[Defeasible subtheory]
\label{def:pdt}
Let $T = \tuple{\Trule, \Tconf,\Tsup}$ denote a defeasible theory.
Given a set  $\Trule' \subseteq \Trule$, the \emph{defeasible subtheory $\subtheory$ of $T$} induced by $\Trule'$ is the defeasible theory $\subtheory = \tuple{\Trule',\Tconf,\Tsup \cap (\Trule' \times \Trule')}$.
\end{myDefinition}
We will use such subtheories later in our probabilistic investigations.
\begin{myNotation}

The set of all the subtheories of a theory $T = \tuple{\Trule,\Tconf,\Tsup}$ is  denoted $\Sub(T)$, i.e. 
$\Sub(T) = \set{ \tuple{\Trule',\Tconf,\Tsup \cap (\Trule' \times \Trule')} \mid \Trule' \subseteq \Trule}.$\smallskip
\end{myNotation}

To recap, starting from a defeasible theory, we can build arguments  and identify \RR{the attack and subargument} relationships over arguments, \RRR{thus creating} an argumentation graph  including its subgraphs. \PB{The acceptance and justification status of the arguments in a graph needs then to be evaluated}. \RRR{The next} section discusses how these status\PB{es} can  be  determined by using labellings.

\subsection{Labelling of arguments}
\label{subsection:labellingarguments}

Given an argumentation graph, the acceptance of arguments is evaluated on the basis of  a formal \RRT{specification} traditionally called \emph{argumentation semantics}. This evaluation can be carried out in terms of sets of arguments, called \emph{extensions}, as originally proposed in \cite{dung95acceptability},
or in terms of labellings \cite{DBLP:journals/ker/BaroniCG11}.
\RRR{ The idea underlying the extension-based approach is to identify sets of arguments, called extensions, \PBB{which can collectively survive the conflict and therefore represent a reasonable position. In general, several alternative extensions exist. The idea underlying the labelling-based approach is to describe each reasonable position by a labelling, namely by assigning to each argument a label taken from a given set.}}
Traditional extension-based semantics have an equivalent labelling-based characterization using the well-known set of three labels $\{\IN, \OUT, \UND\}$ which we also adopt here. \PBB{Briefly, for each extension $E$ the corresponding labelling is defined as follows: each argument belonging to $E$ \PBX{is} labelled $\IN$, each argument attacked by the extension is labelled $\OUT$, every other argument is labelled $\UND$.}
Other approaches differing in the set of adopted labels and/or in the underlying semantics have been \RRT{explored} in the literature (see e.g. \cite{Pollock:1995,Verheij96twoapproaches,DBLP:journals/logcom/JakobovitsV99,Vreeswijk:2006:ACM:1565233.1565247}).
\PB{Our proposal is not bounded to} the choice of a specific set of labels or semantics and can be extended to other approaches too, a full development in this direction is left to future work.


As a first step, we introduce a pair of labels devoted to capture a situation of uncertainty about the presence of an argument within a framework.
Traditionally the set of three labels $\{\IN, \OUT, \UND\}$ is used to characterise the acceptance status of arguments on the basis of a given argumentation semantics. However, in our context of uncertainty, a labelling needs also to be able to express the fact that an argument is actually included and has some effect (let say it is \primaryquote{$\ON$}) in the framework or that it is actually not included and hence has no effect (let say it is \primaryquote{$\OFF$}).
Then, as we will see, \primaryquote{classic} acceptance statuses \primaryquote{$\IN$}, \primaryquote{$\OUT$} and \primaryquote{$\UND$} of arguments and the $\ON$-$\OFF$ view can be seamlessly combined. This combination will allow us to clearly distinguish the probability of \primaryquote{construction} or \primaryquote{inclusion} of an argument and the probability of its acceptance.

Let us first introduce the basic formal notions and the relevant notation.

\begin{myDefinition} [Labelling of a set of arguments]
Let $G$ be an argumentation graph, and $\llabels$ a set of labels for arguments.  An \emph{$\llabels$-labelling} of a set of arguments $ \argset \subseteq \AR_G$ is a total function $\labelling: \argset \rightarrow \llabels$. 
\end{myDefinition}

\begin{myNotation}
Given an argumentation graph $G$, the universe of all possible $\llabels$-labelling assignments of a set $\argset \subseteq \AR_G$ is denoted  $\setofalllabels{\argset}{\llabels}$.
\end{myNotation}

In general a labelling involves an arbitrary set of arguments. When this set is a singleton, with a little abuse of language we also speak  of labelling of an argument. Labellings involving all the arguments of an argumentation graph play a special role and deserve a specific terminology and notation.

\begin{myDefinition} [Labelling of an argumentation graph]
\label{ne}
Let $G$ be an argumentation graph, and $\llabels$ a set of labels for arguments.  An \emph{$\llabels$-labelling} of $G$ is a total function $\labelling: \AR_G \rightarrow \llabels$. 
\end{myDefinition}

\begin{myNotation}
Given an argumentation graph $G$, the universe of all possible $\llabels$-labellings of $G$ is  denoted  $\setofalllabels{G}{\llabels}$.
\end{myNotation}

\noindent\sloppy For instance, a \emph{$\{\IN, \OUT, \UND\}$-labelling} of $G$ is a total function $\labelling: \AR_G \rightarrow \{\IN, \OUT, \UND\}$.

\begin{myNotation}
The set of arguments labelled with a label $l$ by a labelling $\labelling$ is denoted $l(\labelling)$, i.e. $l(\labelling) = \{A|\labelling(A) = l\}$. For instance, if $\IN$ is a label, then  $\IN(\labelling) = \{A|\labelling(A) = \IN\}$.
\end{myNotation}

Not all labellings in $\setofalllabels{G}{\llabels}$ are meaningful or have satisfactory properties.
Moreover, even among meaningful labellings, one may identify those satisfying some specific requirements.
Hence typically, one needs to define some \RRT{criteria} to identify those labellings in $\setofalllabels{G}{\llabels}$ which are meaningful or interesting in some sense. Formally speaking, a \specification identifies a subset of $\setofalllabels{G}{\llabels}$ for every possible argumentation graph $G$.

\begin{myDefinition} [Labelling specification]\label{def:labellingspecification}
\sloppy Given a set of labels $\llabels$, an \emph{$X$-$\llabels$-labelling specification} identifies for every argumentation graph $G$ a set of $\llabels$-labellings of $G$ denoted as $\crit{X}{G}{\llabels} \subseteq \setofalllabels{G}{\llabels}$, where $X$ is called the criterion of the specification.
\end{myDefinition}

If a labelling $\labelling$ belongs to the set of some specified labellings  $\crit{X}{G}{\llabels}$, i.e. \smash{$\labelling \in \crit{X}{G}{\llabels} \subseteq \setofalllabels{G}{\llabels}$},  we say that the labelling $\labelling$ is an $X$-$\llabels$-labelling of $G$, or an  $X$ $\llabels$-labelling (without the hyphen).

\begin{myDefinition}[Mono- and multi-labelling specification]
\label{definition:multi}
An $X$-$\llabels$-labelling specification is a \emph{mono-labelling specification} if, and only if, $|\crit{X}{G}{\llabels}| = 1$ for any  argumentation graph $G$, and it is  a \emph{multi-labelling specification} if, and only if, $\smash{|\crit{X}{G}{\llabels}|> 1}$ for some $G$. 
\end{myDefinition}

\sloppy \RR{As an anticipation of} our probabilistic setting\RRT{,} we \RRT{employ} three types of labellings: $\{\ON,\OFF\}$-labellings, $\{\IN,\OUT,\UND\}$-labellings and $\{\IN,\OUT,\UND,\OFF\}$-labellings. 
In a $\{\ON,\OFF\}$-labelling, each argument is associated with one label which is either $\ON$ or $\OFF$ to indicate whether an argument \textcolor{black}{occurs, that is, whether an} argument is regarded as playing an active role or not in the graph. Intuitively, arguments labeled $\OFF$ can be ignored. In a $\{\IN,\OUT,\UND\}$-labelling, each argument is associated with one label representing its status in the context of a semantics-based evaluation \cite{DBLP:journals/ker/BaroniCG11}.
Intuitively, a label \primaryquote{$\IN$} means the argument is accepted, while a label \primaryquote{$\OUT$} indicates that it is rejected and  \primaryquote{$\UND$} that it is undecided, i.e. neither accepted nor rejected. 
A $\{\IN,\OUT,\UND,\OFF\}$-labelling  extends a $\{\IN,\OUT,\UND\}$-labelling with the $\OFF$ label to indicate that an argument does not occur. Dually, a $\{\IN,\OUT,\UND,\OFF\}$-labelling can be regarded as a refinement of a $\{\ON,\OFF\}$-labelling where occurring (i.e. $\ON$) arguments are the subject of semantics-based evaluation.

Having introduced basic intuitions on some types of labellings, we need now to specify how actual labellings of the various types can be built and which constraints they should satisfy. The simplest specification for a given set of labels is the one without constraints, denoted as the \ALL-$\llabels$-labelling specification.
\begin{myDefinition} [\ALL-$\llabels$-labelling specification]\label{def:ALL}
The \emph{\ALL-$\llabels$-labelling} of an argumentation graph $G$ is such that $\crit{\ALL}{G}{\llabels} = \setofalllabels{G}{\llabels}$.
\end{myDefinition}

Concerning $\{\ON, \OFF\}$-labellings, we will focus on specifications ensuring that the set of arguments labelled $\ON$ features the completeness properties previously discussed.
\begin{myDefinition}[X-{\footnotesize{$\{\mathsf{ON}, \mathsf{OFF}\}$}}-labelling specification]
\label{ne1Arg}
Let $G$ denote an argumentation graph.
\noindent A \emph{$\{\ON,\OFF\}$-labelling} $\labelling$ of $G$ is
\begin{itemize}
\item  \emph{\subargumentcomplete} if, and only if, $\ON(\labelling)$ is \subargumentcomplete;
\item  \emph{\rulecomplete} if, and only if, $\ON(\labelling)$ is \rulecomplete;
\item \emph{legal} if, and only if, $\ON(\labelling)$ is \subargumentcomplete and  \rulecomplete.
\end{itemize}
\end{myDefinition}

\begin{myExample}[continues=example:rulecomplete]
\RRR{A legal $\{\ON, \OFF\}$-labelling and a subargument-complete $\{\ON, \OFF\}$-labelling (which is not legal because it is not rule-complete) are pictured in \Figureref \ref{figure:ONOFFexample}.}

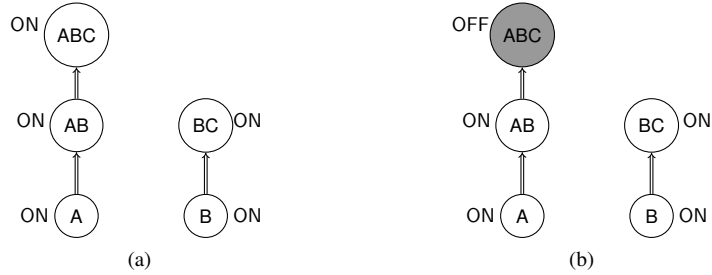
\begin{figure}[ht!]
\centering
\subfloat[]{{\begin{tikzpicture}[node distance=1.7cm,main node/.style={circle,fill=white!20,draw,font=\sffamily\scriptsize}]

 \node[main node] [label={[xshift=-0.55cm,yshift=-0.5cm]$\ON$}] (4) [] {A};
 \node[main node] [label={[xshift=-0.6cm,yshift=-0.5cm]$\ON$}] (5) [above of=4, yshift=-5mm] {AB};
  \node[main node] [label={[xshift=-0.68cm,yshift=-0.5cm]$\ON$}] (6) [above of=5, , yshift=-5mm] {ABC};
 \node[main node] [label={[xshift=0.55cm,yshift=-0.5cm]$\ON$}] (7) [right of=4] {B};
 \node[main node] [label={[xshift=0.55cm,yshift=-0.5cm]$\ON$}] (8) [above of=7, yshift=-5mm] {BC};
   
 \path[->,>=stealth',shorten >=1pt,auto, thick,every node/.style={font=\sffamily\small}]; 
 
  \path[->,>=implies]
 (4) edge [double] node {} (5)
 (5) edge [double] node {} (6)
 (7) edge [double] node {} (8);
  
\end{tikzpicture}
}}
\hspace{2cm}
\subfloat[]{{
\begin{tikzpicture}[node distance=1.7cm,main node/.style={circle,fill=white!20,draw,font=\sffamily\scriptsize}]

 \node[main node] [label={[xshift=-0.55cm,yshift=-0.5cm]$\ON$}] (4) [] {A};
 \node[main node] [label={[xshift=-0.6cm,yshift=-0.5cm]$\ON$}] (5) [above of=4, yshift=-5mm] {AB};
  \node[main node] [label={[xshift=-0.68cm,yshift=-0.5cm]$\OFF$}] (6) [above of=5, , yshift=-5mm, fill = black!40] {ABC};
 \node[main node] [label={[xshift=0.55cm,yshift=-0.5cm]$\ON$}] (7) [right of=4] {B};
 \node[main node] [label={[xshift=0.6cm,yshift=-0.5cm]$\ON$}] (8) [above of=7, yshift=-5mm] {BC};
   
 \path[->,>=stealth',shorten >=1pt,auto, thick,every node/.style={font=\sffamily\small}]; 
 
  \path[->,>=implies]
 (4) edge [double] node {} (5)
 (5) edge [double] node {} (6)
 (7) edge [double] node {} (8);
 
\end{tikzpicture}
}}
\caption{\RRR{A legal {\scriptsize{$\{\mathsf{ON},  \mathsf{OFF}\}$}}-labelling (a), and a subargument-complete {\scriptsize{$\{\mathsf{ON}, \mathsf{OFF}\}$}}-labelling which is not legal (b) .}}
\label{figure:ONOFFexample}
\end{figure}

\hfill $\square$
\end{myExample}

We can note that $\{\ON, \OFF\}$-labellings can be mapped to subgraphs: intuitively they correspond to \primaryquote{switching off} arguments outside a given subgraph.

\begin{myDefinition} [{\footnotesize{$\{\mathsf{ON}, \mathsf{OFF}\}$}}-labelling with respect to a subgraph]
\label{def:onoff1sdfs} 
Let $G$ denote an argumentation graph, and $H$ a subgraph of $G$. 
The \emph{$\{\ON,\OFF\}$-labelling  of $G$ with respect to $H$}, denoted as $\labelling_{G,H}$, is such that for every argument $A$ in $\mathcal{A}_G$, it holds that:
\begin{itemize}
\item $A$ is labelled $\ON$ if, and only if, $A$ is in $\mathcal{A}_H$, \RRT{and}
\item $A$ is labelled $\OFF$ otherwise.
\end{itemize}
\end{myDefinition}

\begin{myProposition}
\label{proposition:basicccx}
Let $G$ denote  an argumentation graph  constructed from a defeasible theory $T$, and $H$ any subgraph  constructed from a subtheory $U$ of $T$, i.e. $U \in \Sub(T)$. The $\{\ON, \OFF\}$-labelling of $G$ with respect to $H$  is \emph{legal}. 
\end{myProposition}
\begin{proof}
If an argumentation (sub)graph $H$ is constructed from a defeasible (sub)theory then $H$ is \subargumentcomplete and \rulecomplete (\Propositionref \ref{proposition:basiccc}). Thus, $\mathcal{A}_H$ is subargument and \rulecomplete (\Definitionref  \ref{definition:subargumentgraph} and \ref{definition:rulecompletegraph}). Consequently, every $\{\ON, \OFF \}$-labelling $\labelling$ of $G$ with respect to $H$ is such that $\ON(\labelling)$ is subargument and \rulecomplete. 
Therefore, every $\{\ON, \OFF \}$-labelling $\labelling$ of $G$ with respect to $H$ is legal (\Definitionref \ref{ne1Arg}) .
\end{proof}

Turning to $\{\IN,\OUT,\UND\}$-labellings\RRT{,} we follow the well-known labelling-based approach to argumentation semantics reviewed in \cite{DBLP:journals/ker/BaroniCG11}. In a nutshell, given an argumentation graph $G$ and a set of labels, an argumentation semantics associates with the graph $G$ a subset
 of the set of all possible $\{\IN,\OUT,\UND\}$-labellings of the graph $G$. 
Note that, in our generalised setting, argumentation semantics can be introduced as an instance of labelling specification (\Definitionref \ref{def:labellingspecification}).

\begin{myDefinition} [Argumentation semantics]
An \emph{argumentation semantics} is an $\textcolor{black}{X}$-$\{\IN,\OUT,\UND\}$-labelling specification.
\end{myDefinition}

Every argumentation semantics $X$ in the literature satisfies the property of being conflict-free (abbreviated \CF), that is, for every argumentation graph $G$ $\smash{\crit{X}{G}{\{\mathsf{IN},\mathsf{OUT},\mathsf{UN}\}} \subseteq \crit{\CF}{G}{\{\mathsf{IN},\mathsf{OUT},\mathsf{UN}\}}}$ as defined below.
\begin{myDefinition} [Conflict-free {\footnotesize{$\{\mathsf{IN}, \mathsf{OUT}, \mathsf{UN}\}$}}-labelling specification]
\sloppy\quad A \emph{conflict-free\, $\{\IN, \OUT, \UND\}$-labelling} (or \CF-$\{\IN, \OUT, \UND\}$-labelling) of an argumentation graph $G$ is a $\{\IN, \OUT, \UND\}$-labelling such that for every argument $A$ in $\AR_G$ it holds that $A$ is labelled $\IN$ if, and only if, all attackers of $A$ are not labelled $\IN$.
\end{myDefinition}

Several literature semantics (and in particular all the semantics introduced in \cite{dung95acceptability}) are based on \emph{complete} semantics: its definition as a specification for complete labellings is recalled below.

\begin{myDefinition} [Complete {\footnotesize{$\{\mathsf{IN}, \mathsf{OUT}, \mathsf{UN}\}$}}-labelling specification]
\label{}

\sloppy A \emph{complete $\{\IN,$ $\OUT, \UND\}$-labelling} of an argumentation graph $G$ is a $\{\IN, \OUT, \UND\}$-labelling such that for every argument $A$ in $\AR_G$ it holds that:
\begin{itemize}
\item $A$ is labelled $\IN$ if, and only if, all attackers of $A$ are $\OUT$, \RRT{and}
\item $A$ is labelled $\OUT$ if, and only if, $A$ has an attacker $\IN$.
\end{itemize}
\end{myDefinition}
Since a complete $\{\IN, \OUT, \UND\}$-labelling is total, if an argument is neither labelled $\IN$ nor $\OUT$, then this argument is labelled $\UND$.

In general an argumentation graph may have several complete $\{\IN, \OUT, \UND\}$-labellings: complete-based semantics, like the grounded, preferred and stable semantics, recalled below, provide additional specifications among the complete labellings.

\begin{myDefinition} [Grounded {\footnotesize{$\{\mathsf{IN}, \mathsf{OUT}, \mathsf{UN}\}$}}-labelling specification]
\label{} \item
\sloppy A \emph{grounded $\{\IN, \OUT, \UND\}$-labelling} $\labelling$ of an argumentation graph $G$ is a complete $\{\IN, \OUT, \UND\}$-labelling of $G$ such that $\IN(\labelling)$ is minimal (w.r.t. set inclusion) among all complete $\{\IN, \OUT, \UND\}$-labellings of $G$.
\end{myDefinition}
\begin{myDefinition} [Preferred {\footnotesize{$\{\mathsf{IN}, \mathsf{OUT}, \mathsf{UN}\}$}}-labelling specification]
\label{}\item
\sloppy A \emph{preferred $\{\IN, \OUT, \UND\}$-labelling} $\labelling$ of an argumentation graph $G$ is a complete $\{\IN, \OUT, \UND\}$-labelling of $G$ such that $\IN(\labelling)$ is maximal (w.r.t. set inclusion) among all complete $\{\IN, \OUT, \UND\}$-labellings of $G$.
\end{myDefinition}
\begin{myDefinition} [Stable {\footnotesize{$\{\mathsf{IN}, \mathsf{OUT}, \mathsf{UN}\}$}}-labelling specification]
\label{}\item
\sloppy A \emph{stable $\{\IN, \OUT, \UND\}$-labelling} $\labelling$ of an argumentation graph $G$ is a complete $\{\IN, \OUT, \UND\}$-labelling of $G$ such that $\UND(\labelling)$ is empty.
\end{myDefinition}

\begin{myExample}[continues=example:basis]
\RRR{A grounded $\{\IN, \OUT, \UND\}$-labelling  which is also  a preferred and stable $\{\IN, \OUT, \UND\}$-labelling
is pictured in \Figureref \ref{figure:INOUTUNDdexample}.}
\begin{figure}[ht!]
\centering
\begin{tikzpicture}[node distance=1.7cm,main node/.style={circle,fill=white!20,draw,font=\sffamily\scriptsize}]

 \node[main node]  [label={$\IN$}] (4) [fill = green!40] {B};
 \node[main node] (5)  [label={$\OUT$}]  [right of=4, , fill = red!40] {C};
 \node[main node] (6) [label={$\IN$}] [right of=5, fill = green!40] {D};
 \node[main node] (7) [label={$\IN$}] [below right of=4, fill = green!40] {B2};
 \node[main node] (8) [label={$\IN$}] [below left of=4, , fill = green!40] {B1};

 \path[->,>=stealth',shorten >=1pt,auto, thick,every node/.style={font=\sffamily\small}]
 (4) edge node {} (5)
 (5) edge node {} (6)
 (6) edge node [right] {} (5);
 
  \path[->,>=implies]
 (8) edge [double] node {} (4)
 (7) edge [double] node {} (4);
 
\end{tikzpicture}
\caption{\RRR{A grounded {\scriptsize{$\{\mathsf{IN}, \mathsf{OUT}, \mathsf{UN}\}$}}-labelling which is also a preferred and stable {\scriptsize{$\{\mathsf{IN}, \mathsf{OUT}, \mathsf{UN}\}$}}-labelling.}}
\label{figure:INOUTUNDdexample}
\end{figure}
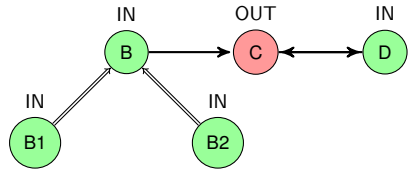
%

\hfill $\square$
\end{myExample}

The traditional literature distinction between \emph{single status} and \emph{multiple status} argumentation semantics corresponds to our notions of \emph{mono-labelling} and \emph{multi-labelling} respectively (see Definition \ref{definition:multi}). For example, the   grounded $\{\IN, \OUT, \UND\}$-labelling semantics is mono-labelling, while the preferred $\{\IN, \OUT, \UND\}$-labelling semantics is multi-labelling. 

\sloppy A few other semantics may  be considered, see \cite{DBLP:journals/ker/BaroniCG11}, but this  is beyond the scope of the present paper.\\

We are now ready to introduce \PBB{the novel} $\{\IN, \OUT, \UND, \OFF\}$-labellings as a way to combine a legal $\{\ON, \OFF\}$-labelling with $\{\IN, \OUT, \UND\}$-labellings of the subgraph induced by the set of arguments labelled  $\ON$.
Intuitively, this corresponds to combine together legality and semantics criteria.

\begin{myDefinition} [$X$-{\footnotesize{$\{\mathsf{IN}, \mathsf{OUT}, \mathsf{UN}, \mathsf{OFF}\}$}}-labelling specification]
Let $G$ denote an argumentation graph, and $H$ a \subargumentcomplete subgraph of  $G$.
A $\{\IN, \OUT, \UND, \OFF \}$-labelling of $G$ is  an \emph{$X$-$\{\IN, \OUT, \UND, \OFF \}$-labelling} with respect to $H$ if, and only if:
\begin{itemize}
\item every argument in $\mathcal{A}_H$ is labelled according to an $X$-$\{\IN, \OUT, \UND\}$-labelling of $H$, and
\item every argument in $\mathcal{A}_G \backslash \mathcal{A}_H $ is labelled $\OFF$.
\end{itemize}
As for terminology, the  $X$-$\{\IN, \OUT, \UND\}$-labelling is called the sublabelling of the  $X$-$\{\IN, \OUT, \UND, \OFF\}$-labelling specification.  
\label{offlabellingArgGen}
\end{myDefinition}

The above definitions can be instantiated. For example,  \Definitionref \ref{offlabellingArgGen} can be used to define the grounded $\{\IN, \OUT, \UND, \OFF\}$-labelling of an argumentation graph.

\begin{myDefinition} [Grounded {\footnotesize{$\{\mathsf{IN}, \mathsf{OUT}, \mathsf{UN}, \mathsf{OFF}\}$}}-labelling specification]\quad Let $H$ be a \subargumentcomplete subgraph of an argumentation graph $G$.
A \emph{grounded $\{\IN, \OUT, \UND, \OFF \}$-labelling }of $G$ with respect to $H$ is a $\{\IN, \OUT, \UND, \OFF \}$-labelling such that:
\begin{itemize}
\item every argument in $\mathcal{A}_H$ is labelled according to the grounded $\{\IN, \OUT, \UND\}$-labelling of $H$, and
\item every argument in $\mathcal{A}_G \backslash \mathcal{A}_H $ is labelled $\OFF$.
\end{itemize}
\label{offlabellingArg}
\end{myDefinition}

\begin{myExample}[continues=example:basis]
A grounded $\{\IN, \OUT, \UND, \OFF\}$-labelling  and  a preferred (and stable) $\{\IN, \OUT, \UND, \OFF\}$-labelling
are pictured in \Figureref \ref{figure:groundedexample}.

\begin{figure}[ht!]
\centering
\subfloat[]{{\begin{tikzpicture}[node distance=1.7cm,main node/.style={circle,fill=white!20,draw,font=\sffamily\scriptsize}]

 \node[main node]  [label={$\OFF$}] (4) [fill = black!35] {B};
 \node[main node] (5)  [label={$\UND$}]  [right of=4, , fill = blue!40] {C};
 \node[main node] (6) [label={$\UND$}] [right of=5, fill = blue!40] {D};
 \node[main node] (7) [label={$\OFF$}] [below right of=4, fill = black!40] {B2};
 \node[main node] (8) [label={$\IN$}] [below left of=4, , fill = green!40] {B1};

 \path[->,>=stealth',shorten >=1pt,auto, thick,every node/.style={font=\sffamily\small}]
 (4) edge node {} (5)
 (5) edge node {} (6)
 (6) edge node [right] {} (5);
 
  \path[->,>=implies]
 (8) edge [double] node {} (4)
 (7) edge [double] node {} (4);
 
\end{tikzpicture}
}}
\hspace{0.8cm}
\subfloat[]{{
\begin{tikzpicture}[node distance=1.7cm,main node/.style={circle,fill=white!20,draw,font=\sffamily\scriptsize}]

 \node[main node] (4)  [label={$\OFF$}]  [fill = black!35] {B};
 \node[main node] (5)  [label={$\OUT$}]  [right of=4, , fill = red!40] {C};
 \node[main node]  [label={$\IN$}]  (6) [right of=5, fill = green!40] {D};
 \node[main node]  [label={$\OFF$}]  (7) [below right of=4, fill = black!40] {B2};
 \node[main node]  [label={$\IN$}] (8) [below left of=4, , fill = green!40] {B1};

 \path[->,>=stealth',shorten >=1pt,auto, thick,every node/.style={font=\sffamily\small}]
 (4) edge node {} (5)
 (5) edge node {} (6)
 (6) edge node [right] {} (5);
 
  \path[->,>=implies]
 (8) edge [double] node {} (4)
 (7) edge [double] node {} (4);
 
\end{tikzpicture}
}}
\caption{A grounded {\scriptsize{$\{\mathsf{IN}, \mathsf{OUT}, \mathsf{UN}, \mathsf{OFF}\}$}}-labelling (a), and a preferred  {\scriptsize{$\{\mathsf{IN}, \mathsf{OUT}, \mathsf{UN}, \mathsf{OFF}\}$}}-labelling (b). }
\label{figure:groundedexample}
\end{figure}
%
\hfill $\square$
\end{myExample}

An argumentation graph $G$ has a unique grounded $\{\IN, \OUT, \UND\}$-labelling, but it has as many grounded $\{\IN, \OUT, \UND, \OFF\}$-labellings as subgraphs of $G$. Therefore, the $\mathsf{grounded}$-$\{\IN, \OUT, \UND, \OFF\}$-labelling specification is in general multi-labelling. More generally, the   $X$-$\{\IN, \OUT, \UND, \OFF\}$-labelling specification is multi-labelling, even though its sublabelling semantics is mono-labelling.\medskip

\PB{Argument  labellings produced by the semantics evaluation are called} argument \emph{acceptance} labellings. On the basis of argument acceptance labellings, it is possible to aggregate the acceptance statuses into a justification status for every considered argument, resulting into  argument \emph{justification} labellings.

In non-probabilistic settings,  the justification status of an argument is usually determined   with respect to all the labellings in the set  of labellings \RRT{$\crit{X}{G}{\llabels}$}  identified by an $X$-$\llabels$-labelling specification, by using the traditional notions of skeptical and credulous justification. 
In particular, an argument is \emph{skeptically justified} if it is labelled $\IN$ by all labellings in $\crit{X}{G}{\llabels}$, while it is \emph{credulously justified} if it is labelled $\IN$ by at least one labelling in $\smash{\crit{X}{G}{\llabels}}$, and \RRT{it} is \emph{not justified} otherwise.

This simple notion of argument justification can be generalised considering a generic set of justification states, represented by a set of labels $\jarglabels$, and a function producing a $\jarglabels$-labelling of $G$ on the basis of the labellings $\crit{X}{G}{\llabels}$. The traditional approach can be regarded as an instance of this general scheme where $\jarglabels=\set{\SKJ, \CRJ, \NOJ}$ and the function is as specified above.
The reader may refer to \cite{DBLP:conf/ecai/BaroniGR16} \RRT{for a wider formal treatment of this \RRT{generalised} labelling approach}. \RRT{We} focus here only on the aspects which are necessary for the development of \RRT{the present} paper. In particular, the extension of the set of argument labels to $\{\IN, \OUT, \UND, \OFF\}$ calls for an extended set of argument justification labels too.
%
%
%
\PB{As a minimum, since an argument can be always ignored (labelled $\OFF$ by all $\crit{X}{G}{\llabels}$) an additional justification label is needed to cover this case.}
On this basis, we introduce the \primaryquote{\gullible} $\{\OFFJ, \SKJ, \CRJ, \NOJ \}$-labelling. 

\begin{myDefinition}[\Gullible {\footnotesize{$\{\mathsf{\RT{OFJ}}, \mathsf{SKJ}, \mathsf{CRJ}, \mathsf{NOJ} \}$}}-labelling]
Let $\mathcal{L}$ denote a non-empty set of $\set{\IN, \OUT, \UND, \OFF}$-labellings of an argumentation graph $G$. 
The \emph{\gullible $\{\OFFJ, \SKJ, \CRJ, \NOJ \}$-labelling $\labelling_\mathsf{J}$ of $G$ with   respect to $\mathcal{L}$} is a $\{\OFFJ, \SKJ, \CRJ, \NOJ \}$-labelling such that for every argument $A$ in $\mathcal{A}_G$, it holds that:
\begin{itemize}
\item $A$ is labelled $\OFFJ$, i.e. $\labelling_\mathsf{J}(A) = \OFFJ$,  if, and only if,  $\forall \labelling \in \mathcal{L}$, $\labelling(A) = \OFF$;
\item $A$ is labelled $\SKJ$, i.e. $\labelling_\mathsf{J}(A) = \SKJ$,  if, and only if,  $\forall \labelling \in \mathcal{L}$, $\labelling(A) = \IN$;
\item $A$ is labelled $\CRJ$, i.e. $\labelling_\mathsf{J}(A) = \CRJ$ if, and only if, \PB{$\exists \labelling \in \mathcal{L}:\labelling(A)= \IN$ and $\labelling_\mathsf{J}(A) \neq \SKJ$};
\item otherwise $A$ is labelled $\NOJ$, i.e.  $\labelling_\mathsf{J}(A) = \NOJ$.
\end{itemize}
\label{def:argumentJustification}
\end{myDefinition}

The \primaryquote{\gullible} $\{\OFFJ, \SKJ, \CRJ, \NOJ \}$-labelling is only one of the many options in the design space of argument justification labelling\RRT{s}. Its application is exemplified below, while further investigations about alternative notions of argument justification in the extended context we are exploring are left to future work. \PBX{In this respect, we remark} that introducing uncertainty in argumentation requires revising the formalisation in all the phases of the process, some of which have received relatively limited attention until now.

\begin{myExample}[continues=example:basis]
\sloppy \RRT{Assume} the set of  $\{\IN, \OUT, \UND, \OFF \}$-labellings $\mathcal{L}$ as previously illustrated in Figure \ref{figure:groundedexample} (two labellings). The semi-skeptical {\scriptsize{$\{\mathsf{\RT{OFJ}}, \mathsf{SKJ}, \mathsf{CRJ}, \mathsf{NOJ}\}$}}-labelling with respect to  $\mathcal{L}$ is illustrated in Figure \ref{figure:Semiskepexample}. 
%
%

\begin{figure}[ht!]
\centering
\begin{tikzpicture}[node distance=1.7cm,main node/.style={circle,fill=white!20,draw,font=\sffamily\scriptsize}]

 \node[main node]  [label={$\OFFJ$}] (4) [fill = black!40] {B};
 \node[main node] (5)  [label={$\NOJ$}]  [right of=4, fill = purple!40] {C};
 \node[main node] (6) [label={$\CRJ$}] [right of=5, fill = green!20] {D};
 \node[main node] (7) [label={$\OFFJ$}] [below right of=4, fill = black!40] {B2};
 \node[main node] (8) [label={$\SKJ$}] [below left of=4, , fill = green!60] {B1};

 \path[->,>=stealth',shorten >=1pt,auto, thick,every node/.style={font=\sffamily\small}]
 (4) edge node {} (5)
 (5) edge node {} (6)
 (6) edge node [right] {} (5);
 
  \path[->,>=implies]
 (8) edge [double] node {} (4)
 (7) edge [double] node {} (4);
 
\end{tikzpicture}
\caption{\RRR{A semi-skeptical {\scriptsize{$\{\mathsf{OFFJ}, \mathsf{SKJ}, \mathsf{CRJ}, \mathsf{NOJ}\}$}}-labelling.}}
\label{figure:Semiskepexample}
\end{figure}\hfill$\square$
\end{myExample}
The \gullible $\{\OFFJ, \SKJ, \CRJ, \NOJ \}$-labelling is obviously total because  the label $\NOJ$ covers all cases not covered by the other labels, but we can nevertheless define $\NOJ$ in more specific terms, as follows.

\begin{myProposition}
\label{prop:NOJ}
Let $\mathcal{L}$ denote a non-empty set of observed argument labellings. 
 It holds that $\labelling_{\mathsf{J}}(A) = \NOJ$ if, and only if, $\exists \labelling \in \mathcal{L}, \labelling(A) = \OUT \mbox{ or } \labelling(A) = \UND $ and $\forall \labelling \in \mathcal{L}, \labelling(A) \neq \IN$.
\end{myProposition}

\medskip

To recap, given an argumentation graph, the acceptance and justification statuses of arguments  can be \PB{expressed} according to a variety of labelling specifications. To encompass uncertainty about the inclusion/omission of arguments, we introduced two additional labels indicating  inclusion (\primaryquote{\ON}) or omission (\primaryquote{\OFF}). We showed how existing labelling specifications can be gently adapted in that regard\RRT{,} and we introduced an upgraded notion of argument justification. On this basis, statements can be labelled in turn, as we will see next.

\subsection{Labelling of statements}
\label{subsection:labellingstatements}

The final phase of the process concerns the labelling of statements, and in particular the labelling of the conclusions supported by arguments. As a matter of fact, assessing conclusions is the ultimate goal of the argumentation process.
 For instance, in Example \ref{example:basis} the research scientist wants to form an opinion about whether the program is running or not.
In the remainder of this section, we focus on the labelling of statements which are, in our case, literals.

Labellings of statements can be performed in different manners, see e.g.  \cite{DBLP:conf/ecai/BaroniGR16}. From an abstract point of view,  given a set of statements $-$ where a statement may not be the conclusion of any argument $-$ a labelling of this set  is a function associating every statement with a label.

%
\begin{myDefinition}[Labelling of literals]
Let $\Phi$ denote a set of literals, and $\klabels$ denote a set of labels on literals.
A \emph{$\klabels$-labelling of $\Phi$} is a function $\lablit:  \Phi \rightarrow \klabels$.
\end{myDefinition}

Per se,  a labelling of literals is just a function mapping a set of literals to a set of labels, but such a labelling may rely on an acceptance labelling of arguments. For this reason, we introduce an acceptance labelling function as follows.

\begin{myDefinition}[Acceptance labelling of literals]
\sloppy Let $G$ denote an argumentation graph,
$\crit{X}{G}{\llabels}$  the set of $X$-$\llabels$-labellings of the argumentation graph $G$,  $\Phi$  a set of literals, and $\klabels$  a set of labels on literals.
An \emph{acceptance $\klabels$-labelling of $\Phi$} is a function $\lablit: \crit{X}{G}{\llabels} \PBX{\times} \Phi \rightarrow \klabels$.
\end{myDefinition}

\begin{myExample}
If we write  $\lablit(\labelling, \lit) = \inn$, then it means that, given the argument labelling $\labarg$, the literal $\lit$ is labelled $\inn$.
\hfill $\square$
\end{myExample}

It is often desirable that the considered labelling function  is total, and we  assume  total labelling functions in the remainder.\smallskip

Different specifications for labellings of statements are possible. A simple labelling is the so-called \emph{bivalent labelling}, according to which a statement is either justified or not, without further sophistication. If a statement is accepted then it is labelled $\inn$, otherwise it is labelled $\no$. In this case, the necessary and sufficient condition for a statement to be labelled $\inn$ is to have at least one argument labelled $\IN$ supporting this statement.

\begin{myDefinition}[Bivalent labelling of literals] 
Let $G$ denote an argumentation graph, 
$\crit{X}{G}{\llabels}$ denote the set of $X$-$\llabels$-labellings of the argumentation graph $G$,  and $\Phi$ denote a set of literals.
A \emph{bivalent $\{\inn, \no \}$-labelling of $\Phi$} is a total function $\lablit: \crit{X}{G}{\llabels}, \Phi \rightarrow \{\inn, \no \}$,  such that $\forall \lit \in \Phi$: 
\begin{itemize}
\item   $\lit$ is labelled $\inn$,  i.e. $\lablit(\labarg, \lit) = \inn$, if and only if $\exists A \in \IN(\labarg) : \Conc(A) = \lit$, i.e.  there is at least one argument $A$ such that
\begin{itemize}[noitemsep,nolistsep]
\item this argument $A$ is labelled $\IN$ in $\labarg$, and
\item the literal $\lit$ is the conclusion of argument $A$;
\end{itemize}
\item    $\lit$ is labelled $\no$ otherwise,  i.e $\lablit(\mathcal{L}, \lit) = \no$.
\end{itemize}
\end{myDefinition}
\begin{myExample}[continues=example:basis]
Consider the grounded $\{\IN, \OUT, \UND, \OFF\}$-labelling $\labelling$ as given in \Figureref \ref{figure:groundedexample}, the bivalent labelling of the set of literals supported by the arguments is such that $\lablit(\labelling, \litbone) = \inn$, $\lablit(\labelling, \litbtwo) = \no$,  $\lablit(\labelling, \litb) = \no$, $\lablit(\labelling, \litc) = \no$, and $\lablit(\labelling, \neg\litc) = \no$.
\hfill $\square$
\end{myExample}

The bivalent labelling is simple, but it takes no advantage of the fine-grained labelling of  arguments. For this reason, we may consider more sophisticated labellings concerning the statuses of statements supported by no arguments labelled $\IN$. 
So, \PBB{following the same spirit underlying $\{\IN, \OUT, \UND, \OFF\}$-labelling for arguments, we propose   the novel \emph{\rough{} labellings} where the labels $\out$, $\und$, $\off$ and $\unp$ are considered together with $\inn$.}

\begin{myDefinition}[\Rough{} labelling of literals]
\label{def:roughlab}
\sloppy Let $G$ denote an argumentation graph, 
$\crit{X}{G}{\llabels}$  the set of $X$-$\llabels$-labellings of the argumentation graph $G$, and   $\Phi$ a set of literals.
A \emph{\rough{} $\{\inn, \out, \und, \off, \unp\}$-labelling of $\Phi$} is a total function $\lablit: \crit{X}{G}{\llabels}, \Phi \rightarrow \{\inn, \out, \und, \off, \unp\}$,  such that $\forall \lit \in \Phi $: 

\begin{itemize}
\item $\lit$ is labelled $\inn$,  i.e. $\lablit(\labarg, \lit) = \inn$, if, and only if, $\exists A \in \IN(\labarg) : \Conc(A) = \lit$, i.e.  there is at least one argument $A$ such that
\begin{itemize}[noitemsep,nolistsep]
\item this argument $A$ is labelled $\IN$ in $\labarg$, and
\item the literal $\lit$ is the conclusion of argument $A$.
\end{itemize}
\item  $\lit$ is labelled $\out$,  i.e. $\lablit(\labarg, \lit) = \out$, if, and only if, $\exists A \in \OUT(\labarg) : \Conc(A) = \lit$ and $\forall A: \Conc(A) = \lit$, $A \in \OUT(\labarg) \cup \OFF(\labarg)$, i.e.
\begin{itemize}[noitemsep,nolistsep]
\item there is at least one argument whose conclusion is $\lit$, which is not labelled $\OFF$ in $\labarg$, and 
\item all arguments whose conclusion is $\lit$ are labelled $\OUT$ or $\OFF$ in $\labarg$.
\end{itemize}
\item  $\lit$ is labelled $\und$,  i.e. $\lablit(\labarg, \lit) = \und$, if, and only if, $\not\exists A \in \IN(\labarg) : \Conc(A) = \lit$ and $\exists A \in \UND(\labarg) : \Conc(A) = \lit$, i.e. 
\begin{itemize}[noitemsep,nolistsep]
\item there is no argument labelled $\IN$ in $\labarg$ whose conclusion is $\lit$, and
\item there is at  least one argument labelled $\UND$ in $\labarg$ whose conclusion is $\lit$.
\end{itemize}
\item $\lit$ is labelled $\off$,  i.e. $\lablit(\labarg, \lit) = \off$, if, and only if, $\exists A: \Conc(A) = \lit$ and $\forall A: \Conc(A) = \lit$ $A \in \OFF(\labarg)$, i.e.
\begin{itemize}[noitemsep,nolistsep]
\item there is at least one argument in $\labarg$ whose conclusion is $\lit$, and 
\item all arguments whose conclusion is $\lit$ are labelled $\OFF$ in $\labarg$.
\end{itemize}
\item   $\lit$ is labelled $\unp$,  i.e. $\lablit(\labarg, \lit) = \unp$, if, and only if, $\not\exists A: \Conc(A) = \lit$, i.e.  there is no argument labelled in $\labarg$ whose conclusion is $\lit$.
\end{itemize}
\end{myDefinition}
The  labels $\out$, $\und$, $\off$ and $\unp$ can be seen as a finer classification of the cases corresponding to the $\no$ label: with the label $\out$ all arguments concluding $\lit$ are either switched off or plainly rejected, with $\und$ at least one of them is undecided, with $\off$ they are all switched off. Finally, the label $\unp$ caters for queries about a statement for which there is no argument (we use the label $\unp$ as an abbreviation for \primaryquote{unprovable}). The name of this labelling reflects a distance from acceptance, so $\out$ is worse than $\und$ which is worse than $\off$, which is  worse than $\unp$.
\begin{myExample}[continues=example:basis]
Consider the grounded $\{\IN, \OUT,$ $\UND, \OFF\}$-labelling $\labelling$ as given in \Figureref \ref{figure:groundedexample}, the \rough{} labelling of the set of literals supported by the arguments is such that $\lablit(\labelling, \litbone) = \inn$, $\lablit(\labelling, \litbtwo) = \off$,  $\lablit(\labelling, \litb) = \off$, $\lablit(\labelling, \litc) = \und$, and $\lablit(\labelling, \neg\litc) = \und$.
\hfill $\square$
\end{myExample}

Other labels and types of labellings for statements can be considered, in particular a notion of statement justification can also be investigated,
but the labellings introduced here are sufficient to illustrate our probabilistic framework, as we will see in the next section.\\ 

\PB{To \RRT{summarise} \RRT{the} section, we have considered a minimalist but articulated formal framework covering the various phases of a rule-based argumentation process.}
 In this framework, arguments are constructed from defeasible theories by chaining defeasible rules. Some arguments can attack or \PB{be subarguments of} other arguments, and these relations amongst arguments are captured into abstract  argumentation graphs. \PB{Then, the acceptance status of arguments can be expressed in terms of labellings specified by an argumentation semantics. The acceptance labellings are in turn the basis for a labelling-based definition of argument justification and statement acceptance.}
 In \PB{addition} to existing labelling specifications, we considered two novel labels indicating the inclusion (\primaryquote{\ON}) or omission (\primaryquote{\OFF}) of arguments. We showed how existing labelling specifications can be gently adapted in that regard.  As we will see next, this framework will allow us to investigate and compare different probabilistic argumentation settings, leading us to a  labelling-oriented approach called probabilistic labellings.

\section{Probabilistic Labelling of Arguments and Statements}
\label{sec:probSetting}
\PB{Building on the framework introduced in Section \ref{sec:genlabelling}}, we are now ready to present, \RRR{as the main contribution of the paper}, a systematic approach to encompass probabilistic uncertainty in this context. In particular, we consider several possible answers to the question \primaryquote{which are the elements we are uncertain about?}, which, adopting  Kolmogorov axioms, amounts to identify various alternatives for the definition of the probability space. To this purpose, we will first introduce probabilistic argumentation frames (Subsection \ref{subsection:probarguments}), and then we will consider their use for  probabilistic labellings of arguments  (Subsection \ref{suibsection:arguments1}) and statements (Subsection \ref{subsection:probstatements}).  

\subsection{Probabilistic argumentation frames}
\label{subsection:probarguments}

Different probability spaces can \RRT{provide a} basis for probabilistic argumentation, each featuring different types of uncertainty. In that regard, we can distinguish probability spaces featuring some uncertainty on the argumentative structures, namely defeasible theories or  argumentation graphs, and probability spaces focusing on the uncertainty on the acceptance of arguments. \RRT{We investigate these alternatives in the remainder.}

\subsubsection{Structural uncertainty: probabilistic theory and graph frames}
\label{subsubsection:probarguments}

Since arguments are built on the basis of a defeasible theory, one may consider uncertainty at the level of the theory itself, \RRT{that is}, given a reference set of defeasible rules, uncertainty about which subset of rules to adopt.
In this view, the sample space corresponds to the set of subtheories of a  defeasible theory, leading to the definition of probabilistic theory frames (\PTFs).

\begin{myDefinition}[Probabilistic theory frame]
\sloppy A \emph{probabilistic theory frame} (\PTF) based on a defeasible theory $T$ is a tuple $\tuple{T, \tuple{\Omega_{\mPTF}, F_{\mPTF}, P_{\mPTF}}}$ where  $\tuple{\Omega_{\mPTF},F_{\mPTF},P_{\mPTF}}$ is a probability space such that:

\begin{itemize}
\item the sample space $\Omega_{\mPTF}$ is the set of subtheories of $T$, i.e. $\Omega_{\mPTF} =  \Sub(T)$;   
\item  the $\sigma$-algebra\footnote{\RRR{An algebra of sets is a collection \RRT{$\mathcal{C}$} of subsets of a given set $S$ such that 
(i) $S \in \RRT{\mathcal{C}}$, (ii) if $X \in \RRT{\mathcal{C}}$ and $Y \in \RRT{\mathcal{C}}$ then $X \cup Y \in \RRT{\mathcal{C}}$,
(iii) if $X \in \RRT{\mathcal{C}}$ then $S \backslash X \in  \RRT{\mathcal{C}}$. The collection \RRT{$\mathcal{C}$} is also closed under intersections. A $\sigma$-algebra is additionally closed under countable unions (and intersections): (iv) If $X_n \in \RRT{\mathcal{C}}$ for all $n$, then $\bigcup^{\infty}_{n=0} X_n \in \RRT{\mathcal{C}}$ \cite{Jech}. Note that, while we specify the $\sigma$-algebra as a power set, it may be specified differently. We adopt  here the most common specification, which appears to us as the most `natural' for our purposes.}} $F_{\mPTF}$ is the power set of $\Omega_{\mPTF}$, i.e. $F_{\mPTF} = 2^{\Omega_{\mPTF}}$;
\item  the  function $P_{\mPTF}$ from $F_{\mPTF}$ to $[0,1]$ is a probability distribution satisfying Kolmogorov axioms.
\end{itemize}
\label{definition:probabilistic_theory_frame}
\end{myDefinition}

\begin{myNotation}
The (infinite) set of the \PTFs based on a defeasible theory  $T$ (differing on the probability function $P$) is denoted  $\basedPTF{T}$.
\end{myNotation}

Every subtheory $\subtheory$ of a theory $T$ generates exactly one argumentation graph $\gengraph{\subtheory}$, thus uncertainty about the set of rules also corresponds to being uncertain about the \RRT{induced} argumentation graph among the subgraphs of $\gengraph{T}$.
For this reason, one may decide to adopt a representation at a more abstract level, where, given a reference argumentation graph, there is uncertainty about which of its subgraphs should be actually \PBX{taken into account}, leading to the definition of probabilistic graph frames (\PGFs).

\begin{myDefinition}[Probabilistic graph frame]
A \emph{probabilistic graph frame} (\PGF)  based on an argumentation graph $G$ is a tuple $\tuple{G, \tuple{\Omega_{\mPGF}, F_{\mPGF}, P_{\mPGF}}}$ where $\tuple{\Omega_{\mPGF},F_{\mPGF},P_{\mPGF}}$ is a probability space such that:
\begin{itemize}
\item the sample space $\Omega_{\mPGF}$ is the set of subgraphs of $G$, i.e. $\Omega_{\mPGF} =  \subgraphs{G}$;   
\item  the $\sigma$-algebra $F_{\mPGF}$ is the power set of $\Omega_{\mPGF}$,  i.e. $F = 2^{\Omega_{\mPGF}}$;
\item  the  function $P_{\mPGF}$ from $F_{\mPGF}$ to $[0,1]$ is a probability distribution satisfying Kolmogorov axioms.
\end{itemize}
\end{myDefinition}

\begin{myNotation}
The (infinite) set of the \PGFs based on an argumentation graph $G$ (differing on the probability function $P$) is denoted $\basedPGF{G}$.
\end{myNotation}

The relation from the subtheories of a theory $T$ to the subgraphs of $\gengraph{T}$ is neither injective nor surjective. 
Non injectivity is due to the fact that it may be the case $\gengraph{\subtheory}=\gengraph{\subtheory'}$ with $\subtheory \neq \subtheory' $ i.e. different subtheories may generate the same graph. This is due to the fact that subtheories may differ due to the presence of rules which are actually not used in the construction of any argument.
Non suriectivity is due to the fact that for every subtheory $\subtheory$, $\gengraph{\subtheory}$ is \subargumentcomplete and \rulecomplete, while $\subgraphs{\gengraph{T}}$ typically includes also subgraphs not satisfying these properties. 
This has an effect on the correspondence from \PTFs to \PGFs defined as follows.

\begin{myDefinition}[\PGF corresponding to  a \PTF]
\sloppy Given a \PTF $\tuple{T, \tuple{ \Omega_{\PTF},  F_{\mPTF}, P_{\mPTF}}}$, its corresponding \PGF is a \PGF tuple $\tuple{\gengraph{T}, \tuple{\Omega_{\mPGF},  F_{\mPGF}, P_{\mPGF}}}$  where the probability distribution  $P_{\mPGF}$ is such that  $\forall S \in  F_{\mPGF}$:\footnote{It is assumed that a sum over the empty set is 0, i.e. $P_{\mPGF}(S) = 0$ if $\set{\subtheory \in \Sub(T) \mid \gengraph{\subtheory} \in S}=\emptyset$.}
$$P_{\mPGF}(S) = \sum_{\set{\subtheory \in \Sub(T) \mid \gengraph{\subtheory} \in S}}P_{\mPTF}(\set{\subtheory}).$$
\label{defPGFfromPTF}
\end{myDefinition}
\begin{myNotation}
The \PGF corresponding to a \PTF $\probFrame$ is denoted $\genPGF{\probFrame}$. 
\end{myNotation}

Applying Definition \ref{defPGFfromPTF} to the case of singletons\RRT{,} we get that the distribution $P_{\mPGF}$ is such that for every subgraph $H$ in $\subgraphs{G}$,  
\begin{equation}P_{\mPGF}(\set{H})=\sum_{\set{\subtheory \in \Sub(T) \mid \gengraph{\subtheory} = H}}P_{\mPTF}(\set{\subtheory}).
\end{equation}
\PBB{Given that for every subtheory $\subtheory \in \Sub(T)$ \RRR{the graph} $\gengraph{\subtheory}$ is \subargumentcomplete and \rulecomplete, it follows also} that for every subgraph $H$ in  $\subgraphs{G}$ which is not \subargumentcomplete or \rulecomplete it holds that $P_{\mPGF}(\set{H})=0$.  \RRR{Hence,} the mapping  between the sample space of a \PTF and its corresponding \PGF   may be neither injective nor surjective.

\begin{myExample}[continues=example:rulecomplete]
\label{ex:neither}
Consider the \PTF $\tuple{\mathsf{T}, \tuple{\Omega_\mPTF, F_\mPTF, P_\mPTF}}$ and its corresponding \PGF $\tuple{\mathsf{G}_\mathsf{T}, \tuple{\Omega_\mPGF, F_\mPGF, P_\mPGF}}$.
\Figureref \ref{fig:samplespaces2} shows the mapping from $\Omega_\mPTF$ to $\Omega_\mPGF$, \PBB{where each graph is represented by the arguments belonging to it. The mapping  is clearly neither injective  nor surjective.} \PBB{An arbitrary probability distribution on the subtheories of $\mathsf{T}$ is given and the corresponding  distribution on the subgraphs of $\mathsf{G}_\mathsf{T}$ is shown.}
\begin{figure}[ht!]
\centering
  \begin{tikzpicture}[cell/.style={rectangle,draw=black}, nodes in empty cells]
  \matrix[
  matrix of math nodes,
  ](m)
  {  
  &    \mathsf{r_{1}} & \mathsf{r_{2}} &  \mathsf{r_{3}} & \mathsf{r_{4}}  &  P_\mPTF & & & & & &  & & & \,\mathsf{A}\, & \,\mathsf{B}\, & \mathsf{AB} & \mathsf{BC} & \mathsf{ABC} 	& P_\mPGF	\\ 
 \mathsf{U}_{1}&    1 & 1 &  1 & 1 & 1/16 & & & & & &  & &  \mathsf{H}_{1}& 1 & 1 & 1 & 1 & 1 & 1/16		\\ [-0.5ex]
 \mathsf{U}_{2}&    1 & 1 &  1 & |[fill=black!30]|0 & 2/16   & & & & & &  & & \mathsf{H}_{2}& 1 & 1 & 1 & 1 & |[fill=black!30]|0 & 0/16	\\ [-0.5ex]
  \mathsf{U}_{3}&   1 & 1 &  |[fill=black!30]|0 & 1 & 0/16  & & & & & &  & & \mathsf{H}_{3}& 1 & 1 & 1 & |[fill=black!30]|0 & 1  & 0/16   \\[-0.5ex]
  \mathsf{U}_4&     1 & 1 &  |[fill=black!30]|0 & |[fill=black!30]|0 & 1/16  & & & & & &  & & \mathsf{H}_{4}& 1 & 1 & 1 & |[fill=black!30]|0 & |[fill=black!30]|0  & 2/16  \\[-0.5ex]
 \mathsf{U}_5&      1 & |[fill=black!30]|0 &  1 & 1 & 1/16  & & & & & &  & & \mathsf{H}_{5}&  1 & 1 & |[fill=black!30]|0 & 1 & |[fill=black!30]|0  & 0/16  \\[-0.5ex]
 \mathsf{U}_6&      1 & |[fill=black!30]|0 &  1 & |[fill=black!30]|0 & 1/16  & & & & & &  & & \mathsf{H}_{6}&  1 & 1 & |[fill=black!30]|0 & |[fill=black!30]|0 & |[fill=black!30]|0  & 1/16    \\[-0.5ex]
 \mathsf{U}_7&      1 & |[fill=black!30]|0 &  |[fill=black!30]|0 & 1 & 0/16  & & & & & &  & & \mathsf{H}_{7}& 1 & |[fill=black!30]|0 & 1 & |[fill=black!30]|0 & 1   & 1/16   \\ [-0.5ex]
 \mathsf{U}_8&      1 & |[fill=black!30]|0 &  |[fill=black!30]|0 & |[fill=black!30]|0 & 2/16  & & & & & &  & & \mathsf{H}_{8}&  1 & |[fill=black!30]|0 & 1 & |[fill=black!30]|0 & |[fill=black!30]|0 & 1/16      \\ [-0.5ex]
 \mathsf{U}_9&      |[fill=black!30]|0 & 1 &  1 & 1 & 1/16  & & & & & &  & & \mathsf{H}_{9}&  1 & |[fill=black!30]|0 & |[fill=black!30]|0 & |[fill=black!30]|0 &  |[fill=black!30]|0  & 2/16 \\ [-0.5ex]
 \mathsf{U}_{10}&   |[fill=black!30]|0 & 1 &  1 & |[fill=black!30]|0 &1/16  & & & & & &  & & \mathsf{H}_{10}& |[fill=black!30]|0 & 1 & |[fill=black!30]|0 & 1 & |[fill=black!30]|0 & 2/16 \\[-0.5ex]
 \mathsf{U}_{11}&   |[fill=black!30]|0 & 1 &  |[fill=black!30]|0 & 1 & 1/16  & & & & & &  & & \mathsf{H}_{11}& |[fill=black!30]|0 & 1 & |[fill=black!30]|0 & |[fill=black!30]|0  & |[fill=black!30]|0 & 2/16 \\[-0.5ex]
  \mathsf{U}_{12}&  |[fill=black!30]|0 & 1 &  |[fill=black!30]|0 & |[fill=black!30]|0 & 1/16  & & & & & &  & &\mathsf{H}_{12}&  |[fill=black!30]|0 & |[fill=black!30]|0 & |[fill=black!30]|0 & |[fill=black!30]|0  & |[fill=black!30]|0 & 4/16 \\[-0.5ex]
  \mathsf{U}_{13}&  |[fill=black!30]|0 & |[fill=black!30]|0 &  1 & 1 & 0/16 & \\[-0.5ex]
  \mathsf{U}_{14}&  |[fill=black!30]|0 & |[fill=black!30]|0 &  1 & |[fill=black!30]|0 & 0/16 & \\[-0.5ex]
  \mathsf{U}_{15}&  |[fill=black!30]|0 & |[fill=black!30]|0 &  |[fill=black!30]|0 & 1 & 3/16 & \\[-0.5ex]
  \mathsf{U}_{16}&  |[fill=black!30]|0 & |[fill=black!30]|0 &  |[fill=black!30]|0 & |[fill=black!30]|0 & 1/16  & \\ [-1ex]
  };
  
  \draw[-latex] (m-2-7.north west) -- (m-2-13.north east);
  \draw[-latex] (m-3-7.north west) -- (m-5-13.north east);  
  \draw[-latex] (m-4-7.north west) -- (m-6-13.north east);  
  \draw[-latex] (m-5-7.north west) -- (m-7-13.north east);  
  \draw[-latex] (m-6-7.north west) -- (m-8-13.north east);  
  \draw[-latex] (m-7-7.north west) -- (m-9-13.north east);  
  \draw[-latex] (m-8-7.north west) -- (m-10-13.north east);  
  \draw[-latex] (m-9-7.north west) -- (m-10-13.north east);          

  \draw[-latex] (m-10-7.north west) -- (m-11-13.north east);        
  \draw[-latex] (m-12-7.north west) -- (m-11-13.north east);        

  \draw[-latex] (m-10-7.north west) -- (m-11-13.north east);        
  \draw[-latex] (m-12-7.north west) -- (m-11-13.north east);    

  \draw[-latex] (m-11-7.north west) -- (m-12-13.north east);        
  \draw[-latex] (m-13-7.north west) -- (m-12-13.north east);   

  \draw[-latex] (m-14-7.north west) -- (m-13-13.north east);        
  \draw[-latex] (m-15-7.north west) -- (m-13-13.north east);   
  \draw[-latex] (m-16-7.north west) -- (m-13-13.north east);        
  \draw[-latex] (m-17-7.north west) -- (m-13-13.north east);   
\end{tikzpicture}
\caption{Non-injective non-surjective mapping from the sample space (of subtheories $\mathsf{U}_i$) of the considered \PTF (on the left) to the sample space of its corresponding \PGF (on the right). }
\label{fig:samplespaces2}
\end{figure}
\hfill $\square$
\end{myExample}

Given a defeasible theory $T$, the set of \PGFs based on the subgraphs of $\gengraph{T}$ is a strict superset of the set of \PGFs that can be generated starting from the \PTFs based on $T$. Formally:
\begin{equation} 
\basedPGF{\gengraph{T}} \not\supseteq \set{\genPGF{\probFrame} \mid \probFrame \in \basedPTF{T}}.
\end{equation}
This is because, for every \PTF $\probFrame$ in $\basedPTF{T}$, $\genPGF{\probFrame}$ is such that the subgraphs of $\gengraph{T}$  which are not \subargumentcomplete and \rulecomplete have probability \RRR{zero}, while $\basedPGF{\gengraph{T}}$ includes also \PGFs not satisfying this constraint.

Hence, \PGFs are stricty more expressive than \PTFs as far uncertainty about the actual graph to be considered is concerned. When moving from a \PTF to a corresponding \PGF, some information is also lost due to non-injectivity: in general many different \PTFs may generate the same \PGF. However, this loss of information has no practical effects, since if two \PTFs generate the same \PGF, the difference between them has no influence on the subsequent argumentation steps.\smallskip

\PTFs and \PGFs capture two related forms of structural uncertainty. \PTFs refer to uncertainty about the structure of the theory to be used, while \PGFs refer to the (consequent) uncertainty about which arguments are present in an argumentation graph. \RRT{Now,} a distinct but possibly coexisting kind of uncertainty not involving structural aspects \RRT{can} be considered too, as investigated  next.

\subsubsection{\Selunc uncertainty: probabilistic labelling frames}

\Selunc uncertainty refers to the fact that, given an argumentation graph, a semantics may prescribe multiple \selunc labellings for it, each representing individually a reasonable outcome of the conflict resolution process.
The agent carrying out argument-based reasoning is then confronted with the choice of one among these outcomes.
Indeed,  choosing one of the reasonable options is mandatory to avoid standstill, at least in practical reasoning contexts, as exemplified by the famous Buridan's donkey example.
One may therefore want to express quantitative uncertainty about which option, i.e. which \selunc labelling, to choose.
Clearly this kind of \selunc uncertainty is \primaryquote{orthogonal} to structural uncertainty and is outside the scope of \PTFs and \PGFs.

To address \selunc uncertainty, we \RRT{devise} a sample space which is the set of possible labellings of an argumentation graph, as previously employed in \cite{DBLP:conf/atal/RiveretG16,DBLP:journals/argcom/RiveretKDP15,DBLP:conf/atal/RiveretPKD15}, and resulting into what we call \emph{probabilistic labelling frames} (\PAFs)  \RRT{that we} investigate in the remainder of this section.

\sloppy \PAFs are based on the assumption that, given an $X$-$\llabels$-labelling specification for an argumentation graph $G$,
an agent can assign a probability to every labelling of the set of labellings  $\crit{X}{G}{\llabels}$ which represents the sample space.
Intuitively, this can be interpreted as follows: the $X$-$\llabels$-labelling specification  gives rise to a first, uncertainty free, selection of the set of labellings $\smash{\crit{X}{G}{\llabels}}$ that are worth considering within the universe $\setofalllabels{G}{\llabels}$. Then the  labellings in $\crit{X}{G}{\llabels}$  may not be regarded as equally probable and hence associated with different probability values.
Note that, at this abstract level, we do not commit to any specific meaning of the probability of a labelling belonging to  $\crit{X}{G}{\llabels}$. For instance, in an epistemic reasoning scenario, different probability values may correspond to the fact that an agent, according to some domain-specific information, \RRT{views} the various labellings as differently credible, while in a social argumentation context, the probability values may correspond to the expectation that a specific labelling is approved by the members of a community.
Our framework, being application agnostic, provides a uniform formal model to investigate general properties valid for a large family of probabilistic argumentation scenarios, each of which can be characterised by the adopted labelling specification and, possibly, by additional specific properties.

\begin{myNotation}
Whilst our notational nomenclature refers to $X$-$\llabels$-labellings (e.g.  $\mathsf{legal}$-$\{\ON, \OFF\}$-labellings, $\mathsf{complete}$-$\{\IN, \OUT,$ $\UND\}$-labellings, $\mathsf{complete}$-$\{\IN, \OUT, \UND, \OFF\}$-labellings and so on) where $X$ and $\llabels$ refers to an $X$-$\llabels$-labelling specification, for the sake of notational conciseness we will sometimes speak of $\mathcal{S}$-labellings, using a single symbol $\mathcal{S}$ to synthesise the pair of symbols $X$-$\llabels$.
\end{myNotation}
\vspace{7mm}

\begin{myDefinition} [Probabilistic labelling frame]
\sloppy A \emph{probabilistic labelling frame} (\PAF)  based on an argumentation graph $G$ is a tuple\footnote{\PB{Later we will often omit the subscript $\mPAF$ for the sake of conciseness.}} $\tuple{G, \mathcal{S}, \tuple{\Omega_{\mPAF}, F_{\mPAF}, P_{\mPAF}}}$ where $\mathcal{S}$ \textcolor{black}{ denotes an $X$-$\llabels$-labelling specification}, and $\tuple{\Omega_{\mPAF}, F_{\mPAF}, P_{\mPAF}}$ is a probability space such that:
\begin{itemize}
\item the sample space $\Omega_{\mPAF}$ is the set of $X$-$\llabels$-labellings of $G$, i.e. $\Omega_{\mPAF} = \crit{X}{G}{\llabels}$; 
\item the $\sigma$-algebra $F_{\mPAF}$ is the power set of  $\Omega_{\mPAF}$, i.e. $F_{\mPAF}= 2^{\Omega_{\mPAF}}$;
\item the function $P_{\mPAF}$ from $F_{\mPAF}$ to $[0,1]$ is a probability distribution satisfying Kolmogorov axioms.
\end{itemize}
\label{pafone}
\end{myDefinition}

Thanks to the versatility of labellings, \PAFs feature enhanced expressiveness with respect to \PGFs.
In fact \PGFs can be regarded as a special case of \PAFs with  $\llabels=\set{\ON,\OFF}$, as illustrated by the following definition. 

\begin{myDefinition} [\PAF corresponding to a \PGF]
\label{defPAFPGF}
Given a \PGF $\tuple{G, \tuple{ \Omega_{\mPGF}, F_{\mPGF}, P_{\mPGF}}}$, 
its corresponding \PAF is a \PAF tuple $ \tuple{G, \ALL\mbox{-}\set{\ON,\OFF}, \tuple{\Omega_{\mPAF}, F_{\mPAF}, P_{\mPAF}}}$ where the probability distribution $P_{\mPAF}$  is such that $\forall S \in  F_{\mPAF}$: 
$$P_{\mPAF}(S) = \sum_{\{H \in \Sub(G)\, \mid\, \labelling_{G,H} \in S  \}} P_{\mPGF}( \set{H} ).$$
\end{myDefinition}

\begin{myNotation}
The \PAF corresponding to a \PGF $\probFrame$ is denoted as $\genPAF{\probFrame}$.
\end{myNotation}

Given a \PGF based on an argumentation graph $G$, one can draw different mappings from the elements of the sample space $\Omega_{\mPGF}=\Sub(G)$ to the elements of the sample space of a \PAF based on $G$, depending on the adopted labelling specification. As mentioned above, if one adopts the $\ALL\mbox{-}\set{\ON,\OFF}$ specification every element of $\Omega_{\mPGF}$, i.e. every subgraph of $G$, corresponds exactly to an element of  $\Omega_{\mPAF}$, i.e. to a $\set{\ON,\OFF}$-labelling and vice versa.
If instead one adopts an $X$-$\set{\IN,\OUT,\UND,\OFF}$ specification where $X$ is a multiple status argumentation semantics, in general for each subgraph $\RRR{H}$ of $G$ there are many corresponding elements in $\Omega_{\mPAF}$, namely the set of $X$-$\set{\IN,\OUT,\UND,\OFF}$\RRT{-}labellings of $G$ corresponding to the set of $X$-$\set{\IN,\OUT,\UND}$\RRT{-}labellings of $\RRR{H}$. However, a reverse correspondence can be drawn as follows.

\begin{myDefinition} [\PGF corresponding to a \PAF]
\label{defPAFPGF}
\item 
Given a \PAF $ \tuple{G, \mathcal{S}, \tuple{\Omega_{\mPAF}, F_{\mPAF}, P_{\mPAF}}}$, 
its corresponding \PGF is a \PGF tuple $\tuple{G, \tuple{ \Omega_{\mPGF}, F_{\mPGF}, P_{\mPGF}}}$
 where the probability distribution $P_{\mPGF}$  is such that $\forall S \in  F_{\mPGF}$: 
$$P_{\mPGF}(S) = \sum_{\PBB{\{\labelling \in \Omega_{\mPAF}\, \mid\, G(\labelling) \in S  \}}} P_{\mPAF}( \set{\labelling} )$$
\PBB{where $G(\labelling)$ is the subgraph of $G$ induced by the set of arguments which are not labelled $\OFF$ in $\labelling$, namely, letting $\mathcal{X}= \mathcal{A} \setminus \OFF(\labelling)$,  $G(\labelling)=(\mathcal{X}, \defeat_G \cap (\mathcal{X} \times \mathcal{X}), \support_G  \cap (\mathcal{X} \times \mathcal{X}))$.}
 \end{myDefinition}
\noindent \RRR{The correspondence is illustrated in Example \ref{ex:neither2}.
If $X$ is single status the sets mentioned above are singletons and the mapping between $\Omega_{\mPGF}$ and $\Omega_{\mPAF}$ is bijective.}

\begin{myExample}[Abstract example II]
\label{ex:neither2}
Let us consider a \PGF  $\tuple{\mathsf{G}, \tuple{\Omega_{\mPGF}, F_{\mPGF}, P_{\mPGF}}}$ and a \PAF  $\tuple{\mathsf{G}, \mathsf{preferred}\mbox{-}\{ \IN, \OUT, \UND, \OFF\},\tuple{\Omega_{\mPAF}, F_{\mPAF}, P_{\mPAF}}}$ with  the argumentation graph $\mathsf{G} = \tuple{ \set{\mathsf{B},\mathsf{C}}, \set{ (\mathsf{B}, \mathsf{C}), (\mathsf{C}, \mathsf{B}) }  }$ as pictured in Figure \ref{fig:examplesupportcons} .
\begin{figure}[ht!]
\centering
\begin{tikzpicture}[node distance=1.4cm,main node/.style={circle,fill=white!20,draw,font=\sffamily\scriptsize}]

 \node[main node] (5) [right of=4] {B};
 \node[main node] (6) [right of=5] {C};

 \path[->,shorten >=1pt,auto, thick,every node/.style={font=\sffamily\small}]
 (5) edge node {} (6)
 (6) edge node [right] {} (5);
 
\end{tikzpicture}
\caption{A \PBB{simple} argumentation graph.}
\label{fig:examplesupportcons}
\end{figure}
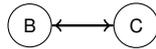

\noindent Figure \ref{figure:mappingx} shows the mapping from the sample space $\Omega_{\mPAF}$ to $\Omega_{\mPGF}$, \PBB{a given probability distribution on $\Omega_{\mPAF}$\PBX{,} and the corresponding distribution on $\Omega_{\mPGF}$}.

\begin{figure}[ht!]
\centering
  \begin{tikzpicture}[cell/.style={rectangle,draw=black}, nodes in empty cells]
  \matrix[
  matrix of math nodes,
  ](m)
  {  
  &  \mathsf{B} & \mathsf{C} & P_{\mPGF}  &  & & & & & & &  & & & \mathsf{B} & \mathsf{C} & P_{\mPAF} & & \\ 
 \mathsf{H}_{1}&  1 & 1 & 0.8 &  &  & & & & & &  & & \labelling_1 & |[fill=green!30]|\IN & |[fill=red!30]|\OUT 	& 0.4	\\ [-0.5ex]
 \mathsf{H}_{2}&   1 & |[fill=black!30]|0 & 0.2 &  &   & & & & & &  & & \labelling_2 & |[fill=red!30]|\OUT & |[fill=green!30]|\IN & 0.4	\\ [-0.5ex]
 \mathsf{H}_{3}&   |[fill=black!30]|0 & 1 &  0 &  &  & & & & & &  & & \labelling_3 & |[fill=green!30]|\IN & |[fill=black!30]|\OFF  & 0.2    \\[-0.5ex]
 \mathsf{H}_4&     |[fill=black!30]|0 & |[fill=black!30]|0 & 0  &  &  & & & & & &  & & \labelling_4 &  |[fill=black!30]|\OFF & |[fill=green!30]|\IN  & 0  \\[-0.5ex]
  &       &  &   &  &  & & & & & &  & & \labelling_5 &  |[fill=black!30]|\PBX{\OFF} & |[fill=black!30]|\PBX{\OFF} & 0  \\[-0.5ex]
  };
  
  \draw[-latex] (m-2-13.north east) -- (m-2-6.north west);
  \draw[-latex] (m-3-13.north east) -- (m-2-6.north west);
  \draw[-latex] (m-4-13.north east) -- (m-3-6.north west);  
  \draw[-latex] (m-5-13.north east) -- (m-4-6.north west);  
  \draw[-latex] (m-6-13.north east) -- (m-5-6.north west);  
\end{tikzpicture}
\caption{\RRR{Mapping from the sample space of the considered \PAF (on the right) to the sample space of \PBB{the relevant \PGF (on the left) for the argumentation graph of Figure \ref{fig:examplesupportcons}.}} }
\label{figure:mappingx}
\end{figure}
\hfill $\square$
\end{myExample}

Since \PGFs are more expressive than \PTFs, as far as uncertainty about the actual argumentation graph is concerned,  this correspondence shows that also \PAFs are strictly more expressive than \PTFs in this respect.
Further, one can also extend the non-injective non-surjective relationship from \PTFs to \PGFs  (see Definition \ref{defPGFfromPTF}) to the corresponding \PAFs.

While \PAFs using $\set{\ON,\OFF}$-labellings are able to encompass the uncertainty about argumentation graphs expressed by \PTFs and \PGFs, \PAFs using \set{\IN, \OUT, \UND}-labellings can  express uncertainty about the actual labelling to choose in a multiple status semantics with a fixed argumentation graph, thus covering the \selunc uncertainty dimension alone.
More interestingly,  \PAFs using $\set{\IN, \OUT, \UND,\OFF}$-labellings are able to combine both kinds of uncertainty in a single representation. As discussed later, this allows  to capture some existing literature approaches as special cases of \PAFs, providing at the same time a comprehensive formalism to combine them.\medskip

To recap, we have  \RRT{laid down} different kinds of probabilistic frames, namely \PTFs, \PGFs and (the novel concept of) \PAFs, each kind featuring a different probability space. We showed that \PAFs subsume \PTFs and \PGFs, and for this reason, we will  focus   on \PAFs in the remainder. Before studying the relationships between \PAFs with other works, we will see next that such \PAFs are  convenient to attach probabilistic measures to different status\PB{es} of arguments and to investigate their relationships.

\subsection{Probabilistic labelling of arguments}
\label{suibsection:arguments1}

The generalised framework introduced in the previous sections provides the basis for introducing further notions related to the reasoning tasks an agent can be interested in.
For instance\RRT{,} an agent may focus attention on a few arguments which are most significant for his/her purposes, and be interested in the labels that can be assigned to these arguments and in the relevant probabilities both at the level of argument acceptance and of argument justification.

\subsubsection{Probabilistic \PB{argument} acceptance labellings}

\PB{In order to define probabilistic argument acceptance labellings,} we introduce suitable random variables on the basis of the generic probabilistic space introduced for \PAFs.

Recall that a random variable is a function (traditionally \PBX{denoted} by an upper case letter as $X$, $Y$ or $Z$ for example) from $\Omega$ into another set $\domain$ of elements. An assignment of a random variable $X$ is denoted as $X=\domainelem$ where $\domainelem \in \domain$. An assignment for a set of random variables $\set{X_1, \ldots, X_n}$ is simply a set $\set{X_1=\domainelem_1, \ldots, X_n=\domainelem_n}$ of assignments to all variables in the set.

For every argument $A$, we use a categorical random variable called a \emph{random labelling} denoted $L_A$ from $\Omega$ into a set $\llabels$ of labels such as $\{\ON, \OFF \}$, or  $\{\IN, \OUT, \UND \}$ or $\{\IN, \OUT, \UND, \OFF \}$. 
So, for instance, if the sample space $\Omega$ is a set of $\{\ON, \OFF \}$-labellings, the event $L_A = \ON$ is a shorthand for the outcomes $\{ \labelling  | \labelling  \in \Omega,  \labelling(A) = \ON \}$. If the sample space $\Omega$ is a set of $\{\IN, \OUT, \UND \}$-labellings or $\{\IN, \OUT, \UND, \OFF \}$-labellings, then
$L_A = \ON$ is a shorthand for the outcomes $\{ \labelling  | \labelling  \in \Omega,  \labelling(A) = \IN \mbox{ or }  \labelling(A) = \OUT \mbox{ or }  \labelling(A) = \UND \}$.

\begin{myNotation}
\item
\begin{itemize}
\item Sets of random labellings are denoted using upper boldface type. So $\textbf{L}$ denotes a set of random labellings $\{L_{A1}, \ldots, L_{An}\}$. 
\item   Assignments to (sets of) random labellings are denoted using boldface type. 
\RRT{For instance,} given a set of random labellings $\textbf{L} = \{L_{A1}, \ldots, L_{An}\}$, a possible assignment is $\mathbf{l} = \{L_{A1} = \ON, \ldots, L_{An}= \OFF \}.$
\item  An assignment to a set of random labellings can be straightforwardly mapped to a labelling, and for this reason,
an \primaryquote{assignment} may be also called a \primaryquote{labelling} in the remainder. When the distinction is made, and for the sake of compactness, the labelling $\labarg$ corresponding to an assignment $\alas$ is denoted $\labarg_{\alas}$. To make a bridge with our notation for labellings and to avoid any ambiguity, we  write $\alas_A= l$ to say that the random labelling of argument $A$ is assigned the value $l$ according to the assignment  $\alas$.
\item The joint distribution over a set 
$\textbf{L} = \{L_{A1}, L_{A2}, \ldots, L_{An} \}$ of random labellings is formally denoted  $P(\{L_{A1}, L_{A2}, \ldots, L_{An} \})$, but following standard notation, we will write it $P(L_{A1}, L_{A2}, \ldots, L_{An})$.
\end{itemize} 
\end{myNotation}

From the definition of random variables, the probability of an assignment  $\alas$ for a set of random labellings is the sum of the probabilities of the labellings $\labelling_{\alas}$ in the sample space where this assignment occurs:

\begin{equation}
P(\alas) = \sum_{\labelling_{\alas} \in \Omega}  P(\set{\labelling_{\alas}}).
\end{equation}

Consequently, if the assignment concerns only one argument $A$, then  the marginal probability of the status of this argument $A$ is the sum of the probabilities of the labellings in the sample space where this assignment occurs: 

\begin{equation}
P(L_A = l) = \sum_{\labelling \in \Omega: \labelling(A)=l}  P(\set{\labelling}).
\label{eqlabsingarg}
\end{equation}

\begin{myExample}[continues=ex:neither2]
\item
\RRR{Suppose the \PAF  $\tuple{\mathsf{G}, \mathsf{preferred}\mbox{-}\{ \IN, \OUT, \UND, \OFF\},\tuple{\Omega_{\mPAF}, F_{\mPAF}, P_{\mPAF}}}$ has a distribution as indicated in Figure \ref{figure:mappingx}.} 
\RRR{By Equation \ref{eqlabsingarg},  we can compute the marginal probability that arguments $\mathsf{B}$ or $\mathsf{C}$ obtain a certain acceptance status, as follows.\smallskip
\begin{center}
\begin{tabular}{lll}
$P_{\mPAF}(L_{\mathsf{B}}= \IN ) =  P_{\mPAF}(\set{\labelling_1} ) +  P_{\mPAF}(\set{\labelling_3})$ &  $(= 0.6)$  \\
$P_{\mPAF}(L_{\mathsf{B}}= \OUT ) = P_{\mPAF}(\set{\labelling_2} )$  & $(= 0.4)$  \\
$P_{\mPAF}(L_{\mathsf{B}}= \UND ) = 0$ & & \\
$P_{\mPAF}(L_{\mathsf{B}}= \OFF ) = 0$ & & \\
$P_{\mPAF}(L_{\mathsf{C}}= \IN ) = P_{\mPAF}( \set{\labelling_2} )$ & $(= 0.4)$ \\
$P_{\mPAF}(L_{\mathsf{C}}= \OUT ) = P_{\mPAF}( \set{\labelling_1} )$ & $( = 0.4)$ \\
$P_{\mPAF}(L_{\mathsf{C}}= \UND ) = 0$ \\
$P_{\mPAF}(L_{\mathsf{C}}= \OFF ) = P_{\mPAF}( \set{\labelling_3})$ & $( = 0.2)$ 
\end{tabular}
\end{center}
\hfill$\square$}
\end{myExample}

Since a random labelling is a random variable, we get directly the desirable property that the probabilities of the different labels for an argument sum up to $1$.

\begin{myProposition}
\label{consistency0}
Let $\tuple{G, X\mbox{-}\llabels , \tuple{\Omega, F, P}}$ be a \PAF.

$$\sum_{l \in \llabels} P(L_A = l) = 1.$$

\end{myProposition}
From Proposition \ref{consistency0}, for any  $X$-$\{\ON, \OFF\}$-labelling specification, we obtain:  
\begin{equation} 
P(L_A = \ON) +   P(L_A = \OFF)  = 1.
\end{equation}
For any  $X$-$\{\IN, \OUT, \UND \}$-labelling specification, we have:
\begin{equation} 
 P(L_A = \IN) +   P(L_A = \OUT) +   P(L_A = \UND)  = 1.
 \end{equation}
For any  $X$-$\{\IN, \OUT, \UND, \OFF \}$-labelling specification: 
\begin{equation}
 P(L_A = \IN) +   P(L_A = \OUT) +   P(L_A = \UND) +  P(L_A = \OFF)   = 1.
 \end{equation}

%
%

%
%

%
%

Stable labellings enjoy more specific properties given that no argument can be labelled $\UND$,  and therefore, for the $\mathsf{stable}$-$\{\IN, \OUT, \UND \}$-labelling specification, we have:
\begin{equation}
 P(L_A = \IN) +   P(L_A = \OUT)  = 1.
\end{equation}
For the $\mathsf{stable}$-$\{\IN, \OUT, \UND, \OFF \}$-labelling specification:
\begin{equation}
P(L_A = \IN) +   P(L_A = \OUT) +  P(L_A = \OFF)   = 1.
\end{equation}
\RRR{Other results on the probability of argument acceptance statuses can be derived\PBX{. W}e will review some of them in the remainder. }

\subsubsection{\PB{Probabilistic argument justification labellings}}
\label{subsec:pajl}
We now move to \RRT{explore} how uncertainty on argument acceptance may in turn affect argument justification: in this respect, different considerations can be drawn depending on the set of labels adopted in the underlying \PAF. 
\PAFs based on $\{\ON, \OFF\}$\RRT{-}labellings express uncertainty about the structure of the framework and this simple set of labels does not carry significant information about argument justification.
%
%
The situation is different for \PAFs based on $\{\IN, \OUT, \UND\}$ \RRT{or}  $\{\IN, \OUT, \UND, \OFF\}$-labellings. 

\RR{Assuming a (possibly random) selection of a \RR{unique} acceptance labelling amongst a set of $\{\IN, \OUT, \UND\}$ or  $\{\IN, \OUT, \UND, \OFF\}$-labellings,} we can then distinguish argument justification status \emph{before} \RR{this} selection (which in the present proposal is specified by Definition \ref{def:argumentJustification}) and argument justification status \emph{after} the selection.
The latter is directly determined by the unique selected labelling, and, \RRT{by} Definition \ref{def:argumentJustification} in this restricted case, only three justification labels are possible, namely $\OFFJ$, $\SKJ$, and $\NOJ$: after selection \RR{of a unique acceptance labelling}, an argument $A$ gets the justification label $\OFFJ$ if and only if $A$ is labelled $\OFF$ in the selected labelling,  $\SKJ$ if and only if it is labelled $\IN$ in the selected labelling, and $\NOJ$ if and only if it is labelled $\OUT$ or $\UND$ in the selected labelling.
\RR{So, }\PB{if one assumes, as we do, that an agent at some point makes a definite choice \RR{of a unique acceptance labelling among} the alternative acceptance labellings produced by a semantics, the traditional notion of argument justification is somehow redundant, since it is related to argument acceptance by direct and simple relationships and some justification notions, like credulous acceptance, actually are not relevant.}
\PB{Argument justification keeps its role if one assumes instead that an agent does not have to choose among the alternative labellings but rather needs to draw a sort of synthetic view about them. Both scenarios make sense in different contexts, the former appearing more suitable for practical reasoning, i.e. reasoning about what to do, the latter for epistemic reasoning, i.e. reasoning about what to believe. Further discussions of these aspects are beyond the scope of the present paper and are left to future work.
We just observe that, in the latter scenario, it is} possible to identify a relationship between the justification status of an argument and \PAFs based on $\{\IN, \OUT, \UND, \OFF\}$-labellings, provided that one assumes that each labelling has non zero probability.

\begin{myProposition}
Let $\tuple{G, X\mbox{-}\llabels , \tuple{\Omega, F, P}}$ be a \PAF,   $\mathcal{L} \subseteq \smash{\crit{X}{G}{\llabels}}$  a non-empty set of argument labellings such that for every labelling $\labelling \in \mathcal{L}$ $P(\set{\labelling}) > 0$, and $\labelling_{\mathsf{J}}$ the  \gullible $\{\OFFJ, \SKJ, \CRJ, \NOJ \}$-labelling  of $G$. For every argument $A$ in $\mathcal{A}_G$, it holds that: 
\begin{itemize}
\item  $\labelling_\mathsf{J}(A) = \OFFJ$  if, and only if, $P(L_A = \OFF) = 1$;
\item  $\labelling_\mathsf{J}(A) = \SKJ$  if, and only if, $P(L_A = \IN) = 1$;
\item $\labelling_\mathsf{J}(A) = \CRJ$ if, and only if, \RR{$0 < P(L_A = \IN) < 1$};
\item  $\labelling_\mathsf{J}(A) = \NOJ$  if, and only if, $P(L_A = \OUT) > 0$ or $P(L_A = \UND) > 0$, and $P(L_A = \IN) = 0$.
\end{itemize}
\label{theorem:argjust}
\end{myProposition}

\begin{proof} Let us make a proof for each justification label.
\begin{description}
\item \textbf{$\OFFJ$.} From \Definitionref \ref{def:argumentJustification}, $\labelling_\mathsf{J}(A) = \OFFJ$  if, and only if,  $\forall \labelling \in \mathcal{L}$, $\labelling(A) = \OFF$.

We have \primaryquote{$\forall \labelling \in \mathcal{L}$, $\labelling(A) = \OFF$} if, and only if, $P(L_A = \OFF) = 1$.

Therefore,  $\labelling_\mathsf{J}(A) = \OFFJ$ if, and only if, $P(L_A = \OFF) = 1$.

\item \textbf{$\SKJ$.} From \Definitionref \ref{def:argumentJustification}, $\labelling_\mathsf{J}(A) = \SKJ$  if, and only if,  $\forall \labelling \in \mathcal{L}$, $\labelling(A) = \IN$.
We have \primaryquote{$\forall \labelling \in \mathcal{L}$, $\labelling(A) = \IN$} if, and only if, $P(L_A = \IN) = 1$.
Therefore,  $\labelling_\mathsf{J}(A) = \SKJ$ if, and only if, $P(L_A = \IN) = 1$.

\item \textbf{$\CRJ$.} From \Definitionref \ref{def:argumentJustification}, $\labelling_\mathsf{J}(A) = \CRJ$  if, and only if, $\exists\labelling \in \mathcal{L}$: $\labelling(A)= \IN$ and $\labelling_\mathsf{J}(A) \neq \SKJ$; and 
\begin{enumerate}
\item  $\exists L \in \mathcal{L}, \labelling(A)= \IN$, if, and only if, $P(L_A = \IN) > 0$.
\item \RR{$\labelling_\mathsf{J}(A) \neq \SKJ$  if, and only if, $\exists\labelling \in \mathcal{L}$: $\labelling(A) \neq \IN$.
That is, $\labelling_\mathsf{J}(A) \neq \SKJ$ if, and only if, $P(L_A = \IN) < 1$.}
\end{enumerate}
\RR{Therefore, $\labelling_\mathsf{J}(A) = \CRJ$  if, and only if,
 $0 < P(L_A = \IN) < 1$.}  

\item \textbf{$\NOJ$.} From \Propositionref \ref{prop:NOJ}, $\labelling_{\mathsf{J}}(A) = \NOJ$ if, and only if,  
$\exists \labelling \in \mathcal{L}$, $\labelling(A)=\OUT \mbox{ or } \labelling(A)=\UND$, and $\forall \labelling \in \mathcal{L}, \labelling(A)\neq\IN$.

Therefore, $\labelling_\mathsf{J}(A) = \NOJ$ with respect to $\mathcal{L}$ if, and only if,  $P(L_A = \OUT) > 0$ or $P(L_A = \UND) > 0$, and $P(L_A = \IN) = 0$.
\end{description}
\end{proof}
%

\begin{myExample}[continues=ex:neither2]
\RRR{Arguments $\mathsf{B}$ and $\mathsf{C}$ are both credulously justified, i.e., $\labelling_\mathsf{J}(\mathsf{B}) = \CRJ$ and $\labelling_\mathsf{J}(\mathsf{C}) = \CRJ$.}
\hfill$\square$
\end{myExample}

\PB{To recap,  \PAFs are a convenient basis to define probabilistic measures about the acceptance status of arguments and to derive some essential properties.} Turning to argument justification labelling, we showed that, \PB{under suitable hypotheses}, it is possible to identify a relationship between the justification status of an argument and \PAFs based on $\{\IN, \OUT, \UND, \OFF\}$-labellings. Probabilistic measures  concerning statement labellings and their relationships can be also derived on the basis of \PAFs, as we will see in next section.

\subsection{Probabilistic labelling of statements}
\label{subsection:probstatements}

The probabilistic labelling of statements directly follows from the probabilistic labelling of arguments.
We introduce, as we did for random labellings of arguments,  random variables concerning the labelling of statements. 
For any literal $\lit$, we introduce a categorical random variable which is denoted $K_\lit$ and which can take value in the set  of labels \RRT{$\klabels$} of a considered $\klabels$-labelling of literals. These random variables are also called random labellings.

The marginal probability of a literal labelled $k$ is the sum of labellings in the sample space where the literal is labelled as such:
\begin{equation}
P(K_\varphi = k) = \sum_{\set{\labarg \in \Omega\, \mid\, \lablit(\labarg, \lit)= k}} P( \set{\labarg}).
\label{e1}
\end{equation}

\noindent Since $K_\varphi$ is a random variable, the sum of marginal probabilities over its possible assignments equals 1.
\begin{myProposition}
\label{consistency}
Let $\tuple{G, \mathcal{S}, \tuple{\Omega, F, P}}$ be a \PAF  and $\lit$ a literal.
For any $\klabels$-labelling of literals, 
$$\sum_{k \in \klabels} P(K_\varphi = k) = 1.$$
\end{myProposition}

\noindent Proposition \ref{consistency} can be instantiated with different statement labellings. For example, given a \PAF $\tuple{G,  \mathcal{S} , \tuple{\Omega, F, P}}$
and the bivalent $\{\inn, \no\}$-labellings, we have:
\begin{equation}\label{eqforTh0}
 P(K_\varphi = \inn) +  P(K_{\varphi} = \no) = 1
\end{equation} 
while in case of \rough{} $\{\inn, \out, \und, \off, \unp\}$-labellings we obtain:
\begin{equation}\label{eqforTh0rough1}
 P(K_\varphi = \inn) + P(K_\varphi = \out) + P(K_\varphi = \und) + P(K_{\varphi} = \off) + P(K_{\varphi} = \unp)  = 1.
\end{equation} 

\noindent Note that either $P(K_{\varphi} = \unp) =0$ or  $P(K_{\varphi} = \unp) =1$, since, given a set of arguments, there is no uncertainty on  the fact that a literal $\varphi$ is the conclusion of at least one argument or none in the set.

\begin{myProposition}
\label{eqforTh1}
Let  $\tuple{G, \mathcal{S},\tuple{\Omega,F,P}}$ be a \PAF, 
and $\lit_1$ and $\lit_2$  literals in conflict, i.e. $\conf(\lit_1, \lit_2)$ and let us consider the  bivalent $\{\inn, \no\}$-labellings \textcolor{black}{and the \rough{} $\{\inn, \out, \und, \off, \unp\}$-labellings}, 
\begin{equation*}
P(K_{\lit_1} = \inn) +  P(K_{\lit_2}  = \inn) \leq 1. 
\end{equation*}
\end{myProposition}

\begin{proof}
In a given labelling a literal $\lit_1$ is labelled $\inn$ or not. \textcolor{black}{If the literal  $\lit_1$ is labelled $\inn$ then there is an argument $A$ labelled $\IN$ such that $\Conc(A)=\lit_1$. Then it must be the case that this argument attacks or is attacked by every other argument $A'$ such that $\Conc(A')=\lit_2$ and given the basic conflict-freeness property satisfied by any argumentation semantics it can not be the case that any such argument $A'$ is labelled $\IN$ in the same labelling. Hence also $\lit_2$ cannot be labelled $\inn$ in the same labelling.
If instead} the literal  $\lit_1$  is not labelled $\inn$, then any conflicting literal  $\lit_2$ can labelled $\inn$ or not. Let  \RRR{$\mathcal{K}_{\lit_1, \inn}$} denote the set of labellings of arguments such that $\lit_1$ is labelled $\inn$,  $\mathcal{K}_{\lit_2, \inn}$ the set of labellings of arguments such that $\lit_2$ is labelled $\inn$, and  $\mathcal{K}$  \textcolor{black}{the complement set of labellings $\Omega \backslash \mathcal{K}_{\lit_1, \inn} \cup \mathcal{K}_{\lit_2, \inn}$}. These three sets form a partition of $\Omega$, thus:
$
P(\mathcal{K}_{\lit_1, \inn}) +  P(\mathcal{K}_{\lit_2, \inn}) +  P(\mathcal{K}) = 1.
$ 
Since $ P(\mathcal{K}) \geq 0$, we have
$
 P(\mathcal{K}_{\lit_1, \inn}) +  P(\mathcal{K}_{\lit_2, \inn}) \leq 1,
$
therefore
$
 P(K_{\lit_1} = \inn) +  P(K_{\lit_2} = \inn) \leq 1.
$
\hfill $\square$
\end{proof}

The next proposition \PB{is a corollary} which follows   from propositions \ref{consistency} and \ref{eqforTh1}.

\begin{myProposition}
\label{eqth35}
Let  $\tuple{G, \mathcal{S}, \tuple{\Omega, F, P}}$ be a \PAF, 
and $\lit_1$ and $\lit_2$  literals in conflict, i.e. $\conf(\lit_1, \lit_2)$. Let us consider the  bivalent $\{\inn, \no\}$-labellings and the \rough{} $\{\inn, \out, \und, \off, \unp\}$-labellings,
\begin{equation*}
P(K_{\lit_1} = \inn) \leq  P(K_{\lit_2}  \neq \inn)
\end{equation*}
where
\begin{itemize}
\item $P(K_{\lit_2}  \neq \inn) = P(K_{\lit_2}  = \no)$ in the case of  bivalent $\{\inn, \no\}$-labellings;
\item $P(K_{\lit_2}  \neq \inn) =  P(K_{\lit_2}  = \out) + P(K_{\lit_2}  = \und) + P(K_{\lit_2}  = \off) + P(K_{\lit_2}  = \unp) $ in the case of  \rough{} $\{\inn, \out, \und, \off, \unp\}$-labellings.
\end{itemize}

\end{myProposition}

Both propositions (\ref{eqforTh1}) and (\ref{eqth35}) feature inequalities which \RRT{are} consequences of the fact that argumentation does not necessarily fulfil (an argumentation counterpart of) the  principle of excluded middle, since it is possible, for example, to have outcomes where, for any statement, neither a statement nor its complement  are labelled $\inn$  -- they may be both labelled $\no$ for instance. These propositions can also be viewed as an expression of a probabilistic notion of consistency: it is not possible to regard two conflicting statements as highly probable at the same time, because a statement can be believed at most as much as any conflicting statement is disbelieved.\\

In the following\RRT{,} we provide a complete example of use of the various probabilistic notions we have introduced in this section.

\begin{myExample}[continues=example:basis]
\label{example:runninggg}
\RRR{The research scientist is asked to report on the degree of uncertainty concerning whether the critical program will be running or not at any point of time of the project. To help her, we decide to use the approach of  probabilistic labellings. }

\PBB{Firstly, we \PBX{deal with} uncertainty at the level of the rules.  The solar panels work well but they only provide power during daytime, thus, assuming that daytime and nighttime are equal, at any \RRT{point of time}  there is 50\% probability that there is enough solar power, i.e. that the rule $\mathsf{r}_{\litbone}$ \RRT{applies}. The battery is partially damaged, and only one expert over five advances the argument concluding that the battery cannot provide power, giving rise to a probability of \RRR{$0.2$} that the rule $\mathsf{r}_{\litbtwo}$ is considered. All other rules are not affected by uncertainty.
Under a reasonable assumption of independence, this  gives rise to a \PTF \RRR{as drawn in Figure $\ref{figcombinedPTF}$
where $P_{\mPTF}(\set{\mathsf{U}_1}) = 0.4$, $P_{\mPTF}(\set{\mathsf{U}_2}) = 0.4$, $P_{\mPTF}(\set{\mathsf{U}_3}) = 0.1$, $P_{\mPTF}(\set{\mathsf{U}_4}) = 0.1$ }   and the probability of all other \RRR{subtheories} is zero.}

\begin{figure}[ht!]
\centering
\begin{tabular}{ c c c c c c c c c c c c c c }
 & \small{$\mathsf{r}_{\litbone}$} & \small{$\mathsf{r}_{\litbtwo}$} & \small{$\mathsf{r}_{\litb}$} & \small{$\mathsf{r}_{\litc}$} & \small{$\mathsf{r}_{\litd}$} & $P_{\mPTF}$  \\
$\mathsf{U}_1$ & \cellcolor[HTML]{C0C0C0}$0$      & \cellcolor[HTML]{C0C0C0}$0$       & $1$  & $1$ & $1$ & $0.4$   \\
$\mathsf{U}_2$ & $1$      &\cellcolor[HTML]{C0C0C0}$0$      & $1$ & $1$ & $1$   & $0.4$  \\
$\mathsf{U}_3$ & \cellcolor[HTML]{C0C0C0}$0$  &  $1$    & $1$ & $1$ &  $1$   & $0.1$  \\
$\mathsf{U}_4$ &  $1$  &  $1$    & $1$ & $1$ & $1$   & $0.1$  \\
\end{tabular}
\caption{\PBB{Subtheories with non-zero probability.}}
\label{figcombinedPTF}
\end{figure}

\sloppy \PBB{The \PGF corresponding to this \PTF, \RRR{as illustrated in Figure \ref{figcombinedPGF}}, is such that  the full argumentation graph $\tuple{ \{\mathsf{B1}, \mathsf{B2}, \mathsf{B}, \mathsf{C}, \mathsf{D} \}, \set{(\mathsf{B}, \mathsf{C}), (\mathsf{C}, \mathsf{D}),(\mathsf{D}, \mathsf{C})} , \set{(\mathsf{B1}, \mathsf{B}), (\mathsf{B2}, \mathsf{B})} }$ has probability \RRR{$0.1$}, while  the subgraphs induced by the sets of arguments $\{\mathsf{B2}, \mathsf{C}, \mathsf{D} \}$, $\{\mathsf{B1}, \mathsf{C}, \mathsf{D} \}$ and $\{\mathsf{C}, \mathsf{D} \}$ have  probabilities  \RRR{$0.1$, $0.4$, and $0.4$} respectively.
}

\begin{figure}[ht!]
\centering
\begin{tabular}{ c c c c c c c c c c c c c c }
 & \small{$\mathsf{B1}$} & \small{$\mathsf{B2}$} & \small{$\mathsf{B}$} & \small{$\mathsf{C}$} & \small{$\mathsf{D}$} & $P_{\mPGF}$  \\
$\mathsf{G}_1$ & \cellcolor[HTML]{C0C0C0}$0$      & \cellcolor[HTML]{C0C0C0}$0$       & \cellcolor[HTML]{C0C0C0}$0$  & $1$ & $1$ & $0.4$   \\
$\mathsf{G}_2$ & $1$      &\cellcolor[HTML]{C0C0C0}$0$      & \cellcolor[HTML]{C0C0C0}$0$ & $1$ & $1$   & $0.4$  \\
$\mathsf{G}_3$ & \cellcolor[HTML]{C0C0C0}$0$  &  $1$    & \cellcolor[HTML]{C0C0C0}$0$ & $1$ &  $1$   & $0.1$  \\
$\mathsf{G}_4$ &  $1$  &  $1$    & $1$ & $1$ & $1$   & $0.1$  \\
\end{tabular}
\caption{\PBB{Argumentation subgraphs with non-zero probability.}}
\label{figcombinedPGF}
\end{figure}

\sloppy  \RRR{The expert has a skeptical stance, and thus she can decide to adopt  grounded $\{\IN, \OUT, \UND, \OFF \}$-labellings.} \PBB{Accordingly, the argument labellings illustrated in the left part of \Figureref \ref{figcombined} are generated and we get a \PAF where
 $P(\set{\labelling_1})=0.4$, $P(\set{\labelling_2})=0.4$, $P(\set{\labelling_3})=0.1$, $P(\set{\labelling_4})=0.1$.}

\RRR{
\begin{figure}[ht!]
\centering
\begin{tabular}{ c c c c c c c c c c c c c c c c }
 & \small{$\mathsf{B1}$} & \small{$\mathsf{B2}$} & \small{$\mathsf{B}$} & \small{$\mathsf{C}$} & \small{$\mathsf{D}$} & $P$ & & \small{$\neg \mathsf{b1}$} & \small{$\neg \mathsf{b2}$} & \small{$\neg \mathsf{b}$} & \small{$\litc$} & \small{$\litd$} \\
$\labelling_1$ & \cellcolor[HTML]{C0C0C0}$\OFF$      & \cellcolor[HTML]{C0C0C0}$\OFF$       & \cellcolor[HTML]{C0C0C0}$\OFF$  & \cellcolor{blue!40}$\UND$ & \cellcolor{blue!40}$\UND$    & 0.4 &  &  \cellcolor[HTML]{C0C0C0}$\off$      & \cellcolor[HTML]{C0C0C0}$\off$       & \cellcolor[HTML]{C0C0C0}$\off$  & \cellcolor{blue!40}$\und$  & \cellcolor{blue!40}$\und$   \\
$\labelling_2$ & \cellcolor[HTML]{9AFF99}$\IN$      &\cellcolor[HTML]{C0C0C0}$\OFF$      & \cellcolor[HTML]{C0C0C0}$\OFF$ & \cellcolor{blue!40}$\UND$ & \cellcolor{blue!40}$\UND$   & 0.4 & &  \cellcolor[HTML]{9AFF99}$\inn$      &\cellcolor[HTML]{C0C0C0}$\off$      & \cellcolor[HTML]{C0C0C0}$\off$ & \cellcolor{blue!40}$\und$  & \cellcolor{blue!40}$\und$   \\
$\labelling_3$ & \cellcolor[HTML]{C0C0C0}$\OFF$  &  \cellcolor[HTML]{9AFF99}$\IN$    & \cellcolor[HTML]{C0C0C0}$\OFF$ & \cellcolor{blue!40}$\UND$ & \cellcolor{blue!40}$\UND$    & 0.1 &  &  \cellcolor[HTML]{C0C0C0}$\off$  &  \cellcolor[HTML]{9AFF99}$\inn$    & \cellcolor[HTML]{C0C0C0}$\off$ & \cellcolor{blue!40}$\und$  & \cellcolor{blue!40}$\und$  \\
$\labelling_4$ &  \cellcolor[HTML]{9AFF99}$\IN$  &  \cellcolor[HTML]{9AFF99}$\IN$    & \cellcolor[HTML]{9AFF99}$\IN$ & \cellcolor[HTML]{FFCCC9}$\OUT$ & \cellcolor[HTML]{9AFF99}$\IN$   & 0.1 &  & \cellcolor[HTML]{9AFF99}$\inn$  &  \cellcolor[HTML]{9AFF99}$\inn$    & \cellcolor[HTML]{9AFF99}$\inn$ & \cellcolor[HTML]{FFCCC9}$\out$ & \cellcolor[HTML]{9AFF99}$\inn$ \\
\end{tabular}
\caption{\PBB{Argument and statement labellings with grounded semantics, with non-zero probability.}}
\label{figcombined}
\end{figure}
}

\PBB{
The corresponding statement labellings according to Definition \ref{def:roughlab} are illustrated in the right part of \Figureref \ref{figcombined} from which 
the following probabilistic labelling of statements can be derived.}\\

\RRR{
\begin{tabular}{lll}
$P(K_{\litb} = \inn)= P(\set{\labelling_4})$  & $(= 0.1)$ \\
$P(K_{\litb} = \off)=  P(\set{\labelling_1}) + P(\set{\labelling_2}) + P(\set{\labelling_3})$ & $(= 0.9)$ \\
$P(K_{\litc} = \und) = P(\set{\labelling_1}) + P(\set{\labelling_2}) + P(\set{\labelling_3})$  &  $(= 0.9)$\\
$P(K_{\litc} = \out) = P(\set{\labelling_4})$ & $(= 0.1)$ \\
$P(K_{\litd} = \inn) = P(\set{\labelling_4})$  & $(= 0.1)$ \\
$P(K_{\litd} = \und) =  P(\set{\labelling_1}) + P(\set{\labelling_2}) + P(\set{\labelling_3})$ & $(= 0.9)$ \\
\\
\end{tabular}}

\RRR{
 Altogether,  we have derived  that, using the skeptical stance of grounded labellings , the overall probability  of rejection of statement $\litc$ (\primaryquote{the program is running}) is  $0.1$ while the probability of acceptance of $\litd$ (\primaryquote{the program is not running}) is estimated at $0.1$. The status of $\litc$ and $\litd$ is \PBB{regarded} as undecided with a probability  $0.9$.
\hfill $\square$}
\end{myExample}

To summarise this section, we have considered different kinds of frames  for probabilistic argumentation, namely \PTFs, \PGFs and (the novel concept of) \PAFs, each frame  featuring a different probability space. While\PB{, with suitable labels,} \PAFs  are able to encompass the uncertainty about argumentation graphs expressed by \PTFs and \PGFs, \PAFs  can also express uncertainty about the final labelling outcome, thus covering a further uncertainty dimension. Considering thus \PAFs as a general approach for probabilistic argumentation, we derived some results concerning probabilistic measures of argument and statement labellings.

\section{On Uncertainty about Inclusion and Acceptance Status}
\label{secconstepa}

\PAFs are an expressive and flexible formalism able to capture various kinds of uncertainty \RRT{because}, in addition to the \primaryquote{traditional} acceptance labels $\IN$, $\OUT$ and $\UND$, we have introduced  labels $\ON$ and $\OFF$ to account for the \primaryquote{inclusion} status of arguments.
In this section, we further develop the analysis of this increased expressiveness by discussing, \PBB{at a technical level}  the relationships between our proposal and \PBB{the treatment given in \cite{DBLP:journals/ijar/Hunter13} of} two influential approaches to probabilistic  argumentation, namely the constellations approach (Subsection \ref{subsection:constellations}) and the epistemic approach (Subsection \ref{subsection:epistemic}), \PBB{finally leading to a} possible combination (Subsection \ref{subsection:combinations}). This is achieved by both carrying out a conceptual analysis and proving some technical properties, \PBB{which provide a basis for a wider discussion of relevant literature at a general level} in Section \ref{related}.

\subsection{Constellations approach}
\label{subsection:constellations}
In the constellations approach, originally investigated in \cite{DBLP:conf/tafa/LiON11},  every argument and attack of an argumentation graph is associated with a  \primaryquote{likelihood}.  In this section we discuss the development of this idea presented in \cite{DBLP:journals/ijar/Hunter13}, while other works related to the constellations approach (e.g. \cite{DBLP:journals/cas/Dondio14}) are discussed in Section \ref{related}.

We first recall the definition of probabilistic argumentation graphs (\PAGs) from \cite{DBLP:journals/ijar/Hunter13} 
where a probability is directly associated \RRT{with} each argument.
\begin{myDefinition}[Probabilistic argumentation graph] 
A \emph{probabilistic argumentation graph} (\PAG) is a tuple $\tuple{ \mathcal{A}, \defeat, P_{\mPAG} }$ where $\tuple{ \mathcal{A},\defeat}$ is an abstract argumentation graph and $P_{\mPAG}: \mathcal{A} \rightarrow [0, 1]$.
\end{myDefinition}
The sample space is left implicit, but  the  interpretation, \PBB{quoting \cite{DBLP:journals/ijar/Hunter13},} is that  given an abstract argumentation graph $G$ one \primaryquote{can treat the set of subgraphs of $G$ as a sample space, where one of the subgraphs is the \secondaryquote{true} argumentation  graph.}

In \cite{DBLP:journals/ijar/Hunter13}\RRT{,} the probability of a subgraph 
$H$ of $G$ induced by a set of arguments $\mathcal{A}_H \subseteq \mathcal{A}_G$ is \PBB{not derived \RRT{from any}  form of axiomatisation but is} \RRR{directly} defined as the following product (Definition 14 of \cite{DBLP:journals/ijar/Hunter13}):
\begin{equation}
\label{equation:constellations}
P_{\mPAG}(H)=\left(\prod_{A \in \mathcal{A}_H}P_{\mPAG}(A) \right) \times \left( \prod_{A \notin \mathcal{A}_H}(1 - P_{\mPAG}(A)\right).
\end{equation}

\RRR{Equation} \ref{equation:constellations} relies on the assumption that
for each argument $A$, the probability of $A$ appearing in the \primaryquote{true} argumentation graph is independent of the probability of appearance of every other argument.  This is motivated by the \emph{justification perspective} adopted in \cite{DBLP:journals/ijar/Hunter13}: \primaryquote{knowing that one argument is a justified point does not affect the probability that another is a justified point}. \PBB{This perspective leads to assume that an assignment of probability values to arguments is given as initial information. Quoting again} \cite{DBLP:journals/ijar/Hunter13}, for each argument $A$,  $P(A)$ \primaryquote{is the probability that $A$ exists
in an arbitrary full subgraph of $G$, and $1 - P(A)$ is the probability that $A$ does not exist in an arbitrary full subgraph of $G$}. 

The assumption of arguments being probabilistically independent contrasts with our proposal \RRR{where no assumptions of probabilistic independence is made \PBB{ and, actually,} the subargument relation constrains the appearance of arguments in a subgraph. In particular, in the rule-based context we proposed for \PTFs, a probability assignment over the subtheories of a theory $T$ is used as a starting point. Then\RRT{,} the probability of a subgraph $H$ is the sum of the probabilities of the subtheories generating $H$, as evidenced in Definition \ref{defPGFfromPTF}. Hence, our framework distinguishes and combines logical dependences and probabilistic dependences.}
\RRR{Nevertheless, in our setting, if the appearance of arguments is assumed independent, then any \PAG and its  notion of \primaryquote{true} argumentation graph \PB{can be captured by a \PGF as follows.}
}

\begin{myDefinition}[\PGF corresponding to a \PAG]
Given a \PAG  $\tuple{ \mathcal{A}, \defeat, P_{\mPAG}}$,
the corresponding \PGF is a tuple $\tuple{ G, \tuple{\Omega_{\mPGF}, F_{\mPGF}, P_{\mPGF}}}$ where $G = \tuple{\mathcal{A}, \defeat}$ and the   probability distribution  $P_{\mPGF}$  is such that $\forall S \in F_{\mPGF}$:  

$$ P_{\mPGF}( S ) =  \sum_{H \in S} P_{\mPAG}(H).$$
\label{PGFtoPAG}
\end{myDefinition}

\RRR{The} difference between \RRR{the  approach in \cite{DBLP:journals/ijar/Hunter13} and ours} can be explained by the fact that in the argumentation model based on classical logic adopted in 
\cite{DBLP:journals/ijar/Hunter13} there is no explicit notion of subargument: an argument is a pair $\tuple{\Phi, \alpha}$ where $\Phi$ is a minimal consistent set of formulae such that $\Phi \vdash \alpha$. In this perspective every argument is self-contained.
We note however that it may be the case that there are two arguments  $\tuple{\Phi, \alpha}$,  $\tuple{\Phi', \beta}$ such that $\Phi \subseteq \Phi'$; then \RRR{it may  appear problematic to assume that the appearance of $\tuple{\Phi', \beta}$ is independent of the appearance of $\tuple{\Phi, \alpha}$}.

A detailed discussion of the differences between logic-based and rule-based argumentation being beyond the scope of this paper, we remark that our approach provides a formal example of the legality constraints binding abstract representations like \PGFs and \PAFs, when considering the actual underlying argument construction process. This example can be useful as a starting point to investigate analogous legality constraints in other argumentation formalisms.
Moreover the labelling representation of the \RRR{constellations} approach through \PAFs based on $\{\ON,\OFF\}$\RRT{-}labellings simplifies the analysis of some basic properties, which, again, can be \RRT{employed} as a term of comparison in other formalisms.
This is illustrated by the following propositions, whose proofs are straightforward.

First, the probability of inclusion of an argument cannot be greater than the probability of inclusion of its subarguments.

\begin{myProposition}\label{propsubon}
Let $\tuple{G, \mathsf{legal}\mbox{-}\{\ON, \OFF\},\tuple{\Omega, F, P}}$ be a \PAF,
and let $A$ and $B$ denote two arguments in $\mathcal{A}_G$ such that $A$ is a  subargument of $B$. 
$$P(L_B = \ON) \leq P(L_A = \ON). $$
\end{myProposition}

Further, the probability of appearance of an argument $A$ is determined by the probability of the subtheories including the rules \RRT{utilised} in its \RRT{construction}.

\begin{myProposition}\label{propsubthon} 
\sloppy \qquad Let $\mathfrak{F}_1 = \tuple{T, \tuple{\Omega_{\mPTF}, F_{\mPTF}, P_{\mPTF}}}$ be a \PTF,  and $\mathfrak{F}_2 = \tuple{G_T, \mathcal{S}, \tuple{\Omega, F, P}}$ its corresponding   \PAF, i.e.  
$\mathfrak{F}_2 = \genPAF{\genPGF{\mathfrak{F}_1}}$.
For every argument $A$ in $\mathcal{A}_{G_T}$ it holds that 
$$P_{\mPAF}(L_A = \ON)=\sum_{\set{\subtheory \in \Sub(T) \mid \Rules(A) \subseteq \subtheory}}P_{\mPTF}(\set{\subtheory}).$$
\end{myProposition}

We suggest that properties analogous to those given  in propositions \ref{propsubon} and \ref{propsubthon} should be regarded as basic requirements in every proposal belonging to the constellations approach. \PB{Simple as they are, such properties are out of the scope of approach focused on abstract argumentation only (see Subsection \ref{subsection:abstractpa}) and, as exemplified by the above discussion of \cite{DBLP:journals/ijar/Hunter13}, they are sometimes overlooked.}\medskip

Turning to issues related to semantics evaluation, in Definition 15 of \cite{DBLP:journals/ijar/Hunter13}
the probability of a set of arguments $\mathcal{A}$ being an extension according to an argumentation semantics $X$, denoted $P^X_{\mPAG}(\mathcal{A})$, is defined as the sum of the probabilities of subgraphs entailing this extension.
The following equation reformulates the original definition in terms of labellings, given that the probability that a set of arguments $\mathcal{A}$ is an extension according to an argumentation semantics $X$ is the probability that all and only arguments in $\mathcal{A}$ are labelled $\IN$ according to an $X$-$\{\IN, \OUT, \UND\}$-labelling specification:
\begin{equation}
\label{equation:weird}
P^X_{\mPAG}(\mathcal{A}) = \sum_{H \in\mathcal{H}^X(\mathcal{A}) } P_{\mPAG}(H)
\end{equation}
where $\mathcal{H}^X(\mathcal{A}) = \{ H | H \in\subgraphs{G}: \mathcal{A} = \IN(\labelling) \mbox{\PB{~for some~}}  \labelling \in \mathcal{L}^X_H  \}$. On the basis of Definition \ref{PGFtoPAG}, $P^X_{\mPAG}(\mathcal{A})$ can be expressed equivalently in the context of  \PGFs given that  of $P_{\mPAG}(H) = P_{\mPGF}( \set{H} )$.\smallskip

\RRR{While any \PAG can be captured by a \PGF or the corresponding \PAF (Definition \ref{defPAFPGF})  with $\set{\ON,\OFF}$-labellings, we have to emphasise that the \selunc uncertainty captured by \PAFs is possibly distinct from the uncertainty captured by $P^X_{\mPAG}$, and provides a more fine-grained probabilistic evaluation of argument acceptance statuses,  as illustrated in Example \ref{exmaple:PAFPLF}. }

\begin{myExample}[continues=ex:neither2]
\label{exmaple:PAFPLF}
\RRR{Suppose the \PAG $\tuple{ \mathcal{A}, \defeat, P_{\mPAG}}$ and the \PAF $ \tuple{ \tuple{\mathcal{A}, \defeat}, \mathsf{preferred}\mbox{-}\set{\IN,\OUT, \UND, \OFF}, \tuple{\Omega_{\mPAF}, F_{\mPAF}, P_{\mPAF}}} $ have probability distributions as given in Figure \ref{figure:mappingxx}.\\
%
}

\begin{figure}[ht!]
\centering
  \begin{tikzpicture}[cell/.style={rectangle,draw=black}, nodes in empty cells]
  \matrix[
  matrix of math nodes,
  ](m)
  {  
  &  \mathsf{B} & \mathsf{C} &  P_{\mPAG} &  & & & & & & &  & & \mathsf{B} & \mathsf{C} & P_{\mPAF} & & \\ 
 \mathsf{H}_{1}&  \ON & \ON &  1 &  &  & & & & & &  & & |[fill=green!30]|\IN & |[fill=red!30]|\OUT 	& 0.4	\\ [-0.5ex]
 \mathsf{H}_{2}&   \ON & |[fill=black!30]|\OFF & 0 &  &   & & & & & &  & & |[fill=red!30]|\OUT & |[fill=green!30]|\IN 	& 0.6 \\ [-0.5ex]
 \mathsf{H}_{3}&   |[fill=black!30]|\OFF & \ON &  0 &  &  & & & & & &  & & |[fill=green!30]|\IN & |[fill=black!30]|\OFF   & 0   \\[-0.5ex]
 \mathsf{H}_4&     |[fill=black!30]|\OFF & |[fill=black!30]|\OFF &  0 &  &  & & & & & &  & & |[fill=black!30]|\OFF & |[fill=green!30]|\IN  & 0  \\[-0.5ex]
  &       &  &   &  &  & & & & & &  & & |[fill=black!30]|\OFF & |[fill=black!30]|\OFF  & 0 \\[-0.5ex]
  };
  
  \draw[-latex] (m-2-6.north west) -- (m-2-13.north east);
  \draw[-latex] (m-2-6.north west) -- (m-3-13.north east);
  \draw[-latex] (m-3-6.north west) -- (m-4-13.north east);  
  \draw[-latex] (m-4-6.north west) -- (m-5-13.north east);  
  \draw[-latex] (m-5-6.north west) -- (m-6-13.north east);  
\end{tikzpicture}
\caption{\PBB{An example of relation and comparison between \PAG and \PAF.}}
\label{figure:mappingxx}
\end{figure}

\noindent \RRR{We have that 
\begin{itemize}
\item  the distribution $P_{\mPAG}$ is such  that $P_{\mPAG}(\mathsf{H1})=1$, reflecting that  the inclusion of $\mathsf{B}$ and  $\mathsf{C}$ have probability $1$; and 
\item  the distribution $P_{\mPAF}$ is such that $P_{\mPAF}(\set{ \tuple{ \set{ \mathsf{B} }, \set{ \mathsf{C}}, \emptyset , \emptyset  }} ) = 0.4$ and $P_{\mPAF}(\set{ \tuple{ \set{ \mathsf{C} }, \set{ \mathsf{B}}, \emptyset, \emptyset   }}) = 0.6$, reflecting that the preferred extension $\set{ \mathsf{C}}$ is more likely  than  the preferred extension $\set{ \mathsf{B}}$. 
\end{itemize}}

\noindent \RRR{The  probability that the arguments $\mathsf{B}$ and $\mathsf{C}$ appear in a preferred extension/labelling is then different in the two approaches:
\begin{itemize}
\item the probability that $\set{ \mathsf{B} }$ is a preferred extension is 1, and equally for the set $\set{ \mathsf{C} }$,  i.e $P^{\mathsf{pr}}_{\mPAG}(\set{ \mathsf{B} })=  1$ and $P^{\mathsf{pr}}_{\mPAG}(\set{ \mathsf{C} })=  1$;
\item\sloppy the probability that $\mathsf{B}$ is labelled $\IN$ is 0.4, while the probability that $\mathsf{C}$ is labelled $\IN$ is 0.6,  i.e $P^{\mathsf{pr}}_{\mPAF}( L_\mathsf{B} = \IN )=  0.4$ and $P^{\mathsf{pr}}_{\mPAF}( L_\mathsf{C} = \IN )=  0.6$. 
\end{itemize}}
This example illustrates  that the \selunc uncertainty addressed with \PAFs is distinct from the uncertainty subject of the constellations approach, \RRR{and \PBB{supports} a more fine-grained probabilistic evaluation of  acceptance statuses of arguments}.
 \hfill $\square$
\end{myExample}

To recap, we have shown that the approach of probabilistic labellings can capture the constellations approach, but the latter cannot capture the former \RRR{since probabilistic labellings  allow a more fine-grained probabilistic evaluation of argument acceptance statuses.} \PB{Some basic properties, e.g. the fact that  the probability of inclusion of an argument cannot be greater than the probability of inclusion of its  subarguments, are directly derived in our framework.} It turns out that the approach of probabilistic labellings can be also related to another approach to probabilistic argumentation, namely the epistemic approach, as we will see next.

\subsection{Epistemic approach}
\label{subsection:epistemic}
In the epistemic approach, \primaryquote{the probability distribution over arguments is used directly to
identify which arguments are believed} \cite{DBLP:journals/ijar/Hunter13}. The idea is that in this case the argumentation framework is fixed but an agent has an \primaryquote{extra epistemic information}  to assign an epistemic probability, denoted in the following as $\EP$, to each argument. Formally, given a set of arguments $\argset$, $\EP: \argset \rightarrow [0, 1]$.
Given that this degree of belief is based on extra information, as exemplified in \cite{DBLP:journals/ijar/Hunter13}  it may be the case that an agent assigns a high degree of belief to an argument which would be rejected according to the acceptance labelling prescribed by the framework. Even more, in \cite{DBLP:journals/ijar/Hunter13}, it is assumed that the degrees of belief of an agent may even be inconsistent i.e. that $\EP$ may violate the probability axioms.
Dealing with inconsistencies is beyond the scope of the present paper, hence we restrict to consistent epistemic probabilities in the following discussion, showing how they can be related to our approach.

First, we \RRT{formalise} the probability space for epistemic probabilities in our context.

\begin{myDefinition}[Probabilistic epistemic frame]
A \emph{probabilistic epistemic frame} (\PEF) is a tuple $\tuple{G, \tuple{\Omega_{\mPEF}, F_{\mPEF}, P_{\mPEF}}}$ where $G=\tuple{\AR, \defeat, \support}$ denotes  an argumentation graph,  and  $\tuple{\Omega_{\mPEF}, F_{\mPEF}, P_{\mPEF}}$ is a probability space such that:
\begin{itemize}
\item the sample space $\Omega_{\mPEF}$ is the set of subsets of $\AR$, i.e. $\Omega =  2^\AR$;   
\item  the $\sigma$-algebra $F_{\mPEF}$ is the power set of $\Omega_{\mPEF}$, i.e. $F_{\mPEF} = 2^{\Omega_{\mPEF}}$;
\item  the  function $P_{\mPEF}$ from $F_{\mPEF}$ to $[0,1]$ is a probability distribution satisfying Kolmogorov axioms.
\end{itemize}
\end{myDefinition}

Each element of the sample space of a \PEF is a set of arguments, an option for the unconstrained choice of which arguments are believed. The epistemic probability, i.e.  the degree of belief, of an argument is the sum of the probabilities of the sets of arguments including it. So given a \PEF tuple  $\tuple{G, \tuple{\Omega_{\mPEF}, F_{\mPEF}, P_{\mPEF}}}$, we have:
\begin{equation}
\label{thimm:1}
\EP(A)=\sum_{\set{\genargset \PB{\in} \Omega_{\mPEF} \mid A \in \genargset}}\EP(\PB{\set{\genargset}}).
\end{equation}

Every PEF can be put in correspondence with a \PAF based on the $\{\IN, \OUT, \UND\}$ labels, where  the label $\UND$ is not used. Basically the epistemic probability of each element of the sample space of a \PEF, i.e. of each set of arguments, is put in correspondence with the probability of a labelling where the members of the set are labelled $\IN$ and the other arguments are labelled $\OUT$.

\begin{myDefinition}[\PAF corresponding to a \PEF]
\sloppy Given a \PEF $\tuple{G, \tuple{\Omega_{\mPEF}, F_{\mPEF}, P_{\mPEF}}}$,
 its corresponding \PAF is  a \PAF tuple  $\tuple{G,  \set{\IN,\OUT,\UND} , \tuple{\Omega_{\mPAF}, F_{\mPAF}, P_{\mPAF}}}$ whose  probability distribution  $P_{\mPAF}$  is such that 
$\forall S \in F_{\mPAF}$:  
$$P_{\mPAF}(S) =  \sum_{ \mathcal{E} \in \mathfrak{E} } P_{\mPEF}(\PB{\set{\genargset}}) $$
where $\mathfrak{E} = \{\genargset \PB{\in} \Omega_{\mPEF} | \exists \labelling \in S: \IN(\labelling) = \genargset,\, \forall A \not \in \genargset: \labelling(A) = \OUT  \} $.
\end{myDefinition}
\vspace{1.5cm}

\begin{myDefinition}[\PEF corresponding to a \PAF]
\item 
Given a \PAF $\tuple{G,\ALL\mbox{-}\set{\IN,\OUT,\UND}, \tuple{\Omega_{\mPAF}, F_{\mPAF}, P_{\mPAF}}}$, its corresponding \PEF is a \PEF tuple  $\tuple{G, \tuple{\Omega_{\mPEF}, F_{\mPEF}, P_{\mPEF}}}$ whose  probability distribution $P_{\mPEF}$ is such that 
$\forall S \in F$: 
$$P_{\mPEF}(S) = \sum_{L \in \mathcal{L}} P_{\mPAF}( \PB{\set{\labelling}} )$$
where $\mathcal{L} = \set{ \labelling \in \setofalllabels{G}{\set{\mathsf{IN},\mathsf{OUT},\mathsf{UN}}} \mid \exists \genargset \in S: \genargset =\IN(\labelling)  }$.
\end{myDefinition}

So, each labelling $\labelling$ in the sample space of the \PAF is put in correspondence with the set of arguments $\IN(\labelling)$ belonging to the sample space of a corresponding \PEF.
In particular, from \Equationref \ref{thimm:1}  and taking into account Proposition (\ref{eqlabsingarg}), we directly obtain for every argument $A$ in $\mathcal{A}$:
\begin{equation}
\EP(A) = P_{\mPAF}(L_A =\IN).
\end{equation}

\begin{myExample}[continues=ex:neither2]
\label{example:PEF}
\sloppy Suppose the \PEF $\tuple{\mathsf{G}, \tuple{\Omega_{\mPEF}, F_{\mPEF}, P_{\mPEF}}}$ and the \PAF $ \tuple{ \mathsf{G}, \mathsf{preferred}\mbox{-}\set{\IN,\OUT,\UND}, \tuple{\Omega_{\mPAF}, F_{\mPAF}, P_{\mPAF}}} $, where the argumentation graph $\mathsf{G}$ is pictured in Figure  \ref{fig:examplesupportcons}.
Suppose that the probability acceptance of $\mathsf{B}$ and  $\mathsf{C}$ are $0.4$ and $0.6$, respectively, then we have $P_{\mPEF}(\mathsf{B}) = P(L_{\mathsf{B}} = \IN) = 0.4$ and $P_{\mPEF}(\mathsf{C}) = P(L_{\mathsf{C}} = \IN ) = 0.6$.
In this example,  values for the probabilities of acceptance are chosen to be consistent with respect to the preferred extension/labellings. However, a distribution $P_{\mPEF}$ does not necessarily satisfy this constraint. For example one may set $P_{\mPEF}(\{\mathsf{B}, \mathsf{C}\}) >0$. \hfill $\square$
\end{myExample}

As mentioned in Example \ref{example:PEF}, the epistemic approach includes the consideration of possibly \PB{anomalous} probability distributions, and for this reason one may/should consider some 
desirable properties  for epistemic probabilities.
In that regard, the proposed correspondences between \PEFs and \PAFs preserve some desirable properties stated in  \cite{DBLP:journals/ijar/Hunter13} for such epistemic probabilities.
In particular, an epistemic probability $\EP$ is called \primaryquote{coherent} if for every pair of arguments $A$ and $B$ such that $A$ attacks $B$ it holds that $\EP(A) + \EP(B) \leq 1$. Coherence implies a weaker property, called \primaryquote{rationality} in \cite{DBLP:journals/ijar/Hunter13}, namely that if $\EP(A) > 0.5$ then $\EP(B) \leq 0.5$.
In our framework, coherence of $P(L_A =\IN)$ is ensured, provided that the labelling satisfies the minimal property of conflict freeness.

\begin{myProposition}
\label{proposition1}
Let $\tuple{G, \CF\mbox{-}\{\IN, \OUT, \UND\},\tuple{\Omega, F, P}}$ be a \PAF.
 Let $A$ and $B$ denote two arguments in $\mathcal{A}_G$ such that $B$ attacks $A$.
$$P(L_A = \IN) + P(L_B = \IN) \leq 1. $$
\end{myProposition}

\begin{proof}
Let us denote $\mathcal{L}_A$ the set of labellings where $\labelling(A)=\IN$, $\mathcal{L}_B$ the set of labellings where $\labelling(B)=\IN$, and  $\mathcal{L}$ the complement  set of labellings $\Omega \backslash (\mathcal{L}_{A}\cup \mathcal{L}_{B} )$. \PB{By conflict freeness, $\mathcal{L}_{A}$ and $\mathcal{L}_{B}$ are disjoint,} we have thus: $P(\mathcal{L}_A) + P(\mathcal{L}_B) +  P(\mathcal{L}) = 1$. By definition, $P(\mathcal{L}_A) = P(L_A=\IN)$ and $P(\mathcal{L}_B) = P(L_B=\IN)$. We also have that  $P(\mathcal{L}) \geq 0$. Therefore
$P(L_A = \IN) + P(L_B = \IN) \leq 1$. \hfill $\square$
\end{proof}  

Further it is easy to see that $P(L_A =\IN)$ satisfies other properties under additional mild assumptions.
First of all an epistemic probability is said to be \primaryquote{founded} \cite{DBLP:conf/ecai/HunterT14a} if every argument not receiving any attack has probability 1. This property is directly achieved if one assumes complete labellings.

\begin{myProposition}
\label{proposition2}
Let $\tuple{G,\mathsf{complete}\mbox{-}\{\IN, \OUT, \UND\},\tuple{\Omega, F, P}}$ be a \PAF. Let $A$ be any argument in $\mathcal{A}_G$ such that no arguments attack $A$.
$$P(L_A = \IN) = 1.$$
\end{myProposition}

\begin{proof}
In the case of the $\mathsf{complete}$-$\{\IN, \OUT, \UND\}$-labelling, for any labelling $\labelling$ in the sample space $\Omega$, we have $\labelling(A)=\IN$, therefore $P(L_A = \IN) = 1 $.
In the case of the $\mathsf{complete}$-$\{\IN, \OUT, \UND,\OFF\}$-labelling, for any labelling $\labelling$ in the sample space $\Omega$, we have $\labelling(A) \neq \IN$ only if $\labelling(A) = \OFF$, therefore $P(L_A = \IN) +  P(L_A = \OFF)= 1 $.\hfill $\square$
\end{proof}

While some properties considered in the literature have a direct counterpart within the proposed framework, as discussed above, for some others the situation is not so clearcut.
Consider for instance \primaryquote{optimistic}  distributions proposed in \cite{DBLP:conf/ecai/HunterT14a}.
Letting $\mathcal{B}$ denote the set of all the attackers of an argument $A$, a distribution is optimistic if for every argument $A$ 
\begin{equation}
\label{equation:optimistic}
P(L_A =\IN) \geq 1 -  \sum_{B \in \mathcal{B}} P(L_B =\IN).
\end{equation}
Considering $\{\IN, \OUT, \UND\}$-labellings,  the constraint may become trivial in the case the probability of some of the attackers of $A$ is high (i.e. when $\sum_{B \in \mathcal{B}} P(L_B =\IN) \geq 1$).
However, the constraint does not always hold.  For instance, if for every attacker $B$ of $A$ the distribution is such that  $P(L_B = \IN) =0$, then the constraint becomes an equality, i.e. $P(L_A =\IN) = 1$ and is not necessarily satisfied.

%
%

So, epistemic probabilities can be put in direct correspondence with \PAFs based on $\{\IN, \OUT, \UND\}$-labellings, where the focus is on arguments labelled $\IN$ only.
The basic properties of coherence and foundedness \PBX{introduced for}  epistemic probabilities have a counterpart in \PAFs in terms of properties of labellings, while other properties of epistemic probabilities like optimism have no counterparts and their conceptual status \PBB{is open to discussion} (see \cite{DBLP:conf/comma/BaroniGV14} for a more extensive analysis on this point).\smallskip

Moreover, our proposal allows one to directly extend the notion of epistemic probability from arguments to argument conclusions: it is rather natural to assume that the final goal of an agent is to express his/her degrees of belief on the statements about which arguments are built, rather than just about arguments.

As a matter of fact, in \cite{DBLP:journals/ijar/Hunter13} the probability of a claim is derived from the probability on the \PBX{underlying} logical models, which is used as input information, independently of the probability of arguments.
Indeed, in this \PBX{context}, the probability of each claim can be computed directly from the input information without  argument construction. The probability of each argument is also computed from the input information and has no effect on the probability of the relevant claim (cf. Proposition 5 of \cite{DBLP:journals/ijar/Hunter13} where it is shown that the probability of a claim is not \RRT{less} than the probability of any argument supporting it).
In our rule-based approach, using the probability on subtheories as input, the probability of claims is evaluated through the argument construction process, which, more coherently with the notion of argument-based reasoning, plays a central role.
Further, our approach allows a fine-grained uncertain evaluation of claims, with probabilities associated \RRT{with} all the relevant labels, while in \cite{DBLP:journals/ijar/Hunter13} a binary evaluation of claims is implicitly adopted.\medskip

To recap, the epistemic approach  can be put in correspondence with the approach of probabilistic labellings \RRT{through}  \PAFs based on $\{\IN, \OUT, \UND\}$-labellings. In such \PAFs, some desirable epistemic properties like \primaryquote{coherence} and \primaryquote{foundedness} \PB{are guaranteed under some basic assumptions, like conflict-freeness}, whereas \PBB{other} properties do not hold in general. Then, while in \cite{DBLP:journals/ijar/Hunter13} the probability of a claim can be computed without argument construction, in our proposal the probability of claims is evaluated through the construction of arguments, an account which appears more coherent with the notion of argument-based reasoning. Eventually,  \RRT{probabilistic labellings} can capture a  combination of the constellations and epistemic approaches, as we will see next.

\subsection{Combining the constellations and the epistemic approach}
\label{subsection:combinations}
The constellations and epistemic approaches  have been \PBX{treated} separately in the previous literature as referring to different types of uncertainty.
It is worth remarking, however, \PBB{that they are, so to say, orthogonal}. \RR{Consequently,} a situation where both types of uncertainty are present cannot be excluded: an agent may be uncertain about which argumentation graph to actually consider and, at the same time, may have different degrees of belief about the arguments accepted \PBB{in the actual (unknown)} argumentation graph.

\PAFs based on $\{\IN, \OUT, \UND, \OFF\}$\RRT{-}labellings provide a way to encompass jointly the two kinds of uncertainty.
In fact, extending to this kind of labelling the reasoning illustrated in the previous sections, the probability of inclusion of an argument $A$ corresponds to $P(L_A \neq \OFF)$ (also denoted as  $P(L_A = \ON)$), while the epistemic probability of $A$ corresponds to $P(L_A =\IN)$.

The correspondences between the different probabilistic approaches preserve  essential properties in the \RRT{generalised} framework. 
%
%
%
%
For example, propositions \ref{proposition1} and \ref{proposition2} can be turned into propositions \ref{proposition1.2} and \ref{proposition2.2} to take into account the inclusion status of  arguments.
\begin{myProposition}
\label{proposition1.2}
Let $\tuple{G, \mathsf{\CF}\mbox{-}\{\IN, \OUT, \UND, \OFF\},\tuple{\Omega, F, P}}$ be a \PAF. 
Let $A$ and $B$ denote two arguments in $\mathcal{A}_G$ such that $B$ attacks $A$.
$$P(L_A = \IN) + P(L_B = \IN) \leq 1. $$
\end{myProposition}

\begin{myProposition}
\label{proposition2.2}
Let $\tuple{G,  \mathsf{complete}\mbox{-}\{\IN, \OUT, \UND, \OFF\} , \tuple{\Omega, F, P}}$ be a \PAF. Let $A$ be any argument in $\mathcal{A}_G$ such that no arguments attack $A$.
$$P(L_A = \IN) +  P(L_A = \OFF)= 1. $$
\end{myProposition}
Other propositions can be derived, for example\RRT{,} to relate the probability of acceptance and inclusion of arguments, cf.  \Propositionref \ref{propsubon}.

\begin{myProposition}\label{propsubon2}
\RRR{Let  $\tuple{G,  \mathsf{\CF}\mbox{-}\{\IN, \OUT, \UND, \OFF\} , \tuple{\Omega, F, P}}$  be a \PAF, and let $A$  denote any argument in $\mathcal{A}_G$.
$$P(L_A = \IN) \leq P(L_A = \ON). $$}
\end{myProposition}

\begin{myCorollary}\label{propsubon3}
\RRR{Let $\tuple{G,   \mathsf{\CF}\mbox{-}\{\IN, \OUT, \UND, \OFF\} , \tuple{\Omega, F, P}}$ be a \PAF, and let $A$ and $B$ denote two arguments in $\mathcal{A}_G$ such that $A$ is a  subargument of $B$. 
$$P(L_B = \IN) \leq P(L_A = \ON). $$}
\end{myCorollary}

\RRR{Eventually}, our proposal is able to capture in a formal way the effect of the combination of these different kinds of uncertainty on argument conclusions, \PBB{in particular in the case where a multi-labelling semantics is adopted, as illustrated in the following example.}
\medskip 

\begin{myExample}[continues=example:basis]
\label{example:runninggg}
\RRR{Suppose that the research scientist, in light of some previous experiences, \RRT{estimates} that the acceptance of $\mathsf{C}$ is twice more credible than the acceptance of $\mathsf{D}$, whenever the choice is open, \PBB{namely in the cases where they are both labelled $\UND$ by grounded semantics in the \RR{example} at the end of Section \ref{sec:probSetting}. To \RRT{convey} this, she decides to adopt preferred $\{\IN, \OUT, \UND, \OFF\}$-labellings,}
\PBB{\RRT{and thus} the  labellings for arguments and statements} illustrated in \Figureref \ref{figcombined2}  are now deemed possible.}

\RRR{
\begin{figure}[ht!]
\centering
\begin{tabular}{ c c c c c c c c c c c c c c }
 & \small{$\mathsf{B1}$} & \small{$\mathsf{B2}$} & \small{$\mathsf{B}$} & \small{$\mathsf{C}$} & \small{$\mathsf{D}$} &  & \small{$\neg \mathsf{b1}$} & \small{$\neg \mathsf{b2}$} & \small{$\neg \mathsf{b}$} & \small{$\litc$} & \small{$\litd$} \\
$\labelling_1$ & \cellcolor[HTML]{C0C0C0}$\OFF$      & \cellcolor[HTML]{C0C0C0}$\OFF$       & \cellcolor[HTML]{C0C0C0}$\OFF$  & \cellcolor[HTML]{FFCCC9}$\OUT$ & \cellcolor[HTML]{9AFF99}$\IN$    &  &  \cellcolor[HTML]{C0C0C0}$\off$      & \cellcolor[HTML]{C0C0C0}$\off$       & \cellcolor[HTML]{C0C0C0}$\off$  & \cellcolor[HTML]{FFCCC9}$\out$ & \cellcolor[HTML]{9AFF99}$\inn$  \\
$\labelling_2$ &  \cellcolor[HTML]{C0C0C0}$\OFF$      & \cellcolor[HTML]{C0C0C0}$\OFF$       & \cellcolor[HTML]{C0C0C0}$\OFF$  & \cellcolor[HTML]{9AFF99}$\IN$ &     \cellcolor[HTML]{FFCCC9}$\OUT$  & & \cellcolor[HTML]{C0C0C0}$\off$      & \cellcolor[HTML]{C0C0C0}$\off$       & \cellcolor[HTML]{C0C0C0}$\off$  & \cellcolor[HTML]{9AFF99}$\inn$ &     \cellcolor[HTML]{FFCCC9}$\out$  \\
$\labelling_3$ & \cellcolor[HTML]{9AFF99}$\IN$      &\cellcolor[HTML]{C0C0C0}$\OFF$      & \cellcolor[HTML]{C0C0C0}$\OFF$ & \cellcolor[HTML]{FFCCC9}$\OUT$ & \cellcolor[HTML]{9AFF99}$\IN$   & &  \cellcolor[HTML]{9AFF99}$\inn$      &\cellcolor[HTML]{C0C0C0}$\off$      & \cellcolor[HTML]{C0C0C0}$\off$ & \cellcolor[HTML]{FFCCC9}$\out$ & \cellcolor[HTML]{9AFF99}$\inn$   \\
$\labelling_4$ & \cellcolor[HTML]{9AFF99}$\IN$      &\cellcolor[HTML]{C0C0C0}$\OFF$      & \cellcolor[HTML]{C0C0C0}$\OFF$ &  \cellcolor[HTML]{9AFF99}$\IN$     & \cellcolor[HTML]{FFCCC9}$\OUT$  & & \cellcolor[HTML]{9AFF99}$\inn$      &\cellcolor[HTML]{C0C0C0}$\off$      & \cellcolor[HTML]{C0C0C0}$\off$ &  \cellcolor[HTML]{9AFF99}$\inn$     & \cellcolor[HTML]{FFCCC9}$\out$   \\
$\labelling_5$ & \cellcolor[HTML]{C0C0C0}$\OFF$  &  \cellcolor[HTML]{9AFF99}$\IN$    & \cellcolor[HTML]{C0C0C0}$\OFF$ & \cellcolor[HTML]{FFCCC9}$\OUT$ & \cellcolor[HTML]{9AFF99}$\IN$    & &  \cellcolor[HTML]{C0C0C0}$\off$  &  \cellcolor[HTML]{9AFF99}$\inn$    & \cellcolor[HTML]{C0C0C0}$\off$ & \cellcolor[HTML]{FFCCC9}$\out$ & \cellcolor[HTML]{9AFF99}$\inn$  \\
$\labelling_6$ & \cellcolor[HTML]{C0C0C0}$\OFF$  &  \cellcolor[HTML]{9AFF99}$\IN$    & \cellcolor[HTML]{C0C0C0}$\OFF$ &  \cellcolor[HTML]{9AFF99}$\IN$   &    \cellcolor[HTML]{FFCCC9}$\OUT$  & & \cellcolor[HTML]{C0C0C0}$\off$  &  \cellcolor[HTML]{9AFF99}$\inn$    & \cellcolor[HTML]{C0C0C0}$\off$ &  \cellcolor[HTML]{9AFF99}$\inn$   &    \cellcolor[HTML]{FFCCC9}$\out$   \\
$\labelling_7$ &  \cellcolor[HTML]{9AFF99}$\IN$  &  \cellcolor[HTML]{9AFF99}$\IN$    & \cellcolor[HTML]{9AFF99}$\IN$ & \cellcolor[HTML]{FFCCC9}$\OUT$ & \cellcolor[HTML]{9AFF99}$\IN$   & & \cellcolor[HTML]{9AFF99}$\inn$  &  \cellcolor[HTML]{9AFF99}$\inn$    & \cellcolor[HTML]{9AFF99}$\inn$ & \cellcolor[HTML]{FFCCC9}$\out$ & \cellcolor[HTML]{9AFF99}$\inn$ \\
\end{tabular}
\caption{\PBB{Argument and statement labellings with preferred semantics.}}
\label{figcombined2}
\end{figure}
}

\noindent Then, the different argument labellings are assigned the following probabilities:\\


%

\RRR{
\begin{tabular}{lll}
$P(\set{\labelling_1})= ( 1- \vBone) \times (1 - \vBtwo) \times (1/3)$ & $(\approx 0.13)$ \\ 
$P(\set{\labelling_2})= ( 1- \vBone) \times (1 - \vBtwo) \times (2/3) $ & $(\approx 0.27)$ \\
$P(\set{\labelling_3})= \vBone \times (1 - \vBtwo) \times (1/3)$ & $(\approx 0.13)$\\ 
$P(\set{\labelling_4})= \vBone \times (1 - \vBtwo) \times (2/3)$ & $(\approx 0.27)$\\
$P(\set{\labelling_5})= (1 - \vBone) \times \vBtwo \times (1/3) $ & $(\approx 0.03)$\\
$P(\set{\labelling_6})= (1 - \vBone) \times \vBtwo \times (2/3)$ & $(\approx 0.07)$\\ 
$P(\set{\labelling_7})=  \vBone \times \vBtwo  $ & $(= 0.1)$\\
\\
\end{tabular}}

\noindent The following probabilistic labelling of statements can then be derived.\\

\RRR{
\begin{tabular}{lll}
$P(K_{\litb} = \inn)= P(\set{\labelling_7})$  & $(= 0.1)$ \\
$P(K_{\litb} = \off)= 1-P(\set{\labelling_7})$ & $(= 0.9)$ \\
$P(K_{\litc} = \inn) = P(\set{\labelling_2}) +  P(\set{\labelling_4}) +  P(\set{\labelling_6}) $  & $(\approx 0.61)$\\
$P(K_{\litc} = \out) = 1- P(K_{\litc} = \inn)$ & $(\approx 0.39)$ \\
$P(K_{\litd} = \inn) = P(\set{\labelling_1}) +  P(\set{\labelling_3}) +  P(\set{\labelling_5}) +  P(\set{\labelling_7})$ & $(\approx 0.39)$ \\
$P(K_{\litd} = \out) = 1- P(K_{\litd} = \inn)$ & $(\approx 0.61)$ \\
\\
\end{tabular}}

\RRR{
On the basis of arguments and their approximate probabilities of being included or accepted in an argumentative reasoning,  we have derived  that the overall probability  of acceptance of statement $\litc$ (\primaryquote{the program is running}) is approximately   $0.61$ while the probability of acceptance of $\litd$ (\primaryquote{the program is not running}) is estimated at $0.39$.
\hfill $\square$}
\end{myExample}

\RRR{To sum up this section, probabilistic labellings can capture the constellations approach but \RRT{not vice versa} \RRT{because} probabilistic labellings yield a more fine-grained probabilistic evaluation of argument acceptance statuses. Probabilistic labellings  can also  be put in correspondence with the epistemic approach, and \PBX{provide a formal} tool to back or question its assertions. Finally, the constellations and epistemic approaches can be seamlessly combined using probabilistic labellings. }

\section{\PBX{Discussion of Related Literature}}
\label{related}
 \RRT{We relate
in this section} our proposal with \PBB{other relevant literature on probabilistic argumentation}.
We classify these works into three  groups: probabilistic structured argumentation (Subsection \ref{subsection:structuredpa}), probabilistic abstract argumentation (Subsection \ref{subsection:abstractpa}), and other approaches connecting argumentation and probability at different levels  (Subsection \ref{subsec:others}).

\subsection{Probabilistic structured argumentation}
\label{subsection:structuredpa}

A significant amount of work on probabilistic argumentation concerned structured argumentation where the origin and structure of arguments are explicitly dealt with.  These are discussed next, distinguishing between those which, like in our proposal, use or can be easily related to  abstract argumentation frameworks for argument or statement acceptance evaluation and those which have little or no \PB{relation with this formalism}.

\subsubsection{Probabilistic structured argumentation related to Dung's framework} 

Several structured probabilistic argumentation constructs in the literature use or can be easily related to Dung's framework, as a component for the evaluation of argument or statement acceptance in the context of a staged process analogous to the one presented in this paper.

For instance,  \RRR{the probability of the defeasible status of a conclusion is \RRT{employed} in \cite{Roth:2007} to determine strategies to maximise the chances of winning legal disputes, in the context of a} multi-layer argumentation formalism based on Defeasible Logic, which can be interpreted in terms of Dung's framework  \cite{jlc:argumentation,comma16}. In this work, the probability of the defeasible status of a conclusion is the sum of cases where this status is entailed. A case is a set of premises which are assumed independent; the treatment of uncertainty in this context can be seen as  a restricted account of \PTFs.

\RRR{The probability to win a legal dispute, and thus the probability that arguments and  statements get accepted by an adjudicator is explored  in \cite{Riveret:2007}, on the basis of a probabilistic variant of  a rule-based argumentation framework  akin to the framework in \cite{doi:10.1080/11663081.1997.10510900}.} The proposal has a focus on the computation of probabilities of acceptance reflecting moves in a dialogue, with a particular emphasis on the distinction between the probability of the event of construction of an argument and the probability of acceptance of an argument. 
Premises of  arguments are assumed independent, so that the probability of an argument is the product of the probability of its premises. 
\cite{DBLP:journals/ail/RiveretRS12} further developed the idea of computing the probability of the event of an argument to a rule-based argumentation framework where premises are rules, a treatment which is similar to \PTFs. The setting in \cite{DBLP:journals/ail/RiveretRS12} relaxes the assumption of independent rules by associating   a potential \RRT{with} theories. The potentials are then learned \RRT{through} reinforcement learning to match the softmax policies of reinforced learning agents. 

\RRR{A probabilistic development of a rule-based argumentation system inspired from the ASPIC$^+$ formalism is proposed in \cite{DBLP:conf/at/Rienstra12}}. The framework shares with \cite{DBLP:journals/ail/RiveretRS12} multiple intuitions and similar results, in particular a treatment of probabilistic rules which is close to our \PTFs, \RRT{and} a notion of theory state which corresponds to our notion of subtheory. 

A probabilistic argumentation formalism \RRT{built} on Assumption-Based Argumentation \cite{Bondarenko:1997:10.1016/S0004-3702(97)00015-5} is introduced \RRR{in \cite{Dung:2010:TAJ}} for jury-based dispute resolution.
Arguments are constructed from a set of probabilistic rules and they are  evaluated in the context of probabilistic frameworks \RRT{with} Dung's grounded semantics.
Every juror has a different probability space  to model the fact that they may reach different probabilistic conclusions for the same case.

The present work can be regarded as a systematic generalisation of the ideas presented in these papers at a more initial level, with dedicated formalisms and/or within specific application contexts. It provides a framework where these proposals can be placed, their structure analysed in a principled way and their properties investigated in a domain independent manner.

\subsubsection{Other probabilistic structured  argumentation}

A few other works in the literature \PBX{deal with} probabilistic argumentation within formalisms \RRT{which are} not relying on Dung's framework for the evaluation of arguments.

A probabilistic setting for non-monotonic reasoning is  \RRR{investigated in \cite{DBLP:journals/japll/Haenni09}}.
In this \PBX{proposal}, uncertainty concerns the truth values assumed by some propositional variables called probabilistic variables which are a strict subset of the whole set of variables of interest for an agent. A total assignment of truth values to the probabilistic variables is called a scenario, and basically the agent is uncertain about which scenario to choose.
The agent is also equipped with a knowledge base $\Phi$, which is assumed to be certain, and the chosen scenario must be consistent with $\Phi$, i.e. the scenario and $\Phi$ together must not entail contradictory conclusions.
A consistent scenario is called an argument for a conclusion $\lit$ if, together with $\Phi$, it entails the truth of the conclusion, while it is called counterargument for $\lit$ if, together with the given knowledge base, it entails the falsity of the conclusion. In this sense, an argument can be \RRT{understood} as a consistent set of assumptions, somehow similarly to assumption-based argumentation, but with the difference that this set of assumptions, being a scenario, is exhaustive. 
Note that inconsistent scenarios are ruled out and that a consistent scenario may neither be an argument nor a counterargument for $\lit$, if it does not entail neither the truth nor the falsity of $\lit$.
Moreover, a scenario may be, at the same time, an argument for several conclusions and a \RT{counterargument} for several others. According to this peculiar notion of argument, there is no notion of argument acceptance in this context, since every scenario stands alone and the arguments and counterarguments corresponding to each scenario cannot be in conflict. For this reason the approach focuses on argument conclusions: on the basis of the probability of each scenario, it is possible to define the probability of a conclusion being supported, being rejected and of being neither supported nor rejected. From these values the notions of degree of support, degree of possibility, and degree of ignorance of a conclusion are defined.
While, as  evidenced above, this \PBX{proposal} adopts a quite specific notion of argument with respect to the literature, it may be remarked that some of its elements can be put in correspondence with our approach: uncertainty about the probabilistic variables can be regarded as a quite restricted case of \PTF, \RRT{and} the degrees of support, possibility, and ignorance can be connected  to, respectively, the probability for a statement of being labelled $\inn$, of not being labelled $\out$ and of being labelled $\unp$.

\RRR{ The probability on statement statuses is catered for in \cite{TLP:9011718}}, tackling similar issues as introduced in \cite{Riveret:2007}, but without concerns for reflecting the structure of dialogues and without using Dung's framework. The authors defined  a probabilistic  argumentation logic  and implemented it with  the language CHRiSM \cite{DBLP:journals/corr/abs-1007-3858}, which is a rule-based probabilistic logic programming language based on Constraint Handling Rules (CHR) \cite{Fruhwirth:2009:CHR:1618539} associated \RRT{with} a high-level probabilistic programming language called PRISM \cite{DBLP:journals/jiis/Sato08a}\footnote{This language must not be confused with  the    system PRISM for probabilistic model checking developed by \cite{Hinton:2006:PTA:2182014.2182053}.}. They discussed how it can be seen as a probabilistic generalization of the defeasible logic proposed by \cite{Nute01defeasiblelogic}, and showed how it provides a method to determine the initial probabilities from a given body of precedents. 
Rules in CHRiSM have an attached probability and each rule is applied with respect to its probability. Thus, also in this case\RRT{,} probabilistic uncertainty can \RRT{be related to a sort of \PTF}. 
%

\subsection{Probabilistic abstract argumentation}
\label{subsection:abstractpa}
Several works focus on extending Dung's formalism of abstract argumentation frameworks with probabilistic information, independently of any underlying mechanism of argument construction and of any intended goal of statement evaluation.

\subsubsection{Constellations approach}

\PB{The proposal in \cite{DBLP:conf/tafa/LiON11}, which can be regarded as the starting point of the constellations approach,}  makes a strong assumption of independence \RRT{on} the inclusion of arguments, as already commented in Section \ref{secconstepa}.
It has also to be mentioned that, differently from our proposal,  attacks are also  uncertain  in \cite{DBLP:conf/tafa/LiON11}: given a pair of arguments $A$ and $B$ the fact that $A$ attacks $B$ is subject to a probabilistic evaluation and this attack may be present or not in different frameworks where both $A$ and $B$ are included.
\RRT{In the rule-based construction of arguments} adopted in \RRT{the present paper}, we do not consider uncertainty about attacks, since in our context the fact that an argument attacks another is deterministic.
Encompassing uncertainty about attacks through a \RRT{generalised} labelling approach, where also attacks are labelled, does not appear to pose specific technical difficulties and can be \RRT{exploited in}  future work.

\RRT{In the context of the formal framework of \cite{DBLP:conf/tafa/LiON11},} further work \cite{DBLP:conf/ijcai/FazzingaFP13} \RRT{investigates} the computational complexity of computing the probability of a family of binary predicates, whose truth value depends on the actual structure of the argumentation framework. These predicates concern the fact that a given set of arguments is an extension according to a given semantics, so they are parametric with respect to the chosen semantics. Admissible and stable semantics turn out to be tractable in this respect, while other semantics are not.
This analysis has been extended in \cite{DBLP:journals/tocl/FazzingaFP15} where it is shown that lifting the assumption of independence leads to intractability in all cases. \RRT{The analysis is} complemented in \cite{Fazzinga:2016:EEP:2868292.2868384} by  \RRT{Monte-Carlo simulations} for the efficient estimation of the probability values in the intractable cases.
\RRR{In a parallel line of research,  the semantics of probabilistic argumentation is formulated in \cite{Liao2015,DBLP:journals/corr/LiaoXH16} by characterising subgraphs in relation to an extension, possibly leading  to more efficient algorithms to deal with the complexity of the constellations approach.}

The limits of the independence assumption about the inclusion of arguments is also acknowledged  in \cite{Lietalrelaxing}, \RRR{which utilises evidential argumentation frameworks/systems \cite{DBLP:conf/comma/OrenN08,DBLP:conf/comma/PolbergO14} to lift the assumption.}
 The formalism uses a special argument, denoted as $\eta$, to represent evidence and introduces a relation of support between arguments, which is regarded as uncertain, leading to the notion of Probabilistic Evidential Argumentation Frameworks (PrEAFs), where a probability value is associated \RRT{with} each support. 
Our proposal does resort to neither specific notions like incontrovertible premises nor to a specific evidence-based interpretation of the notion of support.  We rather point out that the generic notion of subargument creates by itself a dependence among arguments and induces some legality constraints, as explained in Section \ref{sec:genlabelling}.

%

The work presented in \cite{DBLP:journals/cas/Dondio14} shows some closer similarities to ours.
First, it introduces a notion of probabilistic argumentation framework where a joint probability distribution $P$ over the set of arguments is given.
For any argument $A$, $P(A)$ \RRT{denotes} the probability that $A$ holds \primaryquote{in isolation, before the dialectical process starts.} This notion is coherent with the \RRR{constellations} approach and with the kind of uncertainty we capture with $\set{\ON,\OFF}$-labellings.
It is worth remarking that, by assuming a joint probability distribution over arguments, this work does not rely on any independence assumption.
Similarly to our approach, \RRT{\cite{DBLP:journals/cas/Dondio14} \RRT{advances} a distinct probability assessment with reference to $\set{\IN,\OUT,\UND}$-labellings for arguments}.
Several differences are however worth pointing out.
First, at a representation level, the label $\OUT$ \RRT{also covers} the cases where an argument is not included in a framework, namely the cases which we separately label as $\OFF$.
Second, only grounded and preferred semantics are explicitly considered \RRT{in \cite{DBLP:journals/cas/Dondio14}}.
Third, and more importantly, in \cite{DBLP:journals/cas/Dondio14}, the probabilities of acceptance, rejection and undecidedness of an argument $A$ (i.e. the probabilities that $A$ is labelled $\IN$, $\OUT$, or $\UND$ respectively) are assumed to be computable from the initial probabilities of the arguments in isolation (and indeed an algorithm is provided to carry out this computation), while this is not the case in our approach.
This point deserves a specific comment also because an analogous assumption is adopted in other works belonging to the \RRR{constellations} approach \cite{DBLP:conf/tafa/LiON11,Fazzinga:2016:EEP:2868292.2868384} in the context of extension-based semantics. 
In our approach,  a semantics can produce multiple labellings for a given (sub) framework. \RR{So,} an agent can assign different probability values to each labelling. This captures uncertainty about \RR{the selection of acceptance labellings} even when there is no uncertainty about the inclusion of arguments in the frameworks; \RR{this type of acceptance uncertainty} is not addressed in \cite{DBLP:journals/cas/Dondio14}.
To exemplify, if for every argument $A$, $P(A)=1$ or $P(A)=0$, i.e. there is only one possible (and actually certain) subframework of the original framework, in \cite{DBLP:journals/cas/Dondio14}  the probability of acceptance, rejection or undecidedness of every arguments is either $1$ or $0$ in turn (this holds for all the evaluation alternatives in the paper, namely grounded, sceptical preferred, and credulous preferred).
As a consequence of the different assumptions on the nature of uncertainty to be represented, this constraint does not hold in our approach as far as multiple status semantics, like preferred semantics, are concerned.
Clearly, these different assumptions may be appropriate in different domains: comparing the suitability of the various existing approaches to the needs of distinct application contexts is left to future work.

\RRR{\PBB{The interest \RRT{in combinations of argumentation and probability} is also witnessed by the MARF (Markov Argumentation Random Field) software system based on the combination of abstract argumentation and Markov random fields presented in \cite{DBLP:conf/aaai/TangOS16}}. In \cite{DBLP:conf/aaai/TangOS16}, \PBB{argument} acceptability status can take four values: accepted (A), rejected (R), undecided (U), and ignored (I).
\PBB{and the system computes a probabilistic acceptability distribution on these values. 
Exploring in detail the relationships between these values and our labels and investigating the integration of our approach with Markov Random Field theory appear to be interesting lines of future development.}}

\subsubsection{Epistemic approach}

\PB{Besides the constellations approach}, other works interpret probabilities associated \RRT{with} abstract arguments according to the epistemic approach \RRT{which is originally} introduced in \cite{DBLP:conf/ecai/Thimm12} and further developed in \cite{DBLP:journals/corr/HunterT14,KRthimm:16}.
An equational approach to probabilistic abstract argumentation which bears some similarity with the epistemic approach of  \cite{DBLP:conf/ecai/Thimm12} is \RRT{studied} in \cite{Gabbay2015}, where a syntactical and a semantical method to the definition of the probabilities of arguments are proposed.
\cite{DBLP:conf/comma/BaroniGV14} explored an alternative setting for epistemic probabilities in abstract argumentation \RRT{by} using De Finetti's theory of subjective probabilities \cite{de1974theory} rather than Kolmogorov's axiomatization. Moreover they preliminarily investigated the extension of the approach to the case of imprecise probabilities \cite{Walley91}.

The variety of \RRT{flavours} in probabilistic argumentation presented in the above cited  papers suggests that the interpretation of the probability of an argument is, in an epistemic sense,  a largely open issue and that a finer classification could  be \RRT{devised}, where the generic notion of epistemic probability considered up to now in the literature is replaced by a more detailed taxonomy.
While this interesting direction of investigation is beyond the scope of this paper, we remark that the above mentioned papers typically resort to an intuitive explanation of the intended meaning of epistemic probability and assume as a starting point that the epistemic probability is given, leaving implicit the underlying probability space.
As discussed in Section  \ref{secconstepa}, our formalism is able to capture, at a formal level, an assignment of epistemic probabilities to arguments as a special case of  probabilistic $\{\IN, \OUT, \UND\}$-labellings based on a well-identified probability space. In the same formal setting, an epistemic probability assignment corresponding to a generic  probabilistic $\{\IN, \OUT, \UND\}$-labelling is identified, showing that we recover, in our context, some of the desirable technical properties introduced in other papers.
In this sense, we do not commit to any specific intuitive interpretation of epistemic probabilities but rather provide a reference formal framework which can be \RRT{employed} for further analyses of this notion not only at the abstract level, on which the above cited papers are focused, but also concerning its origin at the structured level and its impact on the evaluation of argument conclusions.

\subsection{Other connections between argumentation and probability}\label{subsec:others}

While the approach proposed in this paper and those previously overviewed focus on modelling probabilistic uncertainty in the context of the various steps of formal argumentative reasoning, other kinds of connections and interactions between argumentation and \RRT{probabilistic notions} have been \RRT{conceived; we discuss them} in the following.

At a foundational level, some works use basic probabilistic notions to give an interpretation of  defeasible rules  in argumentative reasoning.
In particular Pollock \cite{Pollock:1995} uses the notion of statistical syllogism to define \emph{prima facie reasons}. Briefly, in the simplest form, if $F$ holds and the conditional probability that $G$ holds given that $F$ holds is above a given threshold (typically $0.5$) then there is a prima facie reason to (defeasibly) infer $G$.
In this view, basic probability concepts (in particular the notion of conditioning) are, in a sense, regarded as more fundamental than argumentative notions which are a sort of derived concept, at least from an epistemological point of view.
As discussed by Verheij in \cite{DBLP:conf/comma/Verheij14}, Pollock however did \RRT{criticise} other aspects of probabilistic reasoning leading to a sort of \primaryquote{anti-probabilistic stance}. To reconcile Pollock's foundational \PBX{perspective} with standard probability laws\RRT{,} Verheij proposes a reformulation of the notion of prima facie reason, \RRT{and shows how}  it addresses some of Pollock's criticisms.
In the same foundational strain, Verheij \cite{DBLP:conf/jelia/Verheij12} proposes to combine basic notions from nonmonotonic reasoning and probability theories to \RRT{formalise} ampliative arguments, namely defeasible arguments going \primaryquote{beyond their premises}.
Again, a standpoint is assumed that argumentative notions can be interpreted in terms of more fundamental (or more primitive) concepts, probability being one of them.
Discussing foundational issues is beyond the scope of the present paper, which deals with combining probability and argumentation at a different level: independently of the possible interpretation of defeasible rules in probabilistic terms, there can be uncertainty about which rules are adopted, which arguments are built, which of the built arguments are accepted and so on. In this sense\RRT{,} our study of probabilistic uncertainty inside the process of argumentative reasoning can be \RRT{understood} as complementary to the study of how the defeasible rules \RRT{in} argumentative reasoning can be interpreted in probabilistic terms.

Connections between argumentation and one of the most popular probability-based formalisms, namely Bayesian networks, have also attracted a lot of attention in the recent years, also due to some basic similarities in the graph-based representations adopted in both areas.
For instance, in \cite{Hepler:2007} Wigmore charts (a graphical formalism useful for \primaryquote{describing and organizing the available evidence in a case and in following reasoning processes through sequential steps}) and Bayesian networks are compared by using them to model the infamous Sacco and Vanzetti case. Several pros and cons of both formalisms are discussed and then attention is focused on making the construction of Bayesian networks for complex legal cases more manageable by introducing a hierarchical object-oriented approach \RRT{utilising} \primaryquote{small modular networks (or network fragments) as building blocks}.
This idea is further pursued in \cite{DBLP:journals/cogsci/FentonNL13} where in order to avoid \emph{ad hoc} construction of Bayesian networks several basic structures, called \emph{idioms}, for the representation of legal cases are introduced. Differently from \cite{Hepler:2007}, this work does not assume a specific argument representation.
In the same line, but \PBX{adopting} a different formal background, \cite{Grabmair:2010:PSC:1860828.1860853} proposes a translation from the Carneades argument model into a Bayesian network providing it a probabilistic semantics.

Other approaches follow, \RRT{so to say}, the converse direction: they assume that a Bayesian network representing a case is given and investigate how to produce an argument-based representation for it \cite{Vreeswijk2005,Keppens2012}, under the assumption that this provides a simpler and more intuitive presentation.
For example, given a Bayesian network, \PBB{in
 \cite{DBLP:conf/jurix/TimmerMPRV14,Timmer2015,Timmer:2015:SAC} a support graph is derived} whose nodes can be associated with a numerical support, based on the original probabilistic information \cite{Timmer:2015:SAC}. Arguments can then be identified on the basis of the support graph \cite{Timmer2015}.

Further works \cite{Verheij:2014,Verheij:2016} \RRT{incorporate} a third kind of representation commonly adopted in legal reasoning, namely \emph{scenarios} and explore various kind of (pairwise and threewise) connections and combinations between them, probability and argumentation.

All the above cited works assume as a starting point that argumentation and probability theory provide complementary/competing formalisms for representing the same situation and, investigate, from a general knowledge representation perspective, their relative merits and disadvantages, and various options for connecting, deriving or combining these \RRT{formalisms}. \RRT{In this context,} argumentation formalisms are \RRT{essentially seen}  as descriptive tools (as in the case of Wigmore charts) while the (numeric) assessment of conclusions is typically assigned \RRT{by} Bayesian networks. 

Our approach \RRT{is different in that} it does not \PBX{follow the perspective of two} competing formalisms and involves a simple but full-fledged argumentation-based reasoning process where argument and conclusion assessments are carried out on the basis of argumentation semantics and the related notions of justifications. As already remarked, probability theory is used with the specific purpose of capturing various kinds of uncertainty that can be present in this process.
Addressing knowledge representation issues and inter-formalism comparisons is beyond the scope of the present paper and represents an interesting direction of future work.

\section{Conclusion}
\label{sec:conclusion}

The main contribution of the paper is a labelling-oriented framework for probabilistic argumentation.
The approach builds on an articulated account of an argumentation process, involving the four main stages of rule-based argument construction, argument graph definition, argument evaluation, and  statement evaluation.
We have considered various forms of uncertainty which can be present in these stages, casted them into a formal probabilistic setting and analysed their relationships.


Differently from other proposals in the literature \RRT{employing} the well-established Dung's abstract argumentation frameworks, we have adopted  semi-abstract argumentation graphs where the key role of the subargument relation is taken into account to \RRT{express} dependences between the events of inclusion/\RRT{exclusion} of arguments.
Moreover, uncertainty about the inclusion of arguments can be combined with uncertainty about their acceptance through the  notion of $\{\IN, \OUT, \UND, \OFF\}$-labellings under suitable legality constraints. 
In particular, we have shown that the proposed approach can capture key concepts underlying the so-called constellations and  epistemic approaches or a combination of both, and \RRR{that it can} ease the analysis of their properties. Some of these properties are easily retrieved by derivation in our context, while \PBB{it has been evidenced that the status of other properties is more open to discussion.}

Along the road,  we also evidenced key correspondences between probabilistic settings at different levels. Since different models of probabilistic
argumentation found in the literature are akin to these settings, we believe that these correspondences can be \RRT{helpful} to position novel proposals too within the whole picture of probabilistic argumentation. \smallskip

At a general level, the main lessons learnt along this journey concern diversity and unification in this research field.
Diversity arises from the fact that \primaryquote{probabilistic argumentation} is far from being a univocal term, leading in the literature to statements like \primaryquote{What is meant by the probability of an argument holding is an open question. Indeed, there seem to be various ways that we could answer this.} \cite{DBLP:journals/ijar/Hunter13}.

To account for this diversity while avoiding the risk of confusion and ambiguity, two aspects turned out to be crucial: analysing the argumentation process in its entirety, from argument construction to statement assessment, and specifying in detail the probability space (in particular the sample space) whenever some form of uncertainty \RRT{arises}.
Analysing all phases of the argumentation process allows to situate distinct forms of uncertainty \primaryquote{in the right place}, and detailed specification of the probability space avoids ambiguity about their nature, since the events one is uncertain about are formally identified. Simple as they are, these specifications are most often left implicit or ambiguous in the literature.
We believe therefore that both aspects provide a contribution to the need of conceptual clarity in the field.
In particular, they are essential to carry out an analysis \RRT{of} the relationships and possible dependences between the different forms of uncertainty, which is only partially doable when they are \RRT{held} in isolation or are not clearly specified.
An example of this kind of analysis is the distinction between the probability of a set of arguments being an extension \RRT{as usually conceived}  in the constellations approach and the probability of a labelling in our probabilistic labelling frames.
While, at a superficial level, they could be \RRT{seen} as similar, we have shown that the former can be regarded as a derived notion in the context of probabilistic graph frames, while the latter refers to a completely distinct form of uncertainty in the context of probabilistic labelling frames.
Suggestions for a critical reappraisal of some concepts adopted in non-probabilistic argumentation are another byproduct of this analysis, as exemplified by the discussion about probabilistic argument justification labellings.

While accounting for a diversity of uncertainties is essential, it has also to be acknowledged that they may be present together, leading to the unification problem: how to combine them within a formal framework, able to \RRT{generalise} some existing approaches and to provide a uniform treatment of diverse uncertainties. Probabilistic labellings turned out to be a suitable tool to this purpose.\smallskip

Future directions \RRT{are} multiple. We may \RRT{lay down} compact representations \RRT{through} probabilistic rules or factors,  paving the way to possible integration of probabilistic argumentation with well-established probabilistic graphical models and machine learning techniques to address the parametrisation and valuation of such models. 
In addition, we mentioned possible use of non-classical approaches to probability (e.g. De Finetti's subjective probabilities) in this context as an alternative to the classical definition of probability function based on Kolmogorov axioms we adopted here. 
Moreover, the argumentation process in this paper is rule-based: extending our analysis to other forms of argumentation (e.g. dialogical) is another direction of future work.
Finally, our investigation is entirely formal and independent from any particular domain. However, some domains may require to capture some particular features.
For example, as mentioned in \PBX{Section \ref{related}}, a variety of research proposed \RRT{a combination of different}  argumentation and probabilistic models to \RRT{address} diverse aspects of uncertainty in legal reasoning.
More generally, it would be interesting to see how probabilistic labellings could better help to characterise diverse aspects of uncertainty in specific application domains.

\section*{Acknowledgement}
\noindent  Yang Gao was supported by National Natural Science  Foundation  of  China  (NSFC)  grant 61602453. \RRT{R\'egis Riveret} was supported by the Marie Curie Intra-European Fellowship PIEF-GA-2012-331472. \RRT{Antonino Rotolo was supported by the European Union’s Horizon 2020 research and innovation programme under the Marie Skłodowska-Curie grant agreement No 690974 for the project MIREL: MIning and REasoning with Legal texts. The final publication is available at Springer via http://dx.doi.org/10.1007/s10472-018-9574-1}

\section*{Notation}
\label{section:notation}

\begin{table}[ht!]
\centering
\caption*{Some key notations used in the paper.}
\begin{tabular}{l | p{9.5cm}}
  $T$          & A  defeasible theory \\
 $G$          & An argumentation graph \\
\hline\\[-2ex]
  $\llabels$          & A set of labels for arguments \\
$l$ & A label for arguments  \\
$\labarg$ & A labelling function of arguments  \\

$L_A$ & The random labelling of an argument $A$ \\

\textbf{L} &  A set of random labellings of arguments  \\

\hline\\[-2ex]
$\klabels$          & A set of labels for literals \\
$k$ & A label for literals  \\
$\lablit$ & A labelling function of literals   \\
$K_\varphi$ & The random labelling of literal $\lit$  \\
$\textbf{K}$ &  A set of random labellings of literals \\
\hline\\[-2ex]
  $\setofalllabels{G}{\llabels}$ & The set of $\llabels$-labellings of the argumentation graph $G$\\
  $\mathcal{S}$ &  An   $X$-$\llabels$-labelling specification         \\
$$\crit{X}{G}{\llabels}$$ &The set of $X$-$\llabels$-labellings of the argumentation graph $G$, identified by the $X$-$\llabels$-labelling specification  \\
\hline\\[-2ex]
\PTF & A probabilistic theory frame\\
\PGF & A probabilistic graph frame\\
\PAF & A probabilistic labelling frame\\
\PAG & A probabilistic argumentation graph (from \cite{DBLP:journals/ijar/Hunter13})\\
\PEF & A probabilistic epistemic frame\\
\end{tabular}
\label{notationtable}
\end{table}


\bibliographystyle{spmpsci}       

\end{document}